\documentclass{article}

\usepackage{microtype}
\usepackage{graphicx}
\usepackage{subfigure}
\usepackage{booktabs} 
\usepackage{xspace}
\usepackage{amsthm}
\usepackage{amsmath}
\usepackage{amssymb}
\usepackage{multirow}
\usepackage{mathtools}
\usepackage{multicol}
\usepackage{color,soul}
\usepackage{xcolor}
\newtheorem{lemma}{Lemma}
\newtheorem{remark}{Remark}
\newtheorem{theorem}{Theorem}
\newtheorem{proof-sketch}{Proof sketch}
\newtheorem{corollary}{Corollary}

\newtheorem*{theorem*}{Theorem}
\newtheorem*{proposition*}{Proposition}
\newtheorem*{corollary*}{Corollary}
\newtheorem*{lemma*}{Lemma}

\usepackage{enumitem}
\setitemize{noitemsep,topsep=0pt,parsep=0pt,partopsep=0pt}

\usepackage{xcolor}
\usepackage[textsize=small]{todonotes}

\newcommand{\tr}[1]{\textcolor{red}{#1}}
\newcommand{\tg}[1]{\textcolor[RGB]{34,139,34}{#1}}

\usepackage{hyperref}

\usepackage{amsmath,amsfonts,bm,bbm}

\newcommand{\sensors}{s_*, s_{\mathcal{I}}, s_{\mathcal{J}}}

\newcommand{\recallthm}{\text{Recall}}

\newcommand{\pipelineaccuracy}{\mathcal{A}^{\textrm{\sys}}}
\newcommand{\mainaccuracy}{\mathcal{A}^{\textrm{main}}}

\newcommand{\decisionmargin}{distance to the optimal decision boundary}

\newcommand{\benign}{\mathcal{D}_{b}}
\newcommand{\adv}{\mathcal{D}_{a}}

\newcommand{\ci}{\mathcal{I}}
\newcommand{\cj}{\mathcal{J}}
\newcommand{\ck}{\mathcal{K}}

\newcommand{\mainsensor}{s_*}

\newcommand{\outputvariable}{o}

\newcommand{\imply}{\hspace{-.4em}\implies \hspace{-.4em}}

\newcommand{\dist}{\mathcal{D}}

\newcommand{\distset}{\small \mathcal{D}\in \hspace{-.1em} \{\mathcal{D}_{b}, \mathcal{D}_{a}\}}

\newcommand{\mainsensorweight}{w_*}

\newcommand{\prob}{\mathbb{P}}

\newcommand{\indicator}{\mathbbm{1}}

\newcommand{\deltaw}{\Delta_{\bf w}(y, s_*, s_{\mathcal{I}}, s_{\mathcal{J}})}

\def\eqref#1{equation~\ref{#1}}

\def\1{\bm{1}}

\DeclareMathAlphabet{\mathsfit}{\encodingdefault}{\sfdefault}{m}{sl}
\SetMathAlphabet{\mathsfit}{bold}{\encodingdefault}{\sfdefault}{bx}{n}

\newcommand{\E}{\mathbb{E}}

\DeclareMathOperator*{\argmax}{arg\,max}
\DeclareMathOperator*{\argmin}{arg\,min}

\newcommand\framework{Knowledge Enhanced Machine Learning Pipeline\xspace}

\newcommand\sys{KEMLP\xspace}
\newcommand\acc{weighted robust accuracy\xspace}
\newcommand\Acc{Weighted Robust Accuracy\xspace}

\newcommand{\equalContributionAndInternship}{\textsuperscript{*}Equal contribution
}

\usepackage[accepted]{icml2021}

\icmltitlerunning{Knowledge Enhanced Machine Learning Pipeline against Diverse Adversarial Attacks}

\begin{document}

\twocolumn[
\icmltitle{Knowledge Enhanced Machine Learning Pipeline \\ against Diverse Adversarial Attacks}




\icmlsetsymbol{equal}{*}

\begin{icmlauthorlist}
\icmlauthor{Nezihe Merve G\"urel}{eth,equal}
\icmlauthor{Xiangyu Qi}{zju,equal}
\icmlauthor{Luka Rimanic}{eth}
\icmlauthor{Ce Zhang}{eth}
\icmlauthor{Bo Li}{illinois}
\end{icmlauthorlist}

\icmlaffiliation{eth}{ETH Zurich, Zurich, Switzerland}
\icmlaffiliation{illinois}{University of Illinois at Urbana-Champaign, Illinois, USA}
\icmlaffiliation{zju}{Zhejiang University, China (work done during remote internship at UIUC)}

\icmlcorrespondingauthor{Nezihe Merve G\"urel}{nezihe.guerel@inf.ethz.ch}
\icmlcorrespondingauthor{Xiangyu Qi}{unispac@zju.edu.cn} 
\icmlcorrespondingauthor{Ce Zhang}{ce.zhang@inf.ethz.ch}
\icmlcorrespondingauthor{Bo Li}{lbo@illinois.edu}

\icmlkeywords{Machine Learning, ICML}

\vskip 0.3in
]



\printAffiliationsAndNotice{\equalContributionAndInternship}

\begin{abstract}
Despite the great successes achieved by deep neural networks (DNNs), recent studies show that they are vulnerable against adversarial examples, which aim to mislead DNNs by adding small adversarial perturbations.
Several defenses have been proposed against such attacks, while many of them have been adaptively attacked.
In this work, we aim to enhance the ML robustness from a different perspective by leveraging \textit{domain knowledge}:
We propose a \framework (\sys) to integrate domain knowledge (i.e., logic relationships among different predictions) into a probabilistic graphical model via first-order logic rules.
In particular, we develop \sys by integrating a diverse set of weak auxiliary models based on their logical relationships to the main DNN model that performs the target task.
Theoretically, we provide convergence results and prove that, under mild conditions, the prediction of \sys is more robust than that of the main DNN model.
Empirically, we take road sign recognition as an example and leverage the relationships between road signs and their shapes and contents as domain knowledge. 
We show that compared with adversarial training and other baselines, \sys achieves higher robustness against physical attacks, $\mathcal{L}_p$ bounded attacks, unforeseen attacks, and natural corruptions under both whitebox and blackbox settings, while still maintaining high clean accuracy.
\end{abstract}

\section{Introduction}
\label{sec:introduction}
Recent studies show that machine learning (ML) models are vulnerable to different types of adversarial examples, which are adversarially manipulated inputs aiming to mislead ML models to make arbitrarily incorrect predictions~\cite{szegedy2013intriguing,goodfellow2014explaining,bhattad2019unrestricted,eykholt2018robust}. Different defense strategies have been proposed against such attacks, including adversarial training~\cite{shafahi2019adversarial,madry2017towards}, input processing~\cite{ross2018improving}, and approaches with certified robustness against $\mathcal{L}_p$ bounded attacks~\cite{cohen2019certified,yang2020randomized}.
However, these defenses have either been adaptively attacked again~\cite{carlini2017adversarial,athalye2018obfuscated} or can only certify the robustness within a small $\ell_p$ perturbation radius.
In addition, when models are trained to be robust against one type of attack, their robustness is typically not preserved against other attacks~\citep{schott2018towards,kang2019testing}. 
Thus, despite the rapid
recent progress on robust learning, 
it is still {challenging} to provide 
robust ML models against a diverse set of 
adversarial attacks in practice.


In this paper,  
we take a different perspective towards training robust ML models against diverse adversarial attacks by integrating \textit{domain knowledge} during prediction, given the observation that human with knowledge is quite resilient against these attacks. 
We will first take stop sign recognition as a simple example to illustrate the potential role of knowledge in ML prediction.
In this example, the \textbf{main task} is to predict
whether a stop sign appears in the input image. Training a DNN model for this task is known to be vulnerable 
against a range of adversarial attacks~\cite{eykholt2018robust,xiao2018characterizing}.
However, upon such a DNN model, if we could (1)
build a detector for a different \textbf{auxiliary task},
e.g., detecting whether an octagon 
appears in the input by using other learning strategies such as traditional computer vision 
techniques, and (2) integrate the \textbf{domain knowledge} such that \textit{``A stop sign should be of 
an octagon shape''}, it is possible that additional information could enable the ML system to detect or defend against attacks, which lead to conflicts between the DNN prediction and domain knowledge.
For instance, if a speed limit sign with \textit{rectangle} shape is misrecognized as a stop sign, the ML system would identify this conflict and try to correct the prediction. 

Inspired by this 
intuition, we aim to understand how to
\textit{enhance the robustness of ML models
via domain knowledge integration}. Despite the natural 
intuition in the previous simple example,
providing a technically rigorous treatment 
to this problem is far from trivial, yielding the following questions: 
How should we integrate domain knowledge
in a principled way? When will 
integrating domain knowledge help with robustness
and will there be a tradeoff between robustness and clean accuracy?
Can integration of domain knowledge 
genuinely bring additional robustness benefits against practical attacks
when compared with state-of-the-art defenses?

\begin{figure}
\centering
\includegraphics[width=0.5\textwidth]{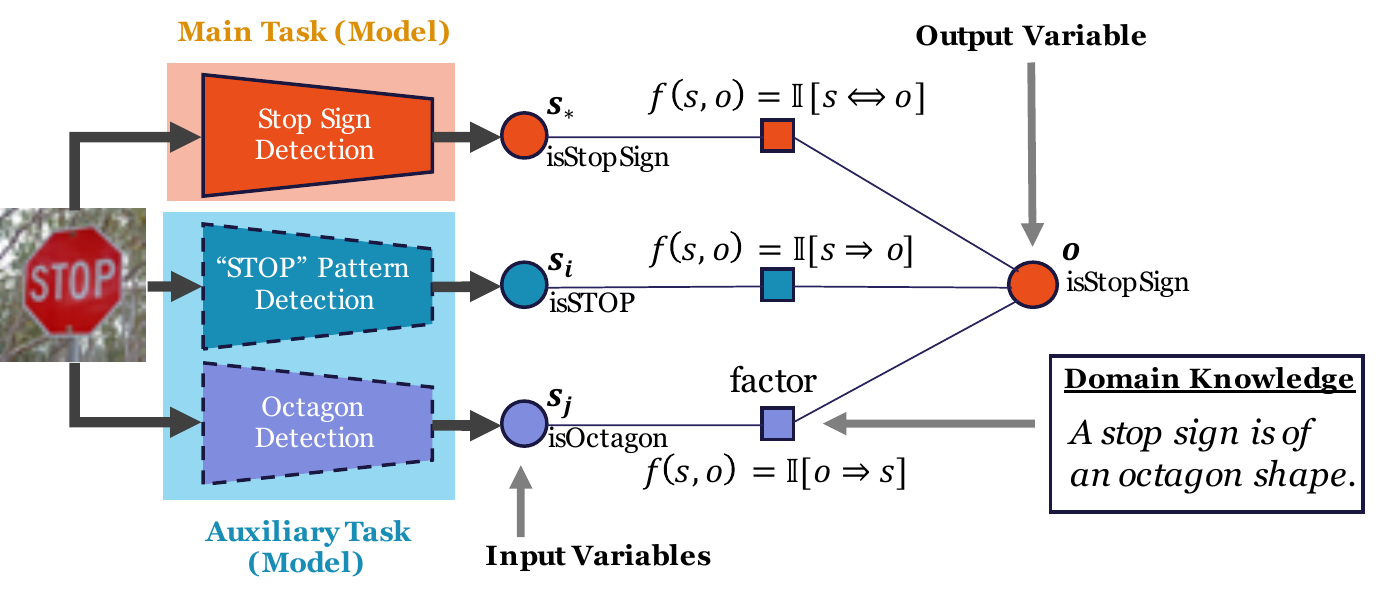}
\vspace{-2.5em}
\caption{An overview of the \sys framework. \sys 
constructs a factor graph by modeling 
the output of ML models as random input variables,
and the \sys prediction as a random output variable.
It integrates domain knowledge via factors connecting different random variables.}
\label{fig:kemlpoverview}
\vspace{-1.5em}
\end{figure}

In this work, we propose \sys, a
framework that facilitates the integration of
\textit{domain knowledge} in order to improve the robustness
of ML models. Figure~\ref{fig:kemlpoverview} illustrates
the \sys framework. In \sys, the 
outputs of different ML models are modeled as 
{random input variables}, whereas
the output of \sys is modeled as another 
variable. To integrate domain knowledge,
\sys introduces corresponding factors connecting 
these random variables. For example,
as illustrated in Figure~\ref{fig:kemlpoverview},
the knowledge rule ``\textit{A stop sign is 
of an octagon shape}'' introduces a factor
between the input variable (i.e., the output of the {octagon detector}) and the output 
variable (i.e., output of the {stop sign detector}) with a factor function that
\textit{the former implies the latter}. 
To make predictions, \sys runs statistical 
inference over the factor graph constructed
by integrating all such domain knowledge expressed as first-order logic rules, and 
output the marginal probability of 
the output variable.

Based on \sys, our main goal is to 
understand two fundamental questions based on \sys : (1) \textit{What type of knowledge is needed to improve the robustness of
the joint inference results from KEMLP, and can we prove it?} (2) \textit{Can we show that knowledge integration in
the KEMLP framework can provide significant robustness gain 
over powerful state-of-the-art models?}

We conduct theoretical analysis to understand the
first question, focusing on two specific types of knowledge rules:\\
(1) \textit{permissive knowledge}
of the form ``$B\implies A$'', and (2) \textit{preventive knowledge} of the form ``$A \implies B$'', 
where $A$
represents the main task, $B$  
an auxiliary task and $\implies$ denotes
logical implication. 
We focus on the \textit{\acc}, which is
a weighted average of accuracies on benign and adversarial examples, respectively, and we
derive sufficient conditions under which 
KEMLP outperforms the main task model alone. 
Under mild conditions, we show
that integrating multiple weak auxiliary  models, 
both in their robustness and quality,
together with the
permissive and preventive
rules, the \acc of \sys can be guaranteed to improve over the single main task model. 
To our best knowledge, this is the first
analysis of proposed form, focusing on 
the intersection of knowledge integration,
joint inference, and robustness. 

We then conduct extensive empirical studies to understand the second question. We focus on the road sign classification task and consider the state-of-the-art adversarial training models based on both the $\mathcal{L}_p$ bounded perturbation and occlusion perturbations~\cite{wu2019defending} as our baselines as well as the main task model.
We show that by training weak auxiliary models for recognizing the shapes and contents of road signs, together with the corresponding knowledge rules as illustrated in Figure~\ref{fig:kemlpoverview}, \sys achieves significant improvements on their robustness compared with baseline main task models against a \textit{diverse} set of adversarial attacks while maintaining similar or even higher clean accuracy, given its improvement on the tradeoff between clean accuracy and robustness.
In particular, we consider existing physical attacks~\cite{eykholt2018robust}, $\mathcal{L}_p$ bounded attacks~\cite{madry2017towards}, unforeseen attacks~\cite{kang2019testing}, and common corruptions~\cite{hendrycks2019benchmarking}, under both whitebox and blackbox settings. To our best knowledge, \sys is the first ML model robust to diverse attacks in practice with high clean accuracy. Our code is publicly available for reputability~\footnote{\url{https://github.com/AI-secure/Knowledge-Enhanced-Machine-Learning-Pipeline}}.

\underline{\bf Technical Contributions.} In this paper, we take the \textit{first} step
towards integrating \textit{domain knowledge} with ML to improve its robustness against different attacks.
We make contributions on both theoretical
and empirical fronts.   
\begin{itemize}[leftmargin=*,itemsep=-0mm]
\item We propose \sys, which integrates a main task ML model with a set of weak auxiliary task models, together with different knowledge rules connecting them. 
\item Theoretically, we provide the robustness guarantees for \sys and prove that under mild conditions, the prediction of \sys is more robust than that of a single main task model.
\item Empirically, we develop \sys based on different main task models, and evaluate them against a diverse set of attacks, including physical attacks, $\mathcal{L}_p$ bounded attacks, unforeseen attacks, and common corruptions. 
We show that the robustness of \sys outperforms all baselines by a wide margin,
with comparable and often higher  clean accuracy.
\end{itemize}



\section{Related Work}
\label{sec:related work}
In the following, we review several bodies of literature that are relevant to the objective of our paper.

\textbf{Adversarial examples} are carefully crafted inputs aiming to mislead well-trained ML models~\citep{goodfellow2014explaining,szegedy2013intriguing}.
A variety of approaches to generate such adversarial examples have also been proposed based on different perturbation measurement metrics, including  $\mathcal{L}_p$ bounded, unrestricted, and physical attacks~\citep{wong2019wasserstein,bhattad2019unrestricted,xiao2018generating,xiao2018spatially,eykholt2018robust}.

\textbf{Defense methods} against such attacks have been proposed. Empirically, \textit{adversarial training}~\citep{madry2017towards} has shown to be effective, together with feature quantization~\cite{xu2017feature} and reconstruction approaches~\citep{samangouei2018defense}. 
Certified robustness has also been studied by propagating the interval bound of a NN~\citep{gowal2018effectiveness}, or randomized smoothing of a given model~\citep{cohen2019certified}. Several approaches have further improved it: by
choosing different smoothing distributions for different $L_p$ norms~\citep{dvijotham2019framework,zhang2020black,yang2020randomized}, or training more robust smoothed classifiers via data augmentation~\citep{cohen2019certified}, unlabeled data~\citep{carmon2019unlabeled}, adversarial training~\citep{salman2019provably}, and regularization~\citep{li2019certified,zhai2019macer}.
While most prior defenses focus on leveraging statistical properties of an ML model to improve its robustness, they can only be robust towards a specific type of attack, such as $\ell_p$ bounded attacks.
This paper aims to explore how to utilize knowledge inference information to improve the robustness of a logically connected ML pipeline against a diverse set of attacks.

\textbf{Joint inference} has been studied to take multiple predictions made by different models, together with the relations among them, to make a final prediction~\citep{xu2020joint,deng2014large,poon2007joint,mccallum2009joint,chen2014joint,chakrabarti2014joint,biba2011protein}.
These approaches usually use different inference models, such as factor graphs~\citep{wainwright2008graphical}, Markov logic networks~\citep{richardson2006markov} and Bayesian networks~\citep{neuberg2003causality}, as a way to characterize their relationships. 
The programmatic weak supervision approaches~\cite{ratner2016data, ratner2017snorkel} also perform joint inference by employing labeling functions and using generative modeling techniques, which aims to create noisy training data. 
In this paper, we take a different perspective on this problem --- we explore the potential of using joint inference with the objective of integrating domain knowledge and to eventually improving the ML robustness.
As we will see, by integrating domain knowledge, it is possible to improve the learning robustness by a wide margin.

\section{KEMLP: Knowledge Enhanced Machine Learning Pipeline}
\label{sec:knowledge enhanced ml pipeline}
We first present the proposed framework \sys, which aims to improve the robustness of an ML model by integrating a diverse set of domain knowledge. In this section, we formally
define the \sys framework.

%

We consider a classification problem under a supervised learning setting, defined on a feature space $\mathcal{X}$ and a finite label space $\mathcal{Y}$. We refer to $x\in \mathcal{X}$ as an input and $y\in \mathcal{Y}$ as the target variable. 
An input $x$ can be a benign example or an
adversarial example. To model this,
we use $z \in \{0, 1\}$, a latent variable that is 
not exposed to \sys.
That is, $x$ is an adversarial example with $(x, y)\sim \adv$ whenever $z=1$, and $(x, y)\sim \benign$ otherwise, where $\adv$ and $\benign$ represent the adversarial and benign data distributions.
We let $\pi_{\adv}=\prob(z=1)$ and $\pi_{\benign}=\prob(z=0)$, implying $\pi_{\adv}+\pi_{\benign}=1$. 
For convenience, we denote $\prob_{\adv}(x, y)=\prob(x, y|z=1)$ and $\prob_{\benign}(x, y)=\prob(x, y|z=0)$. In the following, to ease the exposition, we slightly abuse the notation and use probability densities for discrete distributions.


Given an input $x$ whose corresponding
$z$ is unknown (benign or adversarial), \sys aims to predict the target variable $y$ by employing a set of \textit{models}.
These predictive models are constructed, say, using ML or some other traditional rule-based methods (e.g., edge detector). 
For simplicity, we describe 
the \sys framework as a binary classification task, in which case $\mathcal{Y}=\{0, 1\}$, noting that the multi-class scenario is a simple extension of it. We introduce the \sys framework as follows.


\paragraph*{Models} Models are a collection of predictive ML models, each of which takes as input 
$x$ and outputs some predictions. In \sys, we distinguish three different type of models.
\begin{itemize}[leftmargin=*, wide = 0pt]
\item \textit{Main task model}: We call the (untrusted) ML  model whose robustness users want to enhance as the \textit{main task model}, denoting its predictions by $\mainsensor \in \mathcal{Y}$. 
\item \textit{Permissive models}:
Let $s_\ci =\{s_i \colon i\in \ci\}$
be a set of $m$ permissive
models, each of which corresponds
to the prediction of 
one ML model.
Conceptually, permissive models are usually designed for specific events which are \textit{sufficient} for inferring $y=1$: $s_i \implies y$.
\item \textit{Preventative models}: Similarly,
we have $n$ preventative models:
$s_\cj=\{s_j\colon j\in \cj\}$,
each of which 
corresponds
to the prediction of 
one ML model.
Conceptually, preventative models
capture the events that are \textit{necessary} for the event {$y=1$}: $y \implies s_j$.
\end{itemize}

\paragraph*{Knowledge Integration}
Given 
a data example $(x, y)\sim{\benign}$ or $(x, y)\sim{\adv}$, $y$ is unknown to 
\sys. 
We create 
a factor graph to embed the domain knowledge as follows. 
The outputs of each  model over 
$x$ 
become \textit{input variables}:
$s_*, 
s_{\mathcal{I}}=\{s_i: i \in \mathcal{I}\}, 
s_{\mathcal{J}}=\{s_j: j \in \mathcal{J}\}$.
\sys also has an output variable  $\outputvariable \in \mathcal{Y}$, which corresponds to 
its prediction.
Different models introduce
different types of factors 
connecting these variables:
\begin{itemize}[leftmargin=*, wide = 0pt]
\item Main model:
\sys introduces a factor between 
the main model $\mainsensor$
and the output variable $\outputvariable$ with
factor function $f_{*}(\outputvariable, \mainsensor) = \indicator\{\outputvariable=\mainsensor\}$;
\item Permissive model:
\sys introduces a factor between
each permissive model $s_i$ and
the output variable $\outputvariable$ with
factor function 
$f_{i}(\outputvariable, s_i) = \indicator\{s_i \imply \outputvariable\}$.
\item Preventative model:
\sys introduces a factor between
each preventative model $s_j$ and 
the output variable $\outputvariable$
with factor function $f_{j}(\outputvariable, s_j) = \indicator\{\outputvariable \imply s_j\}$.
\end{itemize}

\paragraph*{Learning with \sys}
To make a prediction, \sys 
outputs the \textit{probability}
of the output variable $o$.
\sys assigns a weight for each
model and constructs the 
following statistical model:
\begin{align*}
& \mathbb{P}[o| s_*,
s_{\mathcal{I}}, s_{\mathcal{J}}, w_*, w_{\mathcal{I}}, w_{\mathcal{J}}, b_o]\ \ \propto \ \ \\
&\hspace{20mm}\exp
\{
b_{o} + w_* f_*(o, s_*) \} \times \\
&\hspace{5mm}\exp\big\{
\sum_{i \in \mathcal{I}}
w_i f_i(o, s_i)\big\} \times \exp\big\{
\sum_{j \in \mathcal{J}}
w_j f_j(o, s_j)
\big\}
\end{align*}
where 
$w_*, w_i, w_j$ are
the corresponding weights 
for models $s_*, s_i, s_j$, $w_{\mathcal{I}}=\{w_i: i \in \mathcal{I}\}, 
w_{\mathcal{J}}=\{w_j: j \in \mathcal{J}\}$ and $b_o$ is some bias parameter that depends on $o$.
For the simplicity of exposition,
we use an equivalent notation 
by putting all the weights and  outputs of factor functions
into vectors using an ordering of models. More precisely, we define

{\vspace{-2em}\small
\begin{align*}
 \mathbf{w}&=[1;w_*; ( w_i)_{i\in\ci}; (w_j)_{j\in\cj}],\\
 \mathbf{f}_{o}(s_*,
s_{\mathcal{I}}, s_{\mathcal{J}})&=[b_o; f_*(o, s_*); (f_i(o, s_i))_{i\in\ci}; (f_j(o, s_j))_{j\in\cj}],
\end{align*}
}\normalsize
for $o\in\mathcal{Y}$.
All concatenated vectors from above are in $\mathbb{R}^{m+n+2}$. 
Given this, 
an equivalent form of \sys's 
statistical model is
\begin{align}\label{eqn: class likelihoods}
\mathbb{P}[o | s_*,
s_{\mathcal{I}}, s_{\mathcal{J}}, \mathbf{w}] 
=\frac{1}{Z_{\bf w}} 
\exp(\langle{\bf w}, \mathbf{f}_{o}(s_*,
s_{\mathcal{I}}, s_{\mathcal{J}})\rangle)
\end{align}
where $Z_{\bf w}$ is the normalization constant over $o\in \mathcal{Y}$. With some abuse of notation, $\mathbf{w}$ is meant to govern all parameters including weights and biases whenever used with probabilities.

\paragraph*{Weight Learning}
During the training phase of \sys, we choose parameters {\bf w} by performing standard maximum likelihood estimation over a training dataset. Given a particular input instance $x^{\scriptscriptstyle(n)}$, respective model predictions $s_*^{\scriptscriptstyle(n)},
s^{\scriptscriptstyle(n)}_{\mathcal{I}}, s^{\scriptscriptstyle(n)}_{\mathcal{J}}$, and the ground truth label $y^{\scriptscriptstyle(n)}$, we minimize the negative log-likelihood function in view of
\small
\[
 \hat{{\bf w}}=\argmin_{{\bf w}}\Big\{- \sum_n \log \Big( \mathbb{P}[o^{\scriptscriptstyle(n)}=y^{\scriptscriptstyle(n)} |s_*^{\scriptscriptstyle(n)},
s^{\scriptscriptstyle(n)}_{\mathcal{I}}, s^{\scriptscriptstyle(n)}_{\mathcal{J}}, \mathbf{w}]\Big) \Big\}.
\]
\normalsize

\paragraph*{Inference}
During the inference phase of \sys, given an input example $\hat x$, we predict $\hat y$ that has the largest probability given the respective model predictions $\hat s_*,
\hat s_{\mathcal{I}}, \hat  s_{\mathcal{J}}$, namely, $\hat y=\argmax_{\tilde y\in\mathcal{Y}} \mathbb{P}[o =\tilde y| \hat  s_*, \hat s_{\mathcal{I}},\hat  s_{\mathcal{J}}, \hat{\mathbf{w}}]$.

\section{Theoretical Analysis}
\textit{How does knowledge integration impact the robustness of \sys?}
In this section, we provide
theoretical analysis about the impact of
domain knowledge integration on 
the robustness of \sys.
We hope to (1) depict the regime under which knowledge
integration can help with robustness; (2) explain how a collection of ``weak''
(in terms of prediction accuracy) but 
``robust'' auxiliary models, on 
tasks different from the main one,
can be used to boost overall robustness.
Here we state the main results, whereas we refer interested readers to Appendix~\ref{sec: proofs} where we provide all relevant details.

\paragraph*{Weighted Robust Accuracy}

Previous theoretical analysis on 
ML robustness~\cite{javanmard2020precise,xu2009robustness,raghunathan2020understanding} have identified 
two natural dimensions of model 
quality: \textit{clean accuracy}
and \textit{robust accuracy},
which are the accuracy of a given ML
model on inputs $x$ drawn from 
either the benign distribution 
$\mathcal{D}_b$ or 
 adversarial distribution $\mathcal{D}_a$.
In this paper, to balance their tradeoff, we use their weighted
average as
our main metric of interest. That is, 
given a classifier $h: \mathcal{X} \rightarrow \mathcal{Y}$ we define its \emph{\Acc} as
\begin{align*}
\mathcal{A}_h = \pi_{D_a} 
   \mathbb{P}_{\mathcal{D}_a}[h(x)=y]
   +
   \pi_{D_b} 
   \mathbb{P}_{\mathcal{D}_b}[h(x)=y].
\end{align*}
We use 
$\mathcal{A}^{\text{\sys}}$ and 
$\mathcal{A}^{\text{main}}$ to denote 
the weighted robust accuracies of 
\sys and main task model, respectively.

\subsection{$\mathcal{A}^{\text{\sys}}$: \Acc of 
\sys} The goal 
of our analysis is to identify 
the regime under which 
$\mathcal{A}^{\text{\sys}} > \mathcal{A}^{\text{main}}$ is guaranteed.
The main analysis
to achieve this hinges on 
deriving the weighted
robust accuracy $\mathcal{A}^{\text{\sys}}$
for \sys. We first 
describe the modeling assumptions 
of our analysis, and then describe two 
key characteristics of models, culminating in 
a lower bound  of $\mathcal{A}^{\text{\sys}}$.

\paragraph*{Modeling Assumptions} We assume that for a fixed $z$, that is, for a fixed $\distset$, the models make independent errors given the target variable. Thus, for all $\distset$, the class conditional distribution can be decomposed as
\[
\prob_{\dist}[s_*,
s_{\mathcal{I}}, s_{\mathcal{J}}|y]=\prob_{\dist}[s_*|y]\prod_{i\in\mathcal{I}}\prob_{\dist}[s_{i}|y]\prod_{j\in\mathcal{J}}\prob_{\dist}[s_{j}|y].
\]

We also assume for simplicity that the main task model makes symmetric errors given the class of target variable, that is, ${\prob_{\dist}[s_*\neq y|y]}$ is fixed with respect to $y$ for all $\distset$. 

\paragraph*{Characterizing Models: Truth Rate ($\alpha$) and False Rate ($\epsilon$)}
Each auxiliary 
model $k \in \mathcal{I} \cup \mathcal{J}$ is characterized 
by two values, their
truth rate ($\alpha$) and
false rate 
($\epsilon$)
over benign and adversarial 
distributions. These
values measure the 
\textit{consistency} 
of the model with 
the ground truth:
\begin{align*}
& \text{Permissive Models:} \\
& \alpha_{i,\dist}:=\prob_{\dist}[s_i=y|y=1], \hspace{.6em} \epsilon_{i,\dist}:=\prob_{\dist}[s_i\neq y|y=0]
\end{align*}
\vspace{-2em}
\begin{align*}
& \text{Preventative Models:} \\
& \alpha_{j,\dist}:=\prob_{\dist}[s_j=y|y=0], \hspace{.6em} \epsilon_{j,\dist}:=\prob_{\dist}[s_j\neq y|y=1]
\end{align*}

Note that, given the 
asymmetric nature 
of these auxiliary models,
we do \textit{not}
necessarily have 
$\epsilon_{k, \mathcal{D}}
= 1 - \alpha_{k, \mathcal{D}}$.
In addition, for a high quality permissive 
model ($k\in\ci)$, or a high quality preventative 
model ($k\in\cj$)
for which the logic rules mostly 
hold, we expect $\alpha_{k,\dist}$ to be large and $\epsilon_{k,\dist}$ to be small. 

We define the truth rate of main model over data examples drawn from $\distset$ as $\alpha_{*,\dist}:=\prob_{\dist}(s_*=y)$, and its false rate as $\epsilon_{*,\dist}:=\prob_{\dist}(s_*\neq y)=1-\alpha_{*,\dist}$. 


These characteristics are of integral importance to weighted robust accuracy of \sys. To combine all the models together, we define upper and lower bounds to truth rates and false rates. For the main model, we have ${}_{\scriptscriptstyle \wedge}\alpha_{*}:=\min_{\dist}\alpha_{*,\dist}$ and ${}_{\scriptscriptstyle \vee}\alpha_{*}:=\max_{\dist}\alpha_{*,\dist}$. For the auxiliary models, on the other hand, for each model index $k \in \ci \cup \cj$, we have
\begin{align*}
   & {}_{\scriptscriptstyle \wedge}\alpha_{k}:=\min_{\dist}\alpha_{k,\dist}, \hspace{.8em} {}_{\scriptscriptstyle \wedge}\epsilon_{k}:=\min_{\dist}\epsilon_{k,\dist}\\
   & {}_{\scriptscriptstyle \vee}\alpha_{k}:=\max_{\dist}\alpha_{k,\dist}, \hspace{.8em} {}_{\scriptscriptstyle \vee}\epsilon_{k}:=\max_{\dist}\epsilon_{k,\dist}.
\end{align*}
Intuitively, the difference between 
${}_{\scriptscriptstyle \wedge}\alpha$
and 
${}_{\scriptscriptstyle \vee}\alpha$
(resp. ${}_{\scriptscriptstyle \wedge}\epsilon$
and 
${}_{\scriptscriptstyle \vee}\epsilon$) indicates 
the ``robustness'' of each individual model. If a model 
performs very similarly 
when it is given a benign and
an adversarial example,
we have that 
${}_{\scriptscriptstyle \wedge}\alpha$
should be similar to 
${}_{\scriptscriptstyle \vee}\alpha$
(resp. ${}_{\scriptscriptstyle \wedge}\epsilon$
to 
${}_{\scriptscriptstyle \vee}\epsilon$).

The truth and false rates of models directly influence the factor weights which govern the influence of models in the main task. In Appendix~\ref{subsec: generative model parameters} we prove that the optimal weight of an auxiliary model is bounded by $w_k\geq \log{}_{\scriptscriptstyle\wedge}\alpha_k(1-{}_{\scriptscriptstyle\vee}\epsilon_k)/(1-{}_{\scriptscriptstyle\wedge}\alpha_k){}_{\scriptscriptstyle\vee}\epsilon_k$, for all $k\in\mathcal{I}\cup\mathcal{J}$. That is, the lowest truth rate and highest false rate of an auxiliary model (resp. ${}_{\scriptscriptstyle\wedge}\alpha_k$ and ${}_{\scriptscriptstyle\vee}\epsilon_k$) are indicative of its influence in the main task. By taking partial derivatives, this lower bound can be shown to be increasing in ${}_{\scriptscriptstyle\wedge}\alpha_k$ and decreasing in ${}_{\scriptscriptstyle\vee}\epsilon_k$. That is, as the lowest truth rate of a model gets higher, \sys increases its influence in the weighted majority voting accordingly -- in the above nonlinear fashion. The lowest truth rate is often determined by the \textit{robust accuracy}. As a result,
 the more ``robust'' an auxiliary model is,
 the larger the influence on 
 \sys, which naturally contributes to its robustness.

\paragraph*{\Acc of \sys}
We now provide a lower bound 
on the weighted robust accuracy of
 \sys, which can
be written as
\begin{equation}\label{eqn:Apipeline}
\pipelineaccuracy=\mathbb{E}_{\dist\sim \{\dist_a,\dist_b\}}\mathbb{E}_{y\sim\mathcal{Y}}\big[\prob_{\dist}[o=y|y, \mathbf{w}]\big].
\end{equation} 

We first provide one key 
technical lemma 
followed by the general 
theorem.

We see that the key component 
in  $\pipelineaccuracy$ 
is $\prob_{\dist}[o=y|y, \mathbf{w}]$,
the conditional probability that a 
\sys pipeline outputs
the correct prediction.
Using knowledge aggregation rules $f_*, f_i$ and $f_j$, as well as~(\ref{eqn: class likelihoods}), for each $\distset$ we have
\begin{align*}
    &\prob_{\dist}[o=y|y, \mathbf{w}]=\prob_{\dist}\Big[\prob[o=y|s_*,s_{\ci},s_{\cj}, \mathbf{w}]>1/2\big| y\Big]\\
    &=\prob_{\dist}\big[\langle{\bf w}, \mathbf{f}_{y}(s_*,s_{\mathcal{I}}, s_{\mathcal{J}})- \mathbf{f}_{1-y}(s_*,s_{\mathcal{I}}, s_{\mathcal{J}})\rangle > 0 |y\big].
\end{align*}
To bound the above value, we need to 
characterize the concentration 
behavior of the random variable
\[
\Delta_{\bf w} (y,s_*,s_{\mathcal{I}}, s_{\mathcal{J}}):=\langle{\bf w}, \mathbf{f}_{y}(s_*,s_{\mathcal{I}}, s_{\mathcal{J}})-\mathbf{f}_{1-y}(s_*,s_{\mathcal{I}}, s_{\mathcal{J}})\rangle.
\]
That is, we need to bound its left tail
below zero. For this purpose,
we reason about its expectation,
leading to the following lemma.
\begin{lemma}\label{lemma: expected tail}
Let $\Delta_{\bf w}$ be a random variable defined above. Suppose that \sys uses optimal parameters {\bf w} such that $\prob[y|s_*,s_{\mathcal{I}}, s_{\mathcal{J}}]=\prob[o|s_*,s_{\mathcal{I}}, s_{\mathcal{J}}, \mathbf{w}]$. Let also $r_y$ denote the log-ratio of class imbalance $\log\frac{\prob[y=1]}{\prob[y=0]}$. For a fixed $y\in\mathcal{Y}$ and $\distset$, one has
\begin{align*}
& \mathbb{E}_{s_*,s_{\mathcal{I}}, s_{\mathcal{J}}}[\Delta_{\bf w}(y,s_*,s_{\mathcal{I}}, s_{\mathcal{J}})|y] \nonumber \\
&\geq  \mu_{d_{*,\dist}} \hspace{-0.2em}+\hspace{-0.2em} y \mu_{d_{\mathcal{I},\dist}} \hspace{-0.2em}+\hspace{-0.2em} (1-y) \mu_{d_{\mathcal{J},\dist}}\hspace{-0.2em}+\hspace{-0.2em}(2y-1)r_y:=\mu_{y,\dist},
\end{align*} 
where
\begin{equation*}
    \mu_{d_{*,\dist}}=\alpha_{*, \dist}\log\frac{{}_{\scriptscriptstyle\wedge}\alpha_{*}}{1-{}_{\scriptscriptstyle\wedge}\alpha_{*}}+(1-\alpha_{*, \dist})\log\frac{1-{}_{\scriptscriptstyle\vee}\alpha_{*}}{{}_{\scriptscriptstyle\vee}\alpha_{*}},
\end{equation*}
\small
\begin{equation*}
    \begin{split}
        \mu_{d_{\mathcal{I},\dist}}&=\sum_{i\in \mathcal{I}}\alpha_{i, \dist}\log\frac{{}_{\scriptscriptstyle\wedge}\alpha_{i}}{{}_{\scriptscriptstyle\vee}\epsilon_{i}}+(1-\alpha_{i, \dist})\log\frac{1-{}_{\scriptscriptstyle\vee}\alpha_{i}}{1-{}_{\scriptscriptstyle\wedge}\epsilon_{i}}\\
        &-\sum_{j\in \mathcal{J}}\epsilon_{j, \dist}\log\frac{{}_{\scriptscriptstyle\vee}\alpha_{j}}{{}_{\scriptscriptstyle\wedge}\epsilon_{j}}-(1-\epsilon_{j, \dist})\log\frac{1-{}_{\scriptscriptstyle\wedge}\alpha_{j}}{1-{}_{\scriptscriptstyle\vee}\epsilon_{j}},
    \end{split}
\end{equation*}
\normalsize
and
\small
\begin{equation*}
    \begin{split}
        \mu_{d_{\mathcal{J},\dist}}&=\sum_{j\in \mathcal{J}}\alpha_{j, \dist}\log\frac{{}_{\scriptscriptstyle\wedge}\alpha_{j}}{{}_{\scriptscriptstyle\vee}\epsilon_{j}}+(1-\alpha_{j, \dist})\log\frac{1-{}_{\scriptscriptstyle\vee}\alpha_{j}}{1-{}_{\scriptscriptstyle\wedge}\epsilon_{j}}\\
        &-\sum_{i\in \mathcal{I}}\epsilon_{i, \dist}\log\frac{{}_{\scriptscriptstyle\vee}\alpha_{i}}{{}_{\scriptscriptstyle\wedge}\epsilon_{i}}-(1-\epsilon_{i, \dist})\log\frac{1-{}_{\scriptscriptstyle\wedge}\alpha_{i}}{1-{}_{\scriptscriptstyle\vee}\epsilon_{i}}.
    \end{split}
\end{equation*}
\normalsize
\end{lemma}

\emph{Proof Sketch.} This lemma can be derived by first decomposing $\Delta_{\bf w}$ into parts that are relevant for $s_*,s_{\ci},s_{\cj}$, namely there exist $d_{*,\dist}, d_{\mathcal{I},\dist}, d_{\mathcal{J},\dist}$ such that
\small
\begin{align*}
    \Delta_{\bf w}(y,s_*,s_{\mathcal{I}}, s_{\mathcal{J}})\hspace{-0.2em}=\hspace{-0.2em} d_{*,\dist}\hspace{-0.1em}+\hspace{-0.1em}y d_{\mathcal{I},\dist}\hspace{-0.1em}+\hspace{-0.1em}(1-y) d_{\mathcal{J},\dist} \hspace{-0.2em}+\hspace{-0.2em}(2y-1)r_y.
\end{align*}
\normalsize
Then we prove that 
$\mu_{*,\dist}\leq\mathbb{E}[d_{*,\dist}]$ for the main model, and 
$\mu_{d_{\mathcal{K},\dist}}\leq\mathbb{E}[d_{\mathcal{K},\dist}]$ for $\mathcal{K}\in\{\ci, \cj\}$, the permissive and preventative models. The full proof is presented in Appendix~\ref{subsec: proof of lemma}.

\paragraph*{Discussion} 
The above lemma illustrates the 
relationship between the models
and $\pipelineaccuracy$.
Intuitively, the larger
$\mu_{y,\dist}$ is, the further away the expectation 
of 
$\Delta_{\bf w}(y,s_*,s_{\mathcal{I}}, s_{\mathcal{J}})$ is from 0, and thus, the larger 
the probability that
$\Delta_{\bf w}(y,s_*,s_{\mathcal{I}}, s_{\mathcal{J}})>0$.
We see that $\mu_{y,\dist}$ consists of three
terms: $\mu_{d_{*,\dist}}$, $ \mu_{d_{\mathcal{I},\dist}}$, $ \mu_{d_{\mathcal{J},\dist}}$, measuring the contributions 
from the main model for all $y$, 
permissive models and preventative models for $y=1$ and $y=0$, respectively. More specifically, $\mu_{y,\dist}$ is increasing in terms of a weighted sum of $\alpha_{i}$, and decreasing in terms of a weighted sum of $\epsilon_{j}$. 
When $s_i \implies y$ holds
(permissive models),
it implies a large $\alpha_{i}$
for $y=1$,
whereas when $y \implies s_j$
holds (preventative model)
it implies a small $\epsilon_j$
for $y=1$. Thus, this lemma 
connects the property of auxiliary 
models to the weighted robust accuracy 
of \sys.
\subsection{Convergence of $\pipelineaccuracy$}
Now we are ready to present our convergence result.
\begin{theorem}[Convergence of $\pipelineaccuracy$]\label{thm: main theorem}
For $y\in \mathcal{Y}$ and $\distset$, let $\mu_{y,\dist}$ be 
defined as in Lemma~\ref{lemma: expected tail}. Suppose that the \textit{modeling assumption} holds, and suppose that $\mu_{d_{\mathcal{K}, \dist}}>0$, for all $\mathcal{K}\in\{\mathcal{I}, \mathcal{J}\}$ and $\distset$. Then 
\begin{equation}\label{eqn: generic lower bound}
    \begin{split}
       \pipelineaccuracy \geq 1-\mathbb{E}_{\mu_{y, \dist}}[\exp\big({-{2}{\mu_{y, \dist}^2}/{v^2}}\big)] ,
    \end{split}
\end{equation}
where $v^2$ is the variance upper bound to $\prob[o=y|y, \mathbf{w}]$ with
\small
\begin{equation*}
\begin{split}
v^2\hspace{-0.2em}= 4\Big(\log\frac{{}_{\scriptscriptstyle \vee}\alpha_{*}}{1-{}_{\scriptscriptstyle \wedge}\alpha_{*}}\Big)^2\hspace{-0.5em}+\hspace{-0.5em}\sum_{k\in \mathcal{I}\cup\mathcal{J}} \Big(\log\frac{{}_{\scriptscriptstyle \vee}\alpha_{k}(1-{}_{\scriptscriptstyle \wedge}\epsilon_{k})}{{}_{\scriptscriptstyle \wedge}\epsilon_{k}(1-{}_{\scriptscriptstyle \vee}\alpha_{k})} \Big)^2\hspace{-0.5em}.
\end{split}
\end{equation*}
\normalsize
\end{theorem}

\emph{Proof Sketch.} We begin by subtracting the term $\mu_{y,\dist}$ from $\prob_{\dist}(o=y|y,\mathbf{w})$, and then decomposing the result into individual summands, where each summand is induced by a single model. We then treat each summand as a bounded increment whose sum is a submartingale.
Followed by an application of generalized bounded difference inequality~\cite{van2002hoeffding}, we arrive at the proof, whose full details can be found in Appendix~\ref{subsec: proof of main theorem}.

\paragraph*{Discussion} 
In the following, we attempt to understand the scaling of the weighted robust accuracy of \sys\ in terms of models' characteristics. 

\textit{Impact of truth rates and false rates:} We note that $\mu_{d_{\mathcal{K}, \dist}}$ for $\mathcal{K}\in\{\mathcal{I}, \mathcal{J}\}$, which is an additive component of $\mu_{y, \dist}$, poses importance to understand the factors contributing to the performance of \sys. Generally, larger $\mu_{d_{\mathcal{K}, \dist}}$ (hence $\mu_{y, \dist}$) would increase the right tail probability of $\Delta_{\mathbf{w}}(y, s_*, s_{\ci}, s_{\cj})$ leading to a larger weighted accuracy for \sys. Although exceptions exist in cases where the variance increases disproportionally, here in our discussion we first focus on parameters that increase $\mu_{d_{\mathcal{K}, \dist}}$. Towards that, we simplify our exposition and let each auxiliary model have the same truth and false rate over both benign and adversarial examples, and within each type,
where the exact parameters are given by ${\alpha}_k:={\alpha}_{k, \dist}={}_{\scriptscriptstyle\wedge}{\alpha}_{k, \dist}={}_{\scriptscriptstyle\vee}{\alpha}_{k, \dist}$ and ${\epsilon}_{k}:={\epsilon}_{k, \dist}={}_{\scriptscriptstyle\wedge}{\epsilon}_{k, \dist}={}_{\scriptscriptstyle\vee}{\epsilon}_{k, \dist}$, for $k\in\mathcal{I}\cup\mathcal{J}$. 
In this simplified setting where the expected performance improvement by the auxiliary models is given by $\mu_{d_{\mathcal{K}, \dist}}$ for $\mathcal{K}\in\{\mathcal{I}, \mathcal{J}\}$ and fixed with respect to $\dist$, one can observe through partial derivatives that $\mu_{d_{\mathcal{K}, \dist}}$ is increasing over ${\alpha}_{k}$ and decreasing over ${\epsilon}_k$. 
This explains why the two types of 
knowledge rules would help: high-quality permissive models would have high truth rate and low false rate ($\alpha_i$ and $\epsilon_i$), as well as the preventative models ($\alpha_j$ and $\epsilon_j$), yet with different coverages for $y\in\mathcal{Y}$.

\textit{Auxiliary models in \sys\ - the more the merrier?} Next, we investigate the effect of the number of auxiliary models. To simplify, let $|\mathcal{I}|=|\mathcal{J}|$, and let $\hat{\mu}_{y, \dist}$ be a random variable with $\hat{\mu}_{y, \dist}={\mu_{y, \dist}}/({n+1})$, and $\hat{v}^2={v^2}/({n+1})$. The exponent thus becomes $-{\mu^2_{y, \dist}}/v^2=-(n+1){\hat{\mu}_{y, \dist}^2}/{\hat{v}^2}$. One can show that ${\hat{\mu}_{y, \dist}^2}/{\hat{v}^2}\geq c$ for some positive constant $c$, implying that $\pipelineaccuracy\geq 1-\exp(-2(n+1)c)$. That is, increasing the number of models generally improves the weighted robust accuracy of \sys. To demonstrate this, we now focus on understanding the scaling of weighted robust accuracy on a simplified setting. We assume that the auxiliary models are \emph{homogeneous} for each type: permissive or preventative. For example, $\alpha_k$ is fixed with respect to $k\in\mathcal{I}\cup\mathcal{J}$, hence we drop the subscripts, i.e., $\alpha_{k, \dist}=\alpha$ and $\epsilon_{k, \dist}=\epsilon$. 
We assume that the same number of auxiliary models are used, namely $|\mathcal{I}|=|\mathcal{J}|=n$, and that the classes are balanced with $\prob_{\dist}(y=1)=\prob_{\dist}(y=0)$, for all $\distset$. Finally, we let $\alpha_{*, \benign}=1$ and $\alpha_{*, \adv}=0$, and $\alpha-\epsilon>0$. Then, the following holds.
\begin{corollary}[Homogenous models]\label{cor: homogenous models}
The weighted robust accuracy of \sys in the homogeneous setting satisfies
\begin{equation*}\label{eqn: lower bound (main)}
    \begin{split}
      \mathcal{A}^{\text{\sys}} \geq1-\exp\big(-2n{(\alpha-\epsilon)^2}\big).
    \end{split}
\end{equation*}
In particular, one has $\lim_{n\rightarrow \infty}\mathcal{A}^{\text{\sys}}=1.$
\end{corollary}
For this particular case, the predicted class for the target variable $y$ is based upon an (unweighted) majority voting decision.
The above result suggests that for a setting where the auxiliary models are homogeneous with different coverage, the performance of \sys to predict the output variable $y$ robustly is determined by: (a) the difference between the probability of predicting the output variable correctly and that of making an erroneous prediction, that is, $\alpha-\epsilon$, and (b)~the number of auxiliary models. Consequently, $\pipelineaccuracy$ converges to $1$ exponentially fast in the number of auxiliary models as long as $\alpha-\epsilon>0$, which is naturally satisfied by the principle \sys employs while constructing the logical relations between the output variable and different knowledge.

\subsection{Comparing $\mathcal{A}^{\text{\sys}}$ and $\mathcal{A}^{\text{main}}$}\label{section: comparison}

Theorem~\ref{thm: main theorem} guarantees that the addition of models allows the weighted robust accuracy of \sys to converge to 1 exponentially fast. We now introduce a sufficient condition under which $\mathcal{A}^{\text{\sys}}$ is strictly better than $\mathcal{A}^{\text{main}}$. 

\begin{theorem}[Sufficient condition for $\pipelineaccuracy>\mainaccuracy$]\label{proposition}
Let the number of permissive and preventative models be the same and denoted by $n$ such that $n:=|\ci|=|\cj|$. Note that the weighted accuracy of the main model in terms of its truth rate is simply $\alpha_*:=\sum_{\distset}\pi_{\dist}\alpha_{*, \dist}$. Moreover, let $\ck, \ck'\in\{\ci, \cj\}$ with $\ck\neq\ck'$ and for any $\distset$, let 
\begin{equation*}
    \begin{split}
       \gamma_{\dist} :=\frac{1}{n+1} \min_{\ck}\Big\{ \alpha_{*, \dist}-1/2+\sum_{k\in\ck}\alpha_{k, \dist}-\sum_{k'\in\mathcal{K}'}\epsilon_{k', \dist}\Big\}. 
    \end{split}
\end{equation*}

If $\gamma_{\dist}>\sqrt{\frac{4}{n+1}\log\frac{1}{1-\alpha_{*}}}$ for all $\distset$, then $\pipelineaccuracy>\mainaccuracy$.
\end{theorem}

\emph{Proof Sketch.} We first approximate $\Delta_{\mathbf{w}}(y, s_*, s_{\ci}, s_{\cj})$ with a Poisson Binomial random variable and apply the relevant Chernoff bound. Imposing a strict bound between the Chernoff result and the true and false rates of main model concludes the proof. We note that this bound is slightly simplified, and our full proof in the Appendix~\ref{subsec: proof of proposition} is tighter.

\paragraph*{Discussion} We start by noting that $\gamma_{\dist}$ is a combined truth rate of all models normalized over the number of models. That is, for a fixed distribution $\dist$, $\alpha_{*, \dist}-1/2$ indicates the truth rate of main task model over a random classifier and $\sum_{k\in\ck}\alpha_{k, \dist}-\sum_{k'\in\mathcal{K}'}\epsilon_{k', \dist}$ refers to the improvement by the auxiliary models on top of the main task model. More specifically, in cases where the true class of output variable is positive with $y=1$, $\sum_{i\in\ci}\alpha_{i, \dist}-\sum_{j\in\mathcal{J}}\epsilon_{j, \dist}$ account for the total (and unnormalized) success of permissive models in identifying $y=1$ interfered by the failure of preventative model in identifying $y=1$ (resp. For $y=0$, $\ck=\cj$). Hence, $\gamma_{\dist}$ is the "worst-case'' combined truth rate of all models, where the worst-case refers to minimization over all possible labels of target variable.

Theorem~\ref{proposition} therefore forms a relationship between the improvement of \sys over the main task model and the combined truth rate of models, and theoretically justifies our intuition -- larger truth rates and lower false rates of individual auxiliary models result in larger combined truth rate $\gamma_{\dist}$, hence making the sufficient condition more likely to hold.
Additionally, employing a large number of auxiliary models is found to be beneficial for better \sys performance, as we conclude in Corollary~\ref{cor: homogenous models} as well. Our finding here also confirms that in the extreme scenarios where the main task model has a perfect clean and robust truth rate ($\alpha_*=1$), it is \textit{not} possible to improve upon the main task model. Conversely, when $\alpha_*=0$, any improvement by \sys would result in absolute improvement over the main model. 

\begin{figure*}[t!]
\centering
\includegraphics[width=\textwidth]{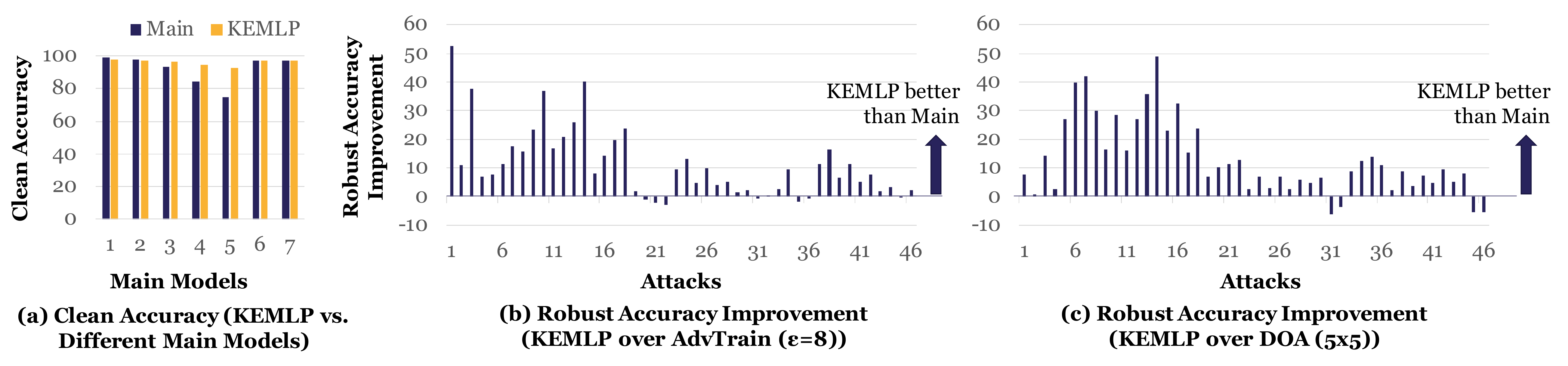}
\vspace{-2em}
\caption{\small (a) Clean accuracy and (b) (c) robust accuracy improvement of \sys ($\beta=0.5$) over baselines against different attacks under both whitebox and blackbox settings. The represented attack list and results of other baselines are in Appendix~\ref{appendix:details-of-attacks-and-corruptions}.}
\label{fig:general_accuracy}
\end{figure*}

\begin{table*}[tp]
    \caption{Model performance (\%) under physical attacks ($\beta=0.4$). Performance \tg{gain} and \textcolor{red}{loss} of \sys over baselines are highlighted.}
    \label{table:stop_sign_attack}
    \centering
    \resizebox{0.8\textwidth}{!}{
    \begin{tabular}{c|c|c|c|c|c|c|c}
        \toprule
            \multirow{2}{*}{} & \multicolumn{3}{c|}{Main} & & \multicolumn{3}{c}{KEMLP} \\ 
            \cline{2-8} &  Clean Acc  &  Robust Acc & W-Robust Acc & & Clean Acc  &  Robust Acc & W-Robust Acc\\  
        \hline
             GTSRB-CNN & ${100}$ & $5$ & $52.5$ & & $100$(${\pm 0}$) & $87.5$($\tg{+82.5}$)& $93.75$($\tg{+41.25}$)\\
        \hline
             AdvTrain ($\epsilon=4$) & ${100}$ & $12.5$ & $56.25$ & & $100$(${\pm 0}$) & $90$($\tg{+77.5}$) & $95$($\tg{+38.75}$)\\
        \hline
             AdvTrain ($\epsilon=8$) & $97.5$ & $37.5$ & $67.5$ & & $100$($\tg{+2.5}$) & $ 90$($\tg{+52.5}$) & $95$($\tg{+27.5}$)\\
        \hline
             AdvTrain ($\epsilon=16$) & $87.5$ & $50$ & $68.75$ & & $100$($\tg{+12.5}$) & $90$($\tg{+40}$) &  $95$($\tg{+26.25}$)\\
        \hline
             AdvTrain ($\epsilon=32$) & $62.5$ & $32.5$ & $47.5$ & & $100$($\tg{+37.5}$) & $90$($\tg{+57.5}$) & $95$($\tg{+47.5}$)\\
        \hline
             DOA (5x5) & $95$ & $90$ & $92.5$ & & $100$($\tg{+5}$) & $100$($\tg{+10}$) & $100$($\tg{+7.5}$)\\
        \hline
             DOA (7x7) & $57.5$ & $32.5$ & $45$ & & $100$($\tg{+42.5}$) & $100$($\tg{+67.5}$) & $100$($\tg{+55}$)\\
        \bottomrule
    \end{tabular}
    }
    \vspace{-0.2em}
\end{table*}

\begin{table*}[th]
\vspace{-5mm}
    \caption{Accuracy (\%) under whitebox $\mathcal{L}_\infty$ attacks ($\beta=0.8$)}
    \label{table:white_box_sensor_linf}
    \centering
    \scriptsize
    \resizebox{0.9\textwidth}{!}{
    \begin{tabular}{c|c|c|c|c|c|c}
        \toprule
             Models & & $\epsilon=0$ & $\epsilon=4$ & $\epsilon=8$ & $\epsilon=16$ & $\epsilon=32$ \\
             \hline
        \hline
            \multirow{2}{*}{GTSRB-CNN} & Main & $\bm{99.38}$ & $67.31$ & $43.13$ & $13.50$ & $3.63$ \\
            \cline{2-7} & KEMLP & $98.28(\tr{-1.10})$ & $85.39(\tg{+18.08})$ & $71.76(\tg{+28.63})$ & $48.89(\tg{+35.39})$ & $26.13(\tg{+22.50})$\\
            \hline
        \hline
            \multirow{2}{*}{AdvTrain ($\epsilon=4$)} & Main & $97.94$ & $87.94$ & $68.85$ & $38.66$ & $8.77$ \\
            \cline{2-7} & KEMLP & $97.89(\tr{-0.05})$ & $\bm{92.80}(\tg{+4.86})$ & $79.58(\tg{+10.73})$ & $57.48(\tg{+18.82})$ & $28.58(\tg{+19.81})$\\
            \hline
        \hline
            \multirow{2}{*}{AdvTrain ($\epsilon=8$)} & Main & $93.72$ & $84.21$ & $71.76$ & $43.16$ & $13.01$ \\
            \cline{2-7} & KEMLP & $96.79(\tg{+3.07})$ & $92.08(\tg{+7.87})$ & $81.58(\tg{+9.82})$ & $59.18(\tg{+16.02})$ & $30.61(\tg{+17.60})$\\
            \hline
        \hline
            \multirow{2}{*}{AdvTrain ($\epsilon=16$)} & Main & $84.54$ & $78.58$ & $71.89$ & $55.99$ & $19.55$ \\
            \cline{2-7} & KEMLP & $94.68(\tg{+10.14})$ & ${91.64}(\tg{+13.06})$ & $\bm{85.55}(\tg{+13.66})$ & $67.98(\tg{+11.99})$ & $32.61(\tg{+13.06})$\\
            \hline
        \hline
            \multirow{2}{*}{AdvTrain ($\epsilon=32$)} & Main & $74.74$ & $70.24$ & $65.61$ & $56.22$ & $29.04$ \\
            \cline{2-7} & KEMLP & $91.46(\tg{+16.72})$ & $88.58(\tg{+18.34})$ & $83.23(\tg{+17.62})$ & $\bm{72.02}(\tg{+15.80})$ & $\bm{41.90}(\tg{+12.86})$\\
            \hline
        \hline
            \multirow{2}{*}{DOA (5x5)} & Main & $97.43$ & $57.46$ & $28.76$ & $5.81$ & $0.85$ \\
            \cline{2-7} & KEMLP & $97.45(\tg{+0.02})$ & $83.85(\tg{+26.39})$ & $67.98(\tg{+39.22})$ & $45.27(\tg{+39.46})$ & $24.28(\tg{+23.43})$\\
            \hline
        \hline
            \multirow{2}{*}{DOA (7x7)} & Main & $97.27$ & $38.50$ & $9.75$ & $2.83$ & $0.67$ \\
            \cline{2-7} & KEMLP & $97.22(\tr{-0.05})$ & $80.89(\tg{+42.39})$ & $63.40(\tg{+53.65})$ & $49.20(\tg{+46.37})$ & $31.04(\tg{+30.37})$\\
            \hline
        \bottomrule
    
    \end{tabular}
    }
\end{table*}

\section{Experimental Evaluation}
\label{sec:experimental evals}

In this section, we evaluate \sys based on the traffic sign recognition task against different adversarial attacks and corruptions, including the physical attacks~\citep{eykholt2018robust}, $\mathcal{L}_{\infty}$ bounded attacks, unforeseen attacks~\citep{kang2019testing}, and common corruptions~\cite{hendrycks2019benchmarking}.
We show that under both whitebox and blackbox settings against a \textit{diverse} set of attacks, 1) \sys achieves significantly higher robustness than baselines, 2) \sys maintains similar clean accuracy with a strong main task model whose clean accuracy is originally high (e.g., vanillar CNN), 3) \sys even achieves higher clean accuracy than a relatively weak main task model whose clean accuracy is originally low as a tradeoff for its robustness (e.g., adversarially trained models). 


\subsection{Experimental Setup}
\label{subsec:experimental setup}

\paragraph{Dataset} Following existing work~\citep{eykholt2018robust,wu2019defending} that evaluate ML robustness on traffic sign data, we adopt LISA~\citep{mogelmose2012vision} and GTSRB~\citep{stallkamp2012man} for training and evaluation. All data are processed by standard crop-and-resize to $32 \times 32$ as described in~\citep{sermanet2011traffic}. In this paper, we conduct the evaluation on two dataset settings: \textit{1) Setting-A:} a subset of GTSRB, which contains 12 types of German traffic signs. In total, there are 14880 samples in the training set, 972 samples in the validation set, and 3888 samples in the test set; \textit{2) Setting-B:} a modified version of Setting-A, where the German stop signs are replaced with the U.S. stop signs from LISA, following \citep{eykholt2018robust}.

\paragraph{Models} We adopt the GTSRB-CNN architecture \citep{eykholt2018robust} as the main task model.  \sys is constructed based on the main task model together with a set of auxiliary task models (e.g., color, shape, and content detectors).
To train the weights of factors in \sys, we use $\beta$ to denote the prior belief on balance between benign and adversarial distributions. 
More details on implementation are provided in Appendix~\ref{appendix:implementation-details}.

\paragraph{Baselines} To demonstrate the superiority of \sys, we compare it with two state-of-the-art baselines: \textbf{adversarial training}~\citep{madry2017towards} and \textbf{DOA}~\citep{wu2019defending}, which are strong defenses against $\mathcal{L}_p$ bounded attacks and physically attacks respectively. Detailed setup for baselines is given in Appendix~\ref{appendix:detailed-setup-of-baselines}.

\paragraph{Evaluated Attacks and Corruptions} We consider four types of attacks for thorough evaluation: \textit{1) physical attacks} on stop signs~\citep{eykholt2018robust}; \textit{2) $\mathcal{L}_\infty$ bounded attacks}~\citep{madry2017towards} with $\epsilon\in\{4,8,16,32\}$; \textit{3) Unforeseen attacks}, which produce a diverse set of unforeseen test distributions (e.g. Elastic, JPEG, Fog) distinct from $\mathcal{L}_{p}$ bounded perturbation~\citep{kang2019testing}; \textit{4) common corruptions}~\citep{hendrycks2019benchmarking}. We present examples of these adversarial instances in Appendix~\ref{appendix:examples-of-generated-adversarial-test-samples}. 
For each attack, we consider both the \textit{whitebox attack} against the main task model and \textit{blackbox attack} by distilling either the main task model or the whole \sys pipeline.
More details can be found in Appendix~\ref{appendix:details-of-attacks-and-corruptions}.


\begin{table*}[th]
\vspace{-5mm}
    \caption{Accuracy~($\%$) under whitebox unforeseen attacks~($\beta=0.8$)}
    \label{table:white_box_sensor_unforeseen}
    \centering
    \resizebox{1.0\textwidth}{!}{
    \begin{tabular}{c|c|c|c|c|c|c|c|c|c|c|c|c}
        \toprule
             & & Clean & Fog-256 & Fog-512 & Snow-0.25 & Snow-0.75 & Jpeg-0.125 & Jpeg-0.25 & Gabor-20 & Gabor-40 & Elastic-1.5 & Elastic-2.0\\
            \hline
        \hline
            \multirow{2}{*}{GTSRB-CNN} & Main & $\bm{99.38}$ & $59.65$  & $34.18$  & $56.58$ & $24.54$ & $55.74$ & $27.01$ & $57.25$ & $32.41$ & $44.78$ & $24.31$ \\
            \cline{2-13} & KEMLP & ${98.28}(\tr{-1.10})$ & $76.95(\tg{+17.30})$  & $62.83(\tg{+28.65})$  & $78.94(\tg{+22.36})$ & $53.22(\tg{+28.68})$ & $79.63(\tg{+23.89})$ & $63.40(\tg{+36.39})$ & $80.17(\tg{+22.92})$ & $65.20(\tg{+32.79})$ & $69.34(\tg{+24.56})$ & $52.37(\tg{+28.06})$\\
            \hline
        \hline
            \multirow{2}{*}{AdvTrain ($\epsilon=4$)} & Main & $97.94$ & $55.53$ & $29.50$ & $66.31$ & $32.61$ & $56.58$ &$28.11$ & $73.30$ & $46.76$ & $57.25$ & $30.09$\\
            \cline{2-13} & KEMLP & ${97.89}(\tr{-0.05})$ & $76.08(\tg{+20.55})$  & $61.96(\tg{+32.46})$  & $80.45(\tg{+14.14})$ & $57.84(\tg{+25.23})$ & $84.23(\tg{+27.65})$ & $68.57(\tg{+40.46})$ & $81.48(\tg{+8.18})$ & $65.77(\tg{+19.01})$ & $71.19(\tg{+13.94})$ & $50.33(\tg{+20.24})$\\
            \hline
        \hline
            \multirow{2}{*}{AdvTrain ($\epsilon=8$)} & Main & $93.72$ & $50.03$ & $23.56$ & $63.71$ & $34.93$ & $57.56$ & $26.16$ &$76.72$ & $53.76$ & $48.25$ & $24.46$\\
            \cline{2-13} & KEMLP & ${96.79}(\tg{+3.07})$ & $76.59(\tg{+26.56})$  & $63.97(\tg{+40.41})$  & $81.40(\tg{+17.69})$ & $57.07(\tg{+22.14})$ & $85.11(\tg{+27.55})$ & $68.70(\tg{+42.54})$ & $85.29(\tg{+8.57})$ & $68.90(\tg{+15.14})$ & $68.78(\tg{+20.53})$ & $49.31(\tg{+24.85})$\\
            \hline
        \hline
            \multirow{2}{*}{AdvTrain ($\epsilon=16$)} & Main & $84.54$ & $47.92$ & $19.75$ & $66.46$ & $37.60$ & $66.56$ &$34.23$ & $78.01$ & $64.33$ & $55.48$ & $32.28$\\
            \cline{2-13} & KEMLP & ${94.68}(\tg{+10.14})$ & $77.13(\tg{+29.21})$  & $64.38(\tg{+44.63})$  & $\bm{81.64}(\tg{+15.18})$ & $58.20(\tg{+20.60})$ & $\bm{86.99}(\tg{+20.43})$ & $70.40(\tg{+36.17})$ & $\bm{87.42}(\tg{+9.41})$ & $72.61(\tg{+8.28})$ & $67.31(\tg{+11.83})$ & $50.28(\tg{+18.00})$\\
            \hline
        \hline
            \multirow{2}{*}{AdvTrain ($\epsilon=32$)} & Main & $74.74$ & $48.71$ & $22.84$ & $61.78$ & $38.91$ & $63.58$ &$43.49$ & $70.37$ & $65.20$ & $54.58$ & $39.45$\\
            \cline{2-13} & KEMLP & ${91.46}(\tg{+16.72})$ & $\bm{79.22}(\tg{+30.51})$  & $\bm{66.33}(\tg{+43.49})$  & ${81.20}(\tg{+19.42})$ & $\bm{64.53}(\tg{+25.62})$ & $86.70(\tg{+23.12})$ & $\bm{73.38}(\tg{+29.89})$ & $87.04(\tg{+16.67})$ & $\bm{74.92}(\tg{+9.72})$ & $66.38(\tg{+11.80})$ & $54.76(\tg{+15.31})$\\
            \hline
        \hline
            \multirow{2}{*}{DOA (5x5)} & Main & $97.43$ & $58.00$ & $32.69$ & $61.19$ & $28.34$ & $41.13$ & $11.29$ & $55.43$ & $29.55$ & $58.02$ & $32.74$\\
            \cline{2-13} & KEMLP & ${97.45}(\tg{+0.02})$ & $76.85(\tg{+18.85})$  & $63.07(\tg{+30.38})$  & $78.78(\tg{+17.59})$ & $56.76(\tg{+28.42})$ & $78.60(\tg{+37.47})$ & $61.78(\tg{+50.49})$ & $80.25(\tg{+24.82})$ & $63.89(\tg{+34.34})$ & $\bm{72.69}(\tg{+14.67})$ & $\bm{57.51}(\tg{+24.77})$\\
            \hline
        \hline
            \multirow{2}{*}{DOA (7x7)} & Main & $97.27$ & $59.88$ & $38.01$ & $62.47$ & $30.17$ & $23.46$ & $3.65$ & $54.58$ & $27.29$ & $56.33$ & $30.97$\\
            \cline{2-13} & KEMLP & ${97.22}(\tr{-0.05})$ & ${78.09}(\tg{+18.21})$  & ${62.76}(\tg{+24.75})$  & $79.68(\tg{+17.21})$ & $58.26(\tg{+28.09})$ & $74.25(\tg{+50.79})$ & $61.39(\tg{+57.74})$ & $79.06(\tg{+24.48})$ & $62.29(\tg{+35.00})$ & $71.27(\tg{+14.94})$ & $55.09(\tg{+24.12})$\\
            \hline
        \bottomrule
    \end{tabular}
    }
\end{table*}

\begin{table}[tp]
\vspace{-3mm}
    \caption{Accuracy~($\%$) under common corruptions~($\beta=0.2$)}
    \label{table:common_corruption}
    \centering
    \scriptsize
    \resizebox{0.5\textwidth}{!}{
    \begin{tabular}{c|c|c|c|c|c}
        \toprule
             & & Clean & Fog & Contrast & Brightness \\
             \hline
        \hline
            \multirow{2}{*}{GTSRB-CNN} & Main & $\bm{99.38}$ & $76.23$ & $57.61$ & $85.52$ \\
            \cline{2-6} & KEMLP & ${98.28}(\tr{-1.10})$ & $\bm{78.14}(\tg{+1.91})$ & $72.43(\tg{+14.82})$ & $\bm{89.58}(\tg{+4.06})$ \\
            \hline
        \hline
            \multirow{2}{*}{AdvTrain ($\epsilon=4$)} & Main & $97.94$ & $63.81$ & $42.31$ & $78.47$  \\
            \cline{2-6} & KEMLP & ${97.89}(\tr{-0.05})$ & $70.29(\tg{+6.48})$ & $67.46(\tg{+25.16})$ & $86.70(\tg{+8.23})$ \\
            \hline
        \hline
            \multirow{2}{*}{AdvTrain ($\epsilon=8$)} & Main & $93.72$ & $59.05$ & $31.97$ & $78.47$  \\
            \cline{2-6} & KEMLP & ${96.79}(\tg{+3.07})$ & $67.41(\tg{+8.36})$ & $66.69(\tg{+34.72})$ & $85.91(\tg{+7.44})$ \\
            \hline
        \hline
            \multirow{2}{*}{AdvTrain ($\epsilon=16$)} & Main & $84.54$ & $56.58$ & $34.31$ & $78.01$  \\
            \cline{2-6} & KEMLP & ${94.68}(\tg{+10.14})$ & $66.80(\tg{+10.22})$ & $68.39(\tg{+34.08})$ & $86.14(\tg{+8.13})$ \\
            \hline
        \hline
            \multirow{2}{*}{AdvTrain ($\epsilon=32$)} & Main & $74.74$ & $50.87$ & $30.45$ & $71.30$  \\
            \cline{2-6} & KEMLP & ${91.46}(\tg{+16.72})$ & $64.94(\tg{+14.07})$ & $68.31(\tg{+37.86})$ & $83.20(\tg{+11.90})$ \\
            \hline
        \hline
            \multirow{2}{*}{DOA (5x5)} & Main & $97.43$ & $73.95$ & $62.24$ & $83.92$  \\
            \cline{2-6} & KEMLP & ${97.45}(\tg{+0.02})$ & $76.08(\tg{+2.13})$ & $\bm{74.38}(\tg{+12.14})$ & $87.60(\tg{+3.68})$ \\
            \hline
        \hline
            \multirow{2}{*}{DOA (7x7)} & Main & $97.27$ & $73.41$ & $57.54$ & $83.56$  \\
            \cline{2-6} & KEMLP & ${97.22}(\tr{-0.05})$ & $76.00(\tg{+2.59})$ & $72.40(\tg{+14.86})$ & $87.78(\tg{+4.22})$ \\
            \hline
        \bottomrule
    \end{tabular}
    }
\end{table}

\subsection{Evaluation Results}
\label{subsec:evaluation results}
Here we compare the clean accuracy, robust accuracy, and weighted robustness (W-Robust Accuracy) for baselines and \sys under different attacks and settings.

\paragraph{Clean accuracy of \sys}
First, we present the clean accuracy of \sys and baselines in Figure~\ref{fig:general_accuracy} (a) and Tables~\ref{table:stop_sign_attack}--\ref{table:common_corruption}.
As demonstrated, the clean accuracy of \sys is generally high~(over $90\%$),  by either maintaining the high clean accuracy of strong main task models (e.g., vanilla DNN) or improving upon the weak main task models with relatively low clean accuracy (e.g., adversarially trained models).
It is clear that \sys can relax the tradeoff between benign and robust accuracy and maintain the high performance for both via knowledge integration.

\paragraph{Robustness against diverse attacks} 
We then present the robustness of \sys based on different main task models against the physical attacks, which is very challenging to defend currently ( Table~\ref{table:stop_sign_attack}), $\ell_p$ bounded attacks ( Table~\ref{table:white_box_sensor_linf}), unseen attacks (Table~\ref{table:white_box_sensor_unforeseen}), and common corruptions  (Table~\ref{table:common_corruption}) under whitebox attack setting.
The corresponding results for blackbox setting can be found in Appendix~\ref{appendix:additional_experiment_results}. 
From the tables, we observe that \sys achieves significant \textit{robustness gain} over baselines.
Note that although adversarial training improves the robustness against $\mathcal{L}_{\infty}$ attacks and DOA helps to defend against physical attacks, they are not robust to other types of attacks or corruptions. In contrast, \sys presents general robustness against a range of attacks and corruptions without further adaptation.

\paragraph{Performance stability of \sys}
We conduct additional ablation studies on $\beta$,
representing the prior belief on the benign and adversarial distribution balance. We set $\beta=0.5$ for \sys indicating a balanced random guess for the distribution tradeoff. We show the clean accuracy and robustness of \sys and baselines under diverse 46 attacks in Figure~\ref{fig:general_accuracy}.
We can see that \sys consistently and significantly outperforms the baselines, which indicates the performance stability of \sys regarding different distribution ratio $\beta$. More results can be found in Appendix~\ref{appendix:additional_experiment_results}.

\section{Discussions and Future Work}
In this paper, we propose \sys, which integrates \textit{domain knowledge} with a set of weak auxiliary models to enhance the ML robustness against a diverse set of adversarial attacks and corruptions.
While our framework can be extended to other applications, for any knowledge system, one naturally needs domain experts to design the knowledge rules specific to that application. 
Here we aim to introduce this framework as a prototype, provide a rigorous analysis of it, and demonstrate the benefit of such construction on an application. Nevertheless, there is probably no universal strategy on how to aggregate knowledge for any arbitrary application, and instead, application-specific constructions are needed.  
We do believe that, once the principled framework of knowledge fusion is ready, application-specific developments of knowledge rules will naturally follow, similar to what happened previously for knowledge-enriched joint inference.




\section*{Acknowledgements}
CZ and the DS3Lab gratefully acknowledge the support from the Swiss National Science Foundation (Project Number 200021\_184628), Innosuisse/SNF BRIDGE Discovery (Project Number 40B2-0\_187132), European Union Horizon 2020 Research and Innovation Programme (DAPHNE, 957407), Botnar Research Centre for Child Health, Swiss Data Science Center, Alibaba, Cisco, eBay, Google Focused Research Awards, Oracle Labs, Swisscom, Zurich Insurance, Chinese Scholarship Council, and the Department of Computer Science at ETH Zurich.
BL and the SLLab would like to acknowledge the support from
NSF grant No.1910100, NSF CNS 20-46726 CAR, and Amazon Research Award.

\bibliographystyle{plainnat}
\bibliography{icml2021_conference}
\clearpage
\pagebreak
\onecolumn
\appendix
\label{sec:appendix}
\pagebreak
\section{Proofs}\label{sec: proofs}
\subsection{Preliminaries}\label{subsec: preliminaries}
For completeness, here we recall our setup and introduce further remarks.
\paragraph{Data model}
We begin by recalling our notation. We consider a classification problem under supervised learning setting, defined on a feature space $\mathcal{X}$ and a finite label space $\mathcal{Y}$. We refer to $x\in \mathcal{X}$ as an input, and $y\in \mathcal{Y}$ as the prediction. 
An input $x$ can be a benign example or an
adversarial example. To model this,
we use $z \in \{0, 1\}$, a latent variable which is 
not exposed to \sys.
That is, $x$ is an adversarial example with $(x, y)\sim \adv$ whenever $z=1$, and $(x, y)\sim \benign$ otherwise, where $\adv$ and $\benign$ represent the adversarial and benign data distribution.
We let $\pi_{\adv}=\prob(z=1)$ and $\pi_{\benign}=\prob(z=0)$, implying $\pi_{\adv}+\pi_{\benign}=1$. 
For convenience, we denote $\prob_{\adv}(x, y)=\prob(x, y|z=1)$ and $\prob_{\benign}(x, y)=\prob(x, y|z=0)$.

For simplicity, we describe 
the \sys framework as a binary classification task, in which case $\mathcal{Y}=\{0, 1\}$, noting that the multi-class scenario is a simple extension of it. We introduce the \sys framework as follows.

\paragraph*{Knowledge Integration}
Given 
a data example $(x, y)\sim{\benign}$ or $(x, y)\sim{\adv}$, $y$ is unknown to 
\sys. 
We create 
a factor graph to embed the domain knowledge as follows. 
The outputs of each  model over 
$x$ 
become \textit{input variables}:
$s_*, 
s_{\mathcal{I}}=\{s_i: i \in \mathcal{I}\}, 
s_{\mathcal{J}}=\{s_j: j \in \mathcal{J}\}$.
\sys also has an output variable  $\outputvariable \in \mathcal{Y}$, which corresponds to 
its prediction.
Different models introduce
different types of factors 
connecting these variables:
\begin{itemize}[leftmargin=*, wide = 0pt]
\item Main model:
\sys introduces a factor between 
the main model $\mainsensor$
and the output variable $\outputvariable$ with
factor function $f_{*}(\outputvariable, \mainsensor) = \indicator\{\outputvariable=\mainsensor\}$;
\item Permissive model:
\sys introduces a factor between
each permissive model $s_i$ and
the output variable $\outputvariable$ with
factor function 
$f_{i}(\outputvariable, s_i) = \indicator\{s_i \imply \outputvariable\}$.
\item Preventative model:
\sys introduces a factor between
each preventative model $s_j$ and 
the output variable $\outputvariable$
with factor function $f_{j}(\outputvariable, s_j) = \indicator\{\outputvariable \imply s_j\}$.
\end{itemize}

\paragraph*{Learning with \sys}
To make a prediction, \sys 
outputs the \textit{probability}
of the output variable $o$.
\sys assigns a weight for each
model and constructs the 
following log-linear statistical model:
\begin{align*}\mathbb{P}[o| s_*,
s_{\mathcal{I}}, s_{\mathcal{J}}, w_*, w_{\mathcal{I}}, w_{\mathcal{J}}]\ \ \propto \ \ \exp
\{
b_{o} + w_* f_*(o, s_*) \} \times \exp\big\{
\sum_{i \in \mathcal{I}}
w_i f_i(o, s_i)\big\} \times \exp\big\{
\sum_{j \in \mathcal{J}}
w_j f_j(o, s_j)
\big\}
\end{align*}
where 
$w_*, w_i, w_j$ are
the corresponding weights 
for models $s_*, s_i, s_j$, $w_{\mathcal{I}}=\{w_i: i \in \mathcal{I}\}, 
w_{\mathcal{J}}=\{w_j: j \in \mathcal{J}\}$ and $b_o$ is some bias parameter that depends on $o$.
For the simplicity of exposition,
we use an equivalent notation 
by putting all the weights and  outputs of factor functions
into vectors using an ordering of models. More precisely, we define

\begin{align*}
 \mathbf{w}&=[1;w_*; ( w_i)_{i\in\ci}; (w_j)_{j\in\cj}],\\
 \mathbf{f}_{o}(s_*,
s_{\mathcal{I}}, s_{\mathcal{J}})&=[b_o; f_*(o, s_*); (f_i(o, s_i))_{i\in\ci}; (f_j(o, s_j))_{j\in\cj}],
\end{align*}
for $o\in\mathcal{Y}$.
All concatenated vectors from above are in $\mathbb{R}^{m+n+2}$. Given this, 
an equivalent form of \sys's 
statistical model is
\begin{align}
\mathbb{P}[o | s_*,
s_{\mathcal{I}}, s_{\mathcal{J}}, \mathbf{w}] 
=\frac{1}{Z_{\bf w}} 
\exp(\langle{\bf w}, \mathbf{f}_{o}(s_*,
s_{\mathcal{I}}, s_{\mathcal{J}})\rangle)
\end{align}
where $Z_{\bf w}$ is the normalization constant over $o\in \mathcal{Y}$ such that
\[
Z_{\bf w} = \exp\big(\langle{\bf w}, \mathbf{f}_{0}(s_*, s_{\mathcal{I}}, s_{\mathcal{J}})\rangle\big)+ \exp\big(\langle{\bf w}, \mathbf{f}_{1}(s_*,s_{\mathcal{I}}, s_{\mathcal{J}})\rangle\big).
\]
With some abuse of notation, $\mathbf{w}$ is meant to govern all parameters including weights and biases whenever used with probabilities.

\paragraph*{Weight Learning}
During the training phase of \sys, we choose parameters {\bf w} by performing standard maximum likelihood estimation over a training dataset. Given a particular input instance $x^{\scriptscriptstyle(n)}$, respective model predictions $s_*^{\scriptscriptstyle(n)},
s^{\scriptscriptstyle(n)}_{\mathcal{I}}, s^{\scriptscriptstyle(n)}_{\mathcal{J}}$, and the ground truth label $y^{\scriptscriptstyle(n)}$, we minimize the negative log-likelihood function in view of

\[
 \hat{{\bf w}}=\argmin_{{\bf w}}\Big\{- \sum_n \log \Big( \mathbb{P}[o^{\scriptscriptstyle(n)}=y^{\scriptscriptstyle(n)} |s_*^{\scriptscriptstyle(n)},
s^{\scriptscriptstyle(n)}_{\mathcal{I}}, s^{\scriptscriptstyle(n)}_{\mathcal{J}}, \mathbf{w}]\Big) \Big\}.
\]

\paragraph*{Inference}
During the inference phase of \sys, given an input example $\hat x$, we predict $\hat y$ that has the largest probability given the respective model predictions $\hat s_*,
\hat s_{\mathcal{I}}, \hat  s_{\mathcal{J}}$, namely, $\hat y=\argmax_{\tilde y\in\mathcal{Y}} \mathbb{P}[o =\tilde y| \hat  s_*, \hat s_{\mathcal{I}},\hat  s_{\mathcal{J}}, \hat{\mathbf{w}}]$.

\paragraph*{Weighted Robust Accuracy}

Previous theoretical analysis on 
ML robustness~\cite{javanmard2020precise,xu2009robustness,raghunathan2020understanding} have identified 
two natural dimensions of model 
quality: \textit{clean accuracy}
and \textit{robust accuracy},
which are the accuracy of a given ML
model on inputs $x$ drawn from 
either the benign distribution 
$\mathcal{D}_b$ or 
 adversarial distribution $\mathcal{D}_a$.
In this paper, to balance their tradeoff, we use their weighted
average as
our main metric of interest. That is, 
given a classifier $h: \mathcal{X} \rightarrow \mathcal{Y}$ we define its \emph{\Acc} as
\begin{align*}
\mathcal{A}_h = \pi_{D_a} 
   \mathbb{P}_{\mathcal{D}_a}[h(x)=y]
   +
   \pi_{D_b} 
   \mathbb{P}_{\mathcal{D}_b}[h(x)=y].
\end{align*}
We use 
$\mathcal{A}^{\text{\sys}}$ and 
$\mathcal{A}^{\text{main}}$ to denote 
the weighted robust accuracies of 
\sys and main task model, respectively.

\paragraph*{Modeling Assumptions} We assume that for a fixed $z$, that is, for a fixed $\distset$, the models make independent errors given the target variable $y$. Thus, for all $\distset$ the class conditional distribution can be decomposed as
\[
\prob_{\dist}[s_*,
s_{\mathcal{I}}, s_{\mathcal{J}}|y]=\prob_{\dist}[s_*|y]\prod_{i\in\mathcal{I}}\prob_{\dist}[s_{i}|y]\prod_{j\in\mathcal{J}}\prob_{\dist}[s_{j}|y].
\]
We also assume for simplicity that the main task model makes symmetric errors given the class of target variable, that is, ${\prob_{\dist}[s_*\neq y|y]}$ is fixed with respect to $y$ for all $\distset$. 

\paragraph*{Characterizing Models: Truth Rate ($\alpha$) and False Rate ($\epsilon$)}
Each auxiliary 
model $k \in \mathcal{I} \cup \mathcal{J}$ is characterized 
by two values, their
truth rate ($\alpha$) and
false rate 
($\epsilon$)
over benign and adversarial 
distributions. These
values measure the 
\textit{consistency} 
of the model with 
the ground truth:
\begin{align*}
& \text{Permissive Models:} \\
& \alpha_{i,\dist}:=\prob_{\dist}[s_i=y|y=1], \hspace{.6em} \epsilon_{i,\dist}:=\prob_{\dist}[s_i\neq y|y=0]
\end{align*}
\vspace{-2em}
\begin{align*}
& \text{Preventative Models:} \\
& \alpha_{j,\dist}:=\prob_{\dist}[s_j=y|y=0], \hspace{.6em} \epsilon_{j,\dist}:=\prob_{\dist}[s_j\neq y|y=1]
\end{align*}

Note that, given the 
asymmetric nature 
of these auxiliary models,
we do \textit{not}
necessarily have 
$\epsilon_{k, \mathcal{D}}
= 1 - \alpha_{k, \mathcal{D}}$.
In addition, for a high quality permissive 
model ($k\in\ci)$, or a high quality preventative 
model ($k\in\cj$)
for which the logic rules mostly 
hold, we expect $\alpha_{k,\dist}$ to be large and $\epsilon_{k,\dist}$ to be small. 

We define the truth rate of main model over data examples drawn from $\distset$ as $\alpha_{*,\dist}:=\prob_{\dist}(s_*=y)$, and its false rate as $\epsilon_{*,\dist}:=\prob_{\dist}(s_*\neq y)=1-\alpha_{*,\dist}$. 

These characteristics are of integral importance to weighted robust accuracy of \sys. To combine all the models together, we define upper and lower bounds to truth rates and false rates. For the main model, we have ${}_{\scriptscriptstyle \wedge}\alpha_{*}:=\min_{\dist}\alpha_{*,\dist}$ and ${}_{\scriptscriptstyle \vee}\alpha_{*}:=\max_{\dist}\alpha_{*,\dist}$. whereas for auxiliary models, for each model index $k \in \ci \cup \cj$, we have
\begin{align*}
   & {}_{\scriptscriptstyle \wedge}\alpha_{k}:=\min_{\dist}\alpha_{k,\dist}, \hspace{.8em} {}_{\scriptscriptstyle \wedge}\epsilon_{k}:=\min_{\dist}\epsilon_{k,\dist}\\
   & {}_{\scriptscriptstyle \vee}\alpha_{k}:=\max_{\dist}\alpha_{k,\dist}, \hspace{.8em} {}_{\scriptscriptstyle \vee}\epsilon_{k}:=\max_{\dist}\epsilon_{k,\dist}.
\end{align*}

\subsection{Parameters}\label{subsec: generative model parameters}
In this section we will derive the closed-form expressions for the parameters based on our generative model, namely, weights and biases.

To make a prediction, \sys 
outputs the \textit{marginal probability}
of the output variable $o$.
\sys assigns a weight for each
model and constructs the 
following statistical model:
\begin{align*}
& \mathbb{P}[o | s_*,
s_{\mathcal{I}}, s_{\mathcal{J}}, \mathbf{w}] \ \ \propto \ \
\exp
\{
b_{o} + w_* f_*(o, s_*) \} \times \exp\big\{
\sum_{i \in \mathcal{I}}
w_i f_i(o, s_i)\big\} \times \exp\big\{
\sum_{j \in \mathcal{J}}
w_j f_j(o, s_j)
\big\},
\end{align*}
where 
$w_*, w_i, w_j$ are
the corresponding weights 
for models $s_*, s_i, s_j$, and $b_o$ is some bias parameter that depends on $o$.
For the simplicity of exposition,
we use an equivalent notation 
by putting all the weights and  outputs of factor functions
into vectors using an ordering of models. More precisely, we define

{\vspace{-2em}
\begin{align*}
 \mathbf{w}&=[1;w_*; ( w_i)_{i\in\ci}; (w_j)_{j\in\cj}],\\
 \mathbf{f}_{o}(s_*,
s_{\mathcal{I}}, s_{\mathcal{J}})&=[b_o; f_*(o, s_*); (f_i(o, s_i))_{i\in\ci}; (f_j(o, s_j))_{j\in\cj}],
\end{align*}
}\normalsize
for $o\in\mathcal{Y}$.
All concatenated vectors from above 
are in $\mathbb{R}^{m+n+2}$. Given this, 
an equivalent form of \sys's 
statistical model is
\begin{align}\label{eqn: class likelihoods (appendix)}
\mathbb{P}[o | s_*,
s_{\mathcal{I}}, s_{\mathcal{J}}, \mathbf{w}] 
=\frac{1}{Z_{\bf w}} 
\exp(\langle{\bf w}, \mathbf{f}_{o}(s_*,
s_{\mathcal{I}}, s_{\mathcal{J}})\rangle),
\end{align}
where $Z_{\bf w}$ is the normalization constant over $o\in \mathcal{Y}$. We can further show that
\begin{equation}\label{eqn: class conditional output}
    \begin{split}
         \prob[o=\tilde{y}|\sensors, \mathbf{w}]&=\frac{\prob[o=\tilde{y}|\sensors, \mathbf{w}]}{\prob[o=\tilde{y}|\sensors, \mathbf{w}]+\prob[o=1-\tilde{y}|\sensors, \mathbf{w}]}\\
    &=\frac{\exp(\langle{\bf w}, \mathbf{f}_{y}(s_*,
    s_{\mathcal{I}}, s_{\mathcal{J}})\rangle)}{\exp(\langle{\bf w}, \mathbf{f}_{\tilde{y}}(s_*,
    s_{\mathcal{I}}, s_{\mathcal{J}})\rangle)+\exp(\langle{\bf w}, \mathbf{f}_{1-\tilde{y}}(s_*,
    s_{\mathcal{I}}, s_{\mathcal{J}})\rangle)}\\
    &=\frac{1}{1+\exp(-\Delta_{\bf w} (\tilde{y},s_*,s_{\mathcal{I}}, s_{\mathcal{J}}))}
    \end{split}
\end{equation}

where $\Delta_{\bf w} (\tilde{y},s_*,s_{\mathcal{I}}, s_{\mathcal{J}})$ is previously defined as
\[
\Delta_{\bf w} (\tilde{y},s_*,s_{\mathcal{I}}, s_{\mathcal{J}}):=\langle{\bf w}, \mathbf{f}_{\tilde{y}}(s_*,s_{\mathcal{I}}, s_{\mathcal{J}})-\mathbf{f}_{1-\tilde{y}}(s_*,s_{\mathcal{I}}, s_{\mathcal{J}})\rangle.
\]

Therefore, we have
\begin{equation}\label{eqn: sigmoid representation}
  \prob[o=\tilde{y}|\sensors, \mathbf{w}]=\sigma(\Delta_{\bf w} (\tilde{y},s_*,s_{\mathcal{I}}, s_{\mathcal{J}}))  
\end{equation}
where $\sigma: \mathbb{R}\mapsto [0, 1]$ is the Sigmoid function.

\begin{remark}[Closed form expression of $\Delta_{\bf w} (\tilde{y},s_*,s_{\mathcal{I}}, s_{\mathcal{J}})$]\label{remark: closed form delta}
Recalling our knowledge integration rules, it can be shown that
\begin{align*}
  \Delta_{\bf w} (\tilde{y},s_*,s_{\mathcal{I}}, s_{\mathcal{J}})&=
  \langle{\bf w}, \mathbf{f}_{\tilde{y}}(s_*,s_{\mathcal{I}}, s_{\mathcal{J}})-\mathbf{f}_{1-\tilde{y}}(s_*,s_{\mathcal{I}}, s_{\mathcal{J}})\rangle\\
  &=b({\tilde{y}})+w_*\big(f_*(\tilde{y},s_*)-f_*(1-\tilde{y},s_*)\big)+\sum_{i\in \ci}w_i\big(f_i(\tilde{y},s_i)-f_i(1-\tilde{y},s_i)\big)\\
  &+\sum_{j\in \cj}w_i\big(f_j(\tilde{y},s_j)-f_j(1-\tilde{y},s_j)\big)
\end{align*}
where $b({\tilde{y}})=b_{\tilde{y}}-b_{1-\tilde{y}}$. Let $b:=b_1-b_0$. Then $b({\tilde{y}})=(2\tilde{y}-1)b$.

Using the logical rules, we moreover have
\begin{align*}
    &f_*(\tilde{y},s_*)-f_*(1-\tilde{y},s_*) =\indicator\{\tilde{y}=s_*\}-\indicator\{1-\tilde{y}=s_*\}=(2\tilde{y}-1)(2s_*-1)\\
    &f_i(\tilde{y},s_i)-f_i(1-\tilde{y},s_i) =\indicator\{s_i\implies \tilde{y}\}-\indicator\{s_i\implies 1-\tilde{y}\}=(2\tilde{y}-1)s_i\\
    &f_j(\tilde{y},s_j)-f_j(1-\tilde{y},s_j) =\indicator\{\tilde{y}\implies s_j\}-\indicator\{1-\tilde{y}\implies s_j\}=(2\tilde{y}-1)(s_j-1)=-(2\tilde{y}-1)(1-s_j).
\end{align*}

Therefore, the closed form expression for $\Delta_{\bf w} (\tilde{y},s_*,s_{\mathcal{I}}, s_{\mathcal{J}})$ is given by
\begin{align*}
    \Delta_{\bf w} (\tilde{y},s_*,s_{\mathcal{I}}, s_{\mathcal{J}})=(2\tilde{y}-1)\Big(b+w_*(2s_*-1)+\sum_{i\in \ci}w_is_i-\sum_{j\in \cj} w_j(1-s_j)\Big)
\end{align*}
\end{remark}

\begin{remark}[Optimal parameters]\label{remark: optimal parameters derivation}
We now analyze the class conditional distribution $\prob[y|\sensors]$. Optimal set of parameters for our generative model must satisfy:
\begin{equation}\label{eqn: class conditional y}
    \begin{split}
            &\prob[y=\tilde{y}|\sensors]=\frac{\prob[y=\tilde{y},\sensors]}{\prob[\sensors]}=\frac{\prob[y=\tilde{y},\sensors]}{\prob[y=\tilde{y}, \sensors]+\prob[y=1-\tilde{y}, \sensors]}\\
    &=\frac{1}{1+\frac{\prob[y=1-\tilde{y},\sensors]}{\prob[y=\tilde{y},\sensors]}}
    =\frac{1}{1+\exp\big(\log\frac{\prob[y=1-\tilde{y},\sensors]}{\prob[y=\tilde{y},\sensors]}\big)}=\frac{1}{1+\exp\big(-\log\frac{\prob[y=\tilde{y},\sensors]}{\prob[y=1-\tilde{y},\sensors]}\big)}.
    \end{split}
\end{equation}
Note that, the optimal parameters satisfy
\[
\prob[o=\tilde{y}|\sensors]=\prob[y=\tilde{y}|\sensors].
\]
Hence, combining (\ref{eqn: class conditional output}) and (\ref{eqn: class conditional y}) as well as Remark~\ref{remark: closed form delta} we further have
\begin{equation}\label{eqn: equating y and o}
    \begin{split}
       \log\frac{\prob[y=\tilde{y},\sensors]}{\prob[y=1-\tilde{y},\sensors]} = (2\tilde{y}-1)\Big(b+w_*(2s_*-1)+\sum_{i\in \ci}w_is_i-\sum_{j\in \cj} w_j(1-s_j)\Big).
    \end{split}
\end{equation}
\end{remark}
Above remark indicates the condition that the optimal parameters must satisfy.

\subsection{Proof of Lemma~\ref{lemma: expected tail}}\label{subsec: proof of lemma}
Recall that for each model index $k \in \ci \cup \cj$ we define upper and lower bounds to truth rates and false rates as 
\begin{align*}
   & {}_{\scriptscriptstyle \wedge}\alpha_{k}:=\min_{\dist}\alpha_{k,\dist}, \hspace{.8em} {}_{\scriptscriptstyle \wedge}\epsilon_{k}:=\min_{\dist}\epsilon_{k,\dist}\\
   & {}_{\scriptscriptstyle \vee}\alpha_{k}:=\max_{\dist}\alpha_{k,\dist}, \hspace{.8em} {}_{\scriptscriptstyle \vee}\epsilon_{k}:=\max_{\dist}\epsilon_{k,\dist}.
\end{align*}

Next, we revisit Lemma~\ref{lemma: expected tail} towards its proof.

\begin{lemma*}[\recallthm]
Let $\Delta_{\bf w}$ be a random variable defined above. Suppose that \sys uses optimal parameters {\bf w} such that $\prob[y|s_*,s_{\mathcal{I}}, s_{\mathcal{J}}]=\prob[o|s_*,s_{\mathcal{I}}, s_{\mathcal{J}}, \mathbf{w}]$. Let also $r_y$ denote the log-ratio of class imbalance $\log\frac{\prob[y=1]}{\prob[y=0]}$. For a fixed $y\in\mathcal{Y}$ and $\distset$, one has
\begin{align*}
 \mathbb{E}_{s_*,s_{\mathcal{I}}, s_{\mathcal{J}}}[\Delta_{\bf w}(y,s_*,s_{\mathcal{I}}, s_{\mathcal{J}})|y] \nonumber \geq  \mu_{d_{*,\dist}}+ y \mu_{d_{\mathcal{I},\dist}} + (1-y) \mu_{d_{\mathcal{J},\dist}}+(2y-1)r_y:=\mu_{y,\dist},
\end{align*} 
where
\begin{equation*}
    \mu_{d_{*,\dist}}=\alpha_{*, \dist}\log\frac{{}_{\scriptscriptstyle\wedge}\alpha_{*}}{1-{}_{\scriptscriptstyle\wedge}\alpha_{*}}+(1-\alpha_{*, \dist})\log\frac{1-{}_{\scriptscriptstyle\vee}\alpha_{*}}{{}_{\scriptscriptstyle\vee}\alpha_{*}},
\end{equation*}
\begin{equation*}
    \begin{split}
        \mu_{d_{\mathcal{I},\dist}}=\sum_{i\in \mathcal{I}}\alpha_{i, \dist}\log\frac{{}_{\scriptscriptstyle\wedge}\alpha_{i}}{{}_{\scriptscriptstyle\vee}\epsilon_{i}}+(1-\alpha_{i, \dist})\log\frac{1-{}_{\scriptscriptstyle\vee}\alpha_{i}}{1-{}_{\scriptscriptstyle\wedge}\epsilon_{i}}-\sum_{j\in \mathcal{J}}\epsilon_{j, \dist}\log\frac{{}_{\scriptscriptstyle\vee}\alpha_{j}}{{}_{\scriptscriptstyle\wedge}\epsilon_{j}}-(1-\epsilon_{j, \dist})\log\frac{1-{}_{\scriptscriptstyle\wedge}\alpha_{j}}{1-{}_{\scriptscriptstyle\vee}\epsilon_{j}},
    \end{split}
\end{equation*}
and
\begin{equation*}
    \begin{split}
        \mu_{d_{\mathcal{J},\dist}}=\sum_{j\in \mathcal{J}}\alpha_{j, \dist}\log\frac{{}_{\scriptscriptstyle\wedge}\alpha_{j}}{{}_{\scriptscriptstyle\vee}\epsilon_{j}}+(1-\alpha_{j, \dist})\log\frac{1-{}_{\scriptscriptstyle\vee}\alpha_{j}}{1-{}_{\scriptscriptstyle\wedge}\epsilon_{j}}-\sum_{i\in \mathcal{I}}\epsilon_{i, \dist}\log\frac{{}_{\scriptscriptstyle\vee}\alpha_{i}}{{}_{\scriptscriptstyle\wedge}\epsilon_{i}}-(1-\epsilon_{i, \dist})\log\frac{1-{}_{\scriptscriptstyle\wedge}\alpha_{i}}{1-{}_{\scriptscriptstyle\vee}\epsilon_{i}}.
    \end{split}
\end{equation*}
\end{lemma*}

\begin{proof}[Proof of Lemma~\ref{lemma: expected tail}]
We show earlier that the optimal parameters satisfy (\ref{eqn: equating y and o}).
Note that the probabilities on the left hand side of~(\ref{eqn: equating y and o}) are mixtures over both the benign and adversarial distributions. Namely,
\[
\prob[y=\tilde{y},\sensors]=\sum_{\distset}\pi_{\dist}\prob_{\dist}[y=\tilde{y},\sensors].
\]
Recall from our modeling assumptions that models are conditionally independent given $y$ with $\prob_{\dist}[\sensors|y=\tilde{y}]=\prob_{\dist}[s_*|y=\tilde{y}]\prod_{i\in\ci}\prob_{\dist}[s_i|y=\tilde{y}]\prod_{j\in\cj}\prob_{\dist}[s_j|y=\tilde{y}]$. Therefore, without loss of generality, this holds not for $\prob[y=\tilde{y},\sensors]$. That is, each parameter is to encode this dependency structure and must be a function of some set of models. Below we propose a strategy to choose optimal weights to satisfy (\ref{eqn: equating y and o}).

We start by decomposing $\log\frac{\prob[y=\tilde{y}, \sensors]}{\prob[y=1-\tilde{y}, \sensors]}$.
\begin{align*}
    \log\frac{\prob[y=\tilde{y}, \sensors]}{\prob[y=1-\tilde{y}, \sensors]}=\log\frac{\prob[y=\tilde{y}, s_*]}{\prob[y=1-\tilde{y}, s_*]}+\sum_{i\in\mathcal{I}}\log\frac{\prob[s_i|y=\tilde{y},s_{I_i}]}{\prob[s_i|y=1-\tilde{y},s_{I_i}]}+\sum_{j\in\mathcal{J}}\log\frac{\prob[s_j|y=\tilde{y},s_{I},s_{J_j}]}{\prob[s_j|y=1-\tilde{y},s_{I},s_{J_j}]}
\end{align*}
where $I_i$ is the set of $i'$ such that $i'\in\mathcal{I}$ and $i'<i$. Similarly, we let $J_j$ be the set of $j'$ such that $j'\in\mathcal{J}$ and $j'<j$. Note that there are multiple such constructions to satisfy~(\ref{eqn: equating y and o}) to have optimal set of weights.

We split our proof into three main steps as follows.

\paragraph{Step 1: Derivation of bounds for optimal set of parameters}
Given our strategy, we then derive the parameters in terms of conditional probabilities of individual models. Towards that, let $b$ be decomposed into its additive components such that $b=b_*+\sum_{i\in \ci} b_i-\sum_{j\in \cj} b_j$. Let also $r_y=\log \frac{\prob[y=1]}{\prob[y=0]}$. We derive bounds for each sensor using (\ref{eqn: equating y and o}) as follows.

\begin{itemize}[leftmargin=*, wide = 0pt]
\item \textit{Main task model}: The parameters for the main model simply satisfies
\[
(2\tilde{y}-1)\big(w_*(2s_*-1)+b_*\big)=\log\frac{\prob[y=\tilde{y}, s_*]}{\prob[y=1-\tilde{y}, s_*]}=\log\frac{\sum_{\distset}\pi_{\dist}\prob_{\dist}[y=\tilde{y}, s_*]}{\sum_{\distset}\pi_{\dist}\prob_{\dist}[y=1-\tilde{y}, s_*]}.
\]
With a simple algebraic manipulation where $y=1$ and $s_*=1$ (resp. for $y=0$, $s_*=1$), we have that
\begin{equation}\label{eqn: w_*+b_*}
    w_*+b_*=\log\frac{\prob[y=1, s_*=1]}{\prob[y=0, s_*=1]}
\end{equation}
and for $y=0$ and $s_*=0$ (resp. for $y=1$, $s_*=0$)
\begin{equation}\label{eqn: -w_*+b_*}
    w_*-b_*=\log\frac{\prob[y=0, s_*=0]}{\prob[y=1, s_*=0]}.
\end{equation}

Combining (\ref{eqn: w_*+b_*}) and (\ref{eqn: -w_*+b_*}) we have
\begin{equation}\label{eqn: w_*}
\begin{split}
    w_* &= \frac{1}{2} \log \frac{\prob[y=1, s_*=1]}{\prob[y=0, s_*=1]}\frac{\prob[y=0, s_*=0]}{\prob[y=1, s_*=0]}\\
    &\stackrel{(*)}{=}\frac{1}{2}\log\frac{\big(\sum_{\distset}\pi_{\dist}\prob_{\dist}[y=1]\alpha_{*,\dist}\big)\big(\sum_{\distset}\pi_{\dist}\prob_{\dist}[y=0]\alpha_{*,\dist}\big)}{\big(\sum_{\distset}\pi_{\dist}\prob_{\dist}[y=1](1-\alpha_{*,\dist})\big)\big(\sum_{\distset}\pi_{\dist}\prob_{\dist}[y=0](1-\alpha_{*,\dist})\big)}
    \end{split}
\end{equation}
where (*) follows from that $\prob_{\dist}[y=s_*|y]=\alpha_{*, \dist}$ and $\prob_{\dist}[y\neq s_*|y]=1-\alpha_{*, \dist}$ for $\distset$.

Similarly, for $b_*$ we have
\begin{equation}\label{eqn: b_*}
\begin{split}
    b_* &= \frac{1}{2} \log \frac{\prob[y=1, s_*=1]}{\prob[y=0, s_*=1]}\frac{\prob[y=1, s_*=0]}{\prob[y=0, s_*=0]}\\
    &{=}\frac{1}{2}\log\frac{\big(\sum_{\distset}\pi_{\dist}\prob_{\dist}[y=1]\alpha_{*,\dist}\big)\big(\sum_{\distset}\pi_{\dist}\prob_{\dist}[y=1](1-\alpha_{*,\dist})\big)}{\big(\sum_{\distset}\pi_{\dist}\prob_{\dist}[y=0](1-\alpha_{*,\dist})\big)\big(\sum_{\distset}\pi_{\dist}\prob_{\dist}[y=0]\alpha_{*,\dist}\big)}.
    \end{split}
\end{equation}

Finally, noting that, for all $\tilde y \in \mathcal{Y}$, we have
\begin{equation*}
    {}_{\scriptscriptstyle \wedge}\alpha_* \hspace{-0.7em} \sum_{\distset}\hspace{-0.5em}\pi_{\dist}\prob_{\dist}[y=\tilde y]={}_{\wedge}\alpha_*\prob[y=\tilde y]\leq \sum_{\distset}\hspace{-0.5em}\pi_{\dist}\prob_{\dist}[y=\tilde y]\alpha_{*,\dist} \leq {}_{\scriptscriptstyle \vee}\alpha_* \hspace{-0.7em}\sum_{\distset}\hspace{-0.5em}\pi_{\dist}\prob_{\dist}[y=\tilde y]={}_{\scriptscriptstyle \vee}\alpha_*\prob[y=\tilde y].
\end{equation*}

Using the above relation as well as (\ref{eqn: w_*}) and (\ref{eqn: b_*}), the weight and bias of the main task model, $w_*$ and $b_*$, can therefore be bounded as

\begin{equation}\label{eqn: w_* bound}
   \log \frac{{}_{\scriptscriptstyle \wedge}\alpha_*}{1-{}_{\scriptscriptstyle \wedge}\alpha_*} \leq w_* \leq \log \frac{{}_{\scriptscriptstyle \vee}\alpha_*}{1-{}_{\scriptscriptstyle \vee}\alpha_*}
\end{equation}
and

\begin{equation}\label{eqn: b_* bound}
    r_y + \log \frac{{}_{\scriptscriptstyle \wedge}\alpha_*(1-{}_{\scriptscriptstyle \vee}\alpha_*)}{(1-{}_{\scriptscriptstyle \wedge}\alpha_*){}_{\scriptscriptstyle \vee}\alpha_*} \leq b_* \leq r_y + \log \frac{{}_{\scriptscriptstyle \vee}\alpha_*(1-{}_{\scriptscriptstyle \wedge}\alpha_*)}{(1-{}_{\scriptscriptstyle \vee}\alpha_*){}_{\scriptscriptstyle \wedge}\alpha_*}.
\end{equation}

To distinguish the effect of class imbalance in our analysis, we will define $b_{**}:=b_*-r_y$.

\item \textit{Permissive models}: For permissive model, we have
\[
\log\frac{\prob[s_i|y=\tilde{y},s_{I_i}]}{\prob[s_i|y=1-\tilde{y},s_{I_i}]} = (2\tilde{y}-1)(w_is_i+b_i).
\]
Therefore
\begin{align*}
    \log\frac{\prob[s_i|y=\tilde{y},s_{I_i}]}{\prob[s_i|y=1-\tilde{y},s_{I_i}]}=\log\frac{\frac{\prob[s_i,y=\tilde{y},s_{I_i}]}{\prob[y=\tilde{y},s_{I_i}]}}{\frac{\prob[s_i,y=1-\tilde{y},s_{I_i}]}{\prob[y=1-\tilde{y},s_{I_i}]}}\stackrel{(*)}{=}\log \frac{
    \frac{\sum_{\distset} \pi_{\dist}\prob_{\dist}[y=\tilde{y},s_{I_i}]\prob_{\dist}[s_{i}|y=\tilde{y}]}{\sum_{\distset} \pi_{\dist}\prob_{\dist}[y=\tilde{y},s_{I_i}]}}{\frac{\sum_{\distset} \pi_{\dist}\prob_{\dist}[y=1-\tilde{y},s_{I_i}]\prob_{\dist}[s_{i}|y=1-\tilde{y}]}{\sum_{\distset} \pi_{\dist}\prob_{\dist}[y=1-\tilde{y},s_{I_i}]}}
\end{align*}
where (*) follows from the conditional independence assumption. \\

Let $\tilde{y}=1$. Therefore, for $s_i=1$ we have 

\begin{align*}
     \min_{\dist}\alpha_{i, \dist}={}_{\scriptscriptstyle \wedge}\alpha_{i} \leq \frac{\sum_{\distset} \pi_{\dist}\prob_{\dist}[y=\tilde{y},s_{I_i}]\prob_{\dist}[s_{i}|y=\tilde{y}]}{\sum_{\distset} \pi_{\dist}\prob_{\dist}[y=\tilde{y},s_{I_i}]}\leq \max_{\dist}\alpha_{i, \dist}={}_{\scriptscriptstyle \vee}\alpha_{i}
\end{align*}
and 
\begin{align*}
     \min_{\dist}\epsilon_{i, \dist}={}_{\scriptscriptstyle \wedge}\epsilon_{i} \leq \frac{\sum_{\distset} \pi_{\dist}\prob_{\dist}[y=1-\tilde{y},s_{I_i}]\prob_{\dist}[s_{i}|y=1-\tilde{y}]}{\sum_{\distset} \pi_{\dist}\prob_{\dist}[y=1-\tilde{y},s_{I_i}]}\leq \max_{\dist}\epsilon_{i, \dist}={}_{\scriptscriptstyle \vee}\epsilon_{i}.
\end{align*}

Above bounds finally lead to
\begin{equation}\label{eqn: w_i+b+i}
\log\frac{{}_{\scriptscriptstyle \wedge}\alpha_{i}}{{}_{\scriptscriptstyle \vee}\epsilon_{i}} \leq \log\frac{\prob[s_i|y=\tilde{y},s_{I_i}]}{\prob[s_i|y=1-\tilde{y},s_{I_i}]} = w_i+b_i \leq \log\frac{{}_{\scriptscriptstyle \vee}\alpha_{i}}{{}_{\scriptscriptstyle \wedge}\epsilon_{i}}.
\end{equation}

Next, we let $s_i=0$. Repeating the same technique above, we have
\begin{align*}
     \min_{\dist}1-\alpha_{i, \dist}=1-{}_{\scriptscriptstyle \vee}\alpha_{i} \leq \frac{\sum_{\distset} \pi_{\dist}\prob_{\dist}[y=\tilde{y},s_{I_i}]\prob_{\dist}[s_{i}|y=\tilde{y}]}{\sum_{\distset} \pi_{\dist}\prob_{\dist}[y=\tilde{y},s_{I_i}]}\leq \max_{\dist}1-\alpha_{i, \dist}=1-{}_{\scriptscriptstyle \wedge}\alpha_{i}
\end{align*}
and 
\begin{align*}
     \min_{\dist}1-\epsilon_{i, \dist}=1-{}_{\scriptscriptstyle \vee}\epsilon_{i} \leq \frac{\sum_{\distset} \pi_{\dist}\prob_{\dist}[y=1-\tilde{y},s_{I_i}]\prob_{\dist}[s_{i}|y=1-\tilde{y}]}{\sum_{\distset} \pi_{\dist}\prob_{\dist}[y=1-\tilde{y},s_{I_i}]}\leq \max_{\dist}1-\epsilon_{i, \dist}=1-{}_{\scriptscriptstyle \wedge}\epsilon_{i}.
\end{align*}

Above bounds finally lead to
\begin{equation}\label{eqn: b+i}
\log\frac{1-{}_{\scriptscriptstyle \vee}\alpha_{i}}{1-{}_{\scriptscriptstyle \wedge}\epsilon_{i}} \leq \log\frac{\prob[s_i|y=\tilde{y},s_{I_i}]}{\prob[s_i|y=1-\tilde{y},s_{I_i}]} = b_i \leq \log\frac{1-{}_{\scriptscriptstyle \wedge}\alpha_{i}}{1-{}_{\scriptscriptstyle \vee}\epsilon_{i}}.
\end{equation}
Note that the same conclusion can be drawn for $\tilde{y}=0$.\\

\item \textit{Preventative models}: For preventative model, we have
\[
\log\frac{\prob[s_j|y=\tilde{y},s_I,s_{J_j}]}{\prob[s_j|y=1-\tilde{y},s_I,s_{J_j}]} = -(2\tilde{y}-1)(w_j(1-s_j)+b_j).
\]
Then
\begin{align*}
    \log\frac{\prob[s_j|y=\tilde{y},s_I,s_{J_j}]}{\prob[s_j|y=1-\tilde{y},s_I,s_{J_j}]}=\log\frac{\frac{\prob[s_j,y=\tilde{y},s_I,s_{J_j}]}{\prob[y=\tilde{y},s_I,s_{J_j}]}}{\frac{\prob[s_j,y=1-\tilde{y},s_I,s_{J_j}]}{\prob[y=1-\tilde{y},s_I,s_{J_j}]}}\stackrel{(*)}{=}\log \frac{
    \frac{\sum_{\distset} \pi_{\dist}\prob_{\dist}[y=\tilde{y},s_I,s_{J_j}]\prob_{\dist}[s_{j}|y=\tilde{y}]}{\sum_{\distset} \pi_{\dist}\prob_{\dist}[y=\tilde{y},s_I,s_{J_j}]}}{\frac{\sum_{\distset} \pi_{\dist}\prob_{\dist}[y=1-\tilde{y},s_I,s_{J_j}]\prob_{\dist}[s_{j}|y=1-\tilde{y}]}{\sum_{\distset} \pi_{\dist}\prob_{\dist}[y=1-\tilde{y},s_I,s_{J_j}]}}
\end{align*}
where (*) follows from the conditional independence assumption. 

Let $\tilde{y}=0$. Therefore, for $s_j=0$ we have 

\begin{align*}
     \min_{\dist}\alpha_{j, \dist}={}_{\scriptscriptstyle \wedge}\alpha_{j} \leq \frac{\sum_{\distset} \pi_{\dist}\prob_{\dist}[y=\tilde{y},s_I,s_{J_j}]\prob_{\dist}[s_{j}|y=\tilde{y}]}{\sum_{\distset} \pi_{\dist}\prob_{\dist}[y=\tilde{y},s_I,s_{J_j}]}\leq \max_{\dist}\alpha_{j, \dist}={}_{\scriptscriptstyle \vee}\alpha_{j}
\end{align*}
and 
\begin{align*}
     \min_{\dist}\epsilon_{j, \dist}={}_{\scriptscriptstyle \wedge}\epsilon_{j} \leq \frac{\sum_{\distset} \pi_{\dist}\prob_{\dist}[y=1-\tilde{y},s_I,s_{J_j}]\prob_{\dist}[s_{j}|y=1-\tilde{y}]}{\sum_{\distset} \pi_{\dist}\prob_{\dist}[y=1-\tilde{y},s_I,s_{J_j}]}\leq \max_{\dist}\epsilon_{j, \dist}={}_{\scriptscriptstyle \vee}\epsilon_{j}.
\end{align*}

Above bounds finally lead to
\begin{equation}\label{eqn: w_j+b+j}
\log\frac{{}_{\scriptscriptstyle \wedge}\alpha_{j}}{{}_{\scriptscriptstyle \vee}\epsilon_{j}} \leq \log\frac{\prob[s_j|y=\tilde{y},s_I,s_{J_j}]}{\prob[s_j|y=1-\tilde{y},s_I,s_{J_j}]} = w_j+b_j \leq \log\frac{{}_{\scriptscriptstyle \vee}\alpha_{j}}{{}_{\scriptscriptstyle \wedge}\epsilon_{j}}.
\end{equation}

Next, we let $s_j=1$. Repeating the same technique above, we have
\begin{align*}
     \min_{\dist}1-\alpha_{j, \dist}=1-{}_{\scriptscriptstyle \vee}\alpha_{j} \leq \frac{\sum_{\distset} \pi_{\dist}\prob_{\dist}[y=\tilde{y},s_I,s_{J_j}]\prob_{\dist}[s_{j}|y=\tilde{y}]}{\sum_{\distset} \pi_{\dist}\prob_{\dist}[y=\tilde{y},s_I,s_{J_j}]}\leq \max_{\dist}1-\alpha_{j, \dist}=1-{}_{\scriptscriptstyle \wedge}\alpha_{j}
\end{align*}
and 
\begin{align*}
     \min_{\dist}1-\epsilon_{j, \dist}=1-{}_{\scriptscriptstyle \vee}\epsilon_{j} \leq \frac{\sum_{\distset} \pi_{\dist}\prob_{\dist}[y=1-\tilde{y},s_I,s_{J_j}]\prob_{\dist}[s_{j}|y=1-\tilde{y}]}{\sum_{\distset} \pi_{\dist}\prob_{\dist}[y=1-\tilde{y},s_I,s_{J_j}]}\leq \max_{\dist}1-\epsilon_{j, \dist}=1-{}_{\scriptscriptstyle \wedge}\epsilon_{j}.
\end{align*}

Similarly as in permissive models, above bounds lead to
\begin{equation}\label{eqn: b_j}
\log\frac{1-{}_{\scriptscriptstyle \vee}\alpha_{j}}{1-{}_{\scriptscriptstyle \wedge}\epsilon_{j}} \leq \log\frac{\prob[s_j|y=\tilde{y},s_I,s_{J_j}]}{\prob[s_j|y=1-\tilde{y},s_I,s_{J_j}]} = b_j \leq \log\frac{1-{}_{\scriptscriptstyle \wedge}\alpha_{j}}{1-{}_{\scriptscriptstyle \vee}\epsilon_{j}}.
\end{equation}
The same conclusion can be drawn for $\tilde{y}=1$.
\end{itemize}

\paragraph{Step 2: Decomposition of $\deltaw$}
Next, we recall Remark~\ref{remark: closed form delta} and present a lower bound for $\deltaw$ that decomposes $\deltaw$ into its additive components such that
\begin{align*}
    \deltaw &= (2\tilde{y}-1)\Big(b+w_*(2s_*-1)+\sum_{i\in \ci}w_is_i-\sum_{j\in \cj} w_j(1-s_j)\Big)\\
    &=(2\tilde{y}-1)\Big(w_*(2s_*-1)+\sum_{i\in \ci}\big(w_is_i+b_i\big)-\sum_{j\in \cj} \big(w_j(1-s_j)+b_j\big)\Big).
\end{align*}

Next, we analyze
\[
\prob_{\dist}\big[\langle{\bf w}, \mathbf{f}_{y}(s_*,s_{\mathcal{I}}, s_{\mathcal{J}})- \mathbf{f}_{1-y}(s_*,s_{\mathcal{I}}, s_{\mathcal{J}})\rangle|y\big]=\prob_{\dist}[\deltaw|y].
\]
Note that $\prob_{\dist}[s_*|y]=\alpha_{*,\dist}$ if $s_*=y$. Therefore, $\prob_{\dist}[s_*=1|y=1]=\alpha_{*,\dist}$ and $\prob_{\dist}[s_*=0|y=1]=1-\alpha_{*,\dist}$. Similarly, $\prob_{\dist}[s_*=0|y=0]=\alpha_{*,\dist}$ and $\prob_{\dist}[s_*=1|y=0]=1-\alpha_{*,\dist}$.
Thus
\begin{equation*}
    \prob_{\dist}[(2\tilde{y}-1)\big(w_*(2s_*-1)+b_*\big)|y]\stackrel{(*)}{=}\prob_{\dist}[w_*(2s_{**}-1)+b_{**}+(2\tilde{y}-1)r_y|y]
\end{equation*}
where $s_{**}$ satisfies $\prob_{\dist}[s_{**}=1]=\alpha_{*, \dist}$ and $\prob_{\dist}[s_{**}=0]=1-\alpha_{*, \dist}$.
Note that (*) stems from the symmetry of $s_*$ and $b_{**}$ with respect to $y$. To reduce exposition, we will stick to $s_*$ notation and continue to refer to $s_{**}$ as $s_*$. Hence, we define $d_{*,\dist}$ as
\begin{equation}\label{eqn: d_*,D}
    d_{*,\dist}:=w_*(2s_*-1)+b_{**}
\end{equation}
where $b_{**}:=b_*-r_y$ as defined earlier. Therefore, the contribution of the main task model in the majority voting random variable $\deltaw$ will be
\begin{equation}\label{eqn: d_**}
    d_{*,\dist} + (2y-1)r_y.
\end{equation}

Next, we analyze the auxiliary model predictions. For $y=1$, 

\[
\prob_{\dist}\big[ (2y-1)\big(\sum_{i\in\ci}(w_is_i+b_i)-\sum_{j\in\cj}(w_j(1-s_j)+b_j)\big)|y\big]=\prob_{\dist}\big[ \sum_{i\in\ci}(w_is_i+b_i)-\sum_{j\in\cj}(w_j(1-s_j)+b_j)|y=1\big]
\]
where, on the right hand side, we have $\prob_{\dist}[s_{i}=1|y=\tilde y]=\alpha_{i, \dist}$ and $\prob_{\dist}[1-s_{j}=1|y=\tilde y]=\epsilon_{j, \dist}$ for $\tilde y=1$ over distribution $\distset$. Therefore, we define $d_{\ci,\dist}$ as
\begin{equation}\label{eqn: d_I,D}
\big(\deltaw-d_{*, \dist}-r_y|y=1\big) = \sum_{i\in\ci}(w_is_i+b_i)-\sum_{j\in\cj}(w_j(1-s_j)+b_j) :=d_{\ci, \dist}.
\end{equation}

Using the same strategy for $y=0$, we define $d_{\cj, \dist}$ as

\begin{equation}\label{eqn: d_J,D}
\big(\deltaw-d_{*, \dist}+r_y|y=0\big) = \sum_{j\in\cj}(w_j(1-s_j)+b_j)-\sum_{i\in\ci}(w_is_i+b_i) :=d_{\cj, \dist} 
\end{equation}
where, on the right hand side, we have $\prob_{\dist}[1-s_{j}=1|y=\tilde y]=\alpha_{j, \dist}$ and $\prob_{\dist}[s_{i}=1|y=\tilde y]=\epsilon_{i, \dist}$ for $\tilde y=0$ over $\distset$.

Combining (\ref{eqn: d_**}), (\ref{eqn: d_I,D}) and (\ref{eqn: d_J,D}), we have
\begin{equation}\label{eqn: decomposition of Delta_w}
    \big(\deltaw|y\big)=d_{* \dist}+yd_{\ci, \dist}+(1-y)d_{\cj, \dist} +(2y-1)r_y.
\end{equation}

\paragraph{Final step: $\E_{\sensors}[\deltaw|y]$} We express $\deltaw|y$ in terms of $y$ and a function of model predictions thus far. In this step, using the bounds on the optimal parameters in the first step as well as the decomposition introduced in the second step, we derive a lower bound for the $\E_{\sensors}[\deltaw|y]$. Towards that, we lower bound the expected value of $d_{*, \dist}$, $d_{\ci, \dist}$ and $d_{\cj, \dist}$ individually.

\begin{itemize}[leftmargin=*, wide = 0pt]
\item $\E_{s_*}[d_{*, \dist}]$: For the main task model, we have
\begin{equation}\label{eqn: mu_*, D}
  \E_{s_*}[d_{*, \dist}] = \E_{s_*}[w_*(2s_*-1)+b_{**}]
\end{equation}
over distribution $\distset$ and $w_*$. One can infer from (\ref{eqn: w_* bound}) and (\ref{eqn: b_* bound}) for $b_{**}=b_*-r_y$ that
\begin{equation}\label{eqn: mu_*}
    \begin{split}
        \E_{s_*}[d_{*, \dist}] = \E_{s_*}[w_*(2s_*-1)+(2y-1)b_{**}] \geq \alpha_{*, \dist}\log\frac{{}_{\scriptscriptstyle\wedge}\alpha_{*}}{1-{}_{\scriptscriptstyle\wedge}\alpha_{*}}+(1-\alpha_{*, \dist})\log\frac{1-{}_{\scriptscriptstyle\vee}\alpha_{*}}{{}_{\scriptscriptstyle\vee}\alpha_{*}}:=\mu_{d_{*,\dist}}.
    \end{split}
\end{equation}

\item $\E_{s_{\ci}, s_{\cj}}[d_{\ci, \dist}]$: For the permissive models, we have
\begin{equation*}
  \E_{s_{\ci}, s_{\cj}}[d_{\ci, \dist}] = \E_{s_{\ci}, s_{\cj}}\Big[\sum_{i\in\ci}(w_is_i+b_i)-\sum_{j\in\cj}(w_j(1-s_j)+b_j)\Big]\stackrel{}{=} \E_{s_{\ci}}\Big[\sum_{i\in\ci}(w_is_i+b_i)\Big]-\E_{s_{\cj}}\Big[\sum_{j\in\cj}(w_j(1-s_j)+b_j)\Big].
\end{equation*}

Note that $w_is_i+b_i=w_i+b_i$ with probability $\alpha_{i, \dist}$ and $w_is_i+b_i=b_i$ otherwise. Therefore, using (\ref{eqn: w_i+b+i}) and (\ref{eqn: b+i}) we lower bound $\E_{s_{\ci}}\Big[\sum_{i\in\ci}(w_is_i+b_i)\Big]$ as
\begin{equation*}
    \E_{s_{\ci}}\Big[\sum_{i\in\ci}(w_is_i+b_i)\Big]\geq \sum_{i\in\ci} \alpha_{i, \dist}\log\frac{{}_{\scriptscriptstyle \wedge}\alpha_{i}}{{}_{\scriptscriptstyle \vee}\epsilon_{i}}+(1-\alpha_{i,\dist})\log\frac{1-{}_{\scriptscriptstyle \vee}\alpha_{i}}{1-{}_{\scriptscriptstyle \wedge}\epsilon_{i}}.
\end{equation*}
Similarly, $-\E_{s_{\cj}}\Big[\sum_{j\in\cj}(w_j(1-s_j)+b_j)\Big]$ can be lower bounded as
\begin{equation*}
    -\E_{s_{\cj}}\Big[\sum_{j\in\cj}(w_j(1-s_j)+b_j)\Big]\geq -\sum_{j\in\cj} \epsilon_{j, \dist}\log\frac{{}_{\scriptscriptstyle \vee}\alpha_{j}}{{}_{\scriptscriptstyle \wedge}\epsilon_{j}}+(1-\epsilon_{j,\dist})\log\frac{1-{}_{\scriptscriptstyle \wedge}\alpha_{j}}{1-{}_{\scriptscriptstyle \vee}\epsilon_{j}}.
\end{equation*}

Combining above result, we have
\begin{equation}\label{eqn: mu_I, D}
    \E_{s_{\ci}, s_{\cj}}[d_{\ci, \dist}]\geq \sum_{i\in \mathcal{I}}\alpha_{i, \dist}\log\frac{{}_{\scriptscriptstyle\wedge}\alpha_{i}}{{}_{\scriptscriptstyle\vee}\epsilon_{i}}+(1-\alpha_{i, \dist})\log\frac{1-{}_{\scriptscriptstyle\vee}\alpha_{i}}{1-{}_{\scriptscriptstyle\wedge}\epsilon_{i}}-\sum_{j\in \mathcal{J}}\epsilon_{j, \dist}\log\frac{{}_{\scriptscriptstyle\vee}\alpha_{j}}{{}_{\scriptscriptstyle\wedge}\epsilon_{j}}-(1-\epsilon_{j, \dist})\log\frac{1-{}_{\scriptscriptstyle\wedge}\alpha_{j}}{1-{}_{\scriptscriptstyle\vee}\epsilon_{j}}:=\mu_{\ci, \dist}.
\end{equation}
\\

\item $\E_{s_{\ci}, s_{\cj}}[d_{\cj, \dist}]$: Following to the same strategy to that of $\E_{s_{\ci}, s_{\cj}}[d_{\ci, \dist}]$, we have
\begin{equation}\label{eqn: mu_J, D}
   \E_{s_{\ci}, s_{\cj}}[d_{\cj, \dist}]\geq \sum_{j\in \mathcal{J}}\alpha_{j, \dist}\log\frac{{}_{\scriptscriptstyle\wedge}\alpha_{j}}{{}_{\scriptscriptstyle\vee}\epsilon_{j}}+(1-\alpha_{j, \dist})\log\frac{1-{}_{\scriptscriptstyle\vee}\alpha_{j}}{1-{}_{\scriptscriptstyle\wedge}\epsilon_{j}}-\sum_{i\in \mathcal{I}}\epsilon_{i, \dist}\log\frac{{}_{\scriptscriptstyle\vee}\alpha_{i}}{{}_{\scriptscriptstyle\wedge}\epsilon_{i}}-(1-\epsilon_{i, \dist})\log\frac{1-{}_{\scriptscriptstyle\wedge}\alpha_{i}}{1-{}_{\scriptscriptstyle\vee}\epsilon_{i}}:=\mu_{\cj, \dist}.
\end{equation}
\end{itemize}

Finally, combining (\ref{eqn: decomposition of Delta_w}, \ref{eqn: mu_*, D}, \ref{eqn: mu_I, D}, \ref{eqn: mu_J, D}) we conclude
\begin{equation}
    \begin{split}
       \mathbb{E}_{s_*,s_{\mathcal{I}}, s_{\mathcal{J}}}[\deltaw|y]&=\mathbb{E}_{s_*,s_{\mathcal{I}}, s_{\mathcal{J}}}[d_{*, \dist}+yd_{\ci,\dist}+(1-y)d_{\cj,\dist}+(2y-1)r_y]\\
       &\geq \mu_{*, \dist}+y\mu_{\ci, \dist}+(1-y)\mu_{\cj, \dist}+(2y-1)r_y:=\mu_{y, \dist}. 
    \end{split}
\end{equation}
The proof is thus completed.
\end{proof}

\subsection{Proof of Theorem~\ref{thm: main theorem}}\label{subsec: proof of main theorem}
We start by recalling our main theorem.
\begin{theorem*}[\recallthm]
For $y\in \mathcal{Y}$ and $\distset$, let $\mu_{y,\dist}$ be 
defined as in Lemma~\ref{lemma: expected tail}. Suppose that the \textit{modeling assumption} holds, and suppose that $\mu_{d_{\mathcal{K}, \dist}}>0$, for all $\mathcal{K}\in\{\mathcal{I}, \mathcal{J}\}$ and $\distset$. Then 

\begin{equation}
    \begin{split}
       \pipelineaccuracy \geq 1-\mathbb{E}_{\mu_{y, \dist}}[\exp\big({-{2}{\mu_{y, \dist}^2}/{v^2}}\big)] ,
    \end{split}
\end{equation}
where $v^2$ is the variance upper bound to $\prob[o=y|y]$ with
\begin{equation*}
\begin{split}
v^2\hspace{-0.2em}= 4\Big(\log\frac{{}_{\scriptscriptstyle \vee}\alpha_{*}}{1-{}_{\scriptscriptstyle \wedge}\alpha_{*}}\Big)^2\hspace{-0.5em}+\hspace{-0.5em}\sum_{k\in \mathcal{I}\cup\mathcal{J}} \Big(\log\frac{{}_{\scriptscriptstyle \vee}\alpha_{k}(1-{}_{\scriptscriptstyle \wedge}\epsilon_{k})}{{}_{\scriptscriptstyle \wedge}\epsilon_{k}(1-{}_{\scriptscriptstyle \vee}\alpha_{k})} \Big)^2\hspace{-0.5em}.
\end{split}
\end{equation*}
\end{theorem*}

\begin{proof}[Proof of Theorem~\ref{thm: main theorem}]
Recall that we define \acc of \sys as
\[
\pipelineaccuracy=\mathbb{E}_{\dist\sim \{\dist_a,\dist_b\}}\mathbb{E}_{y\sim\mathcal{Y}}\big[\prob_{\dist}[o=y|y, \mathbf{w}]\big].
\]
The weighted accuracy definition comes from the latent variable $z$. That is, $\pipelineaccuracy=\prob[o=y|\mathbf{w}]=\sum_{z\in \{0,1\}} \prob[o=y|z, \mathbf{w}]$ where $\prob[o=y|z=0, \mathbf{w}]=\prob_{\benign}[o=y|\mathbf{w}]$ and $\prob[o=y|z=1, \mathbf{w}]=\prob_{\adv}[o=y|\mathbf{w}]$. Hence, $\pipelineaccuracy=\E_{\dist\sim \{\benign, \adv\}}\big[\prob_{\dist}[o=y|\mathbf{w}]\big]=\E_{\dist\sim \{\benign, \adv\}}\E_{y\sim\mathcal{Y}}\big[\prob_{\dist}[o=y|y, \mathbf{w}]\big]$.

Let {\bf w} be the set of optimal parameters. Using~(\ref{eqn: sigmoid representation}) and our inference rule, $\prob_{\dist}[o=y|y, \mathbf{w}]$ can be further expressed as
\begin{align*}
    &\prob_{\dist}[o=y|y,\mathbf{w}]\\ &=\prob_{\dist}\big[\sigma\big(\deltaw\big)>1/2|y\big]=\prob_{\dist}[\deltaw>0|y]=1-\prob_{\dist}[\deltaw<0|y]
\end{align*}

For the rest of the proof, we will focus on bounding the term $\prob_{\dist}[\deltaw<0|y]$, and $\pipelineaccuracy$ will follow from taking expectation of $1-\prob_{\dist}[\deltaw<0|y]$ over $\distset$ and $y\in\mathcal{Y}$.

Next, we recall the generalized bounded difference inequality as well as generalized Hoeffding's inequality~\cite{van2002hoeffding}. Note that the same result can be shown via Azuma's inequality for submartingale sequences~\cite{azuma1967weighted}.
\begin{theorem}[~\cite{azuma1967weighted},~\cite{van2002hoeffding}]\label{thm: azuma}
Assume that $X_t$ be a random variable with respect to filtration $\mathcal{F}_t$, and $\mathcal{L}_t$ and $\mathcal{U}_t$ be $\mathcal{F}_{t-1}$ measurable random variables such that $$\mathcal{L}_t\leq X_t-X_{t-1}\leq\mathcal{U}_t$$ where $\mathcal{L}_t<\mathcal{U}_t$ and $\mathcal{U}_t-\mathcal{L}_t\leq c_t$ almost surely. Therefore, for some $\epsilon>0$, one has
\begin{equation}\label{eqn: bounded differnces inequality}
    \prob(X_n-\E[X_n]<-\epsilon)\leq\exp\Big(-\frac{2\epsilon^2}{\sum_{t=[n]}c_t^2}\Big) \hspace{1em}\text{and symmetrically}\hspace{1em} \prob(X_n-\E[X_n]>\epsilon)\leq\exp\Big(-\frac{2\epsilon^2}{\sum_{t=[n]}c_t^2}\Big).
\end{equation}
\end{theorem}
We now consider the random variable $\deltaw=d_{*, \dist}+yd_{\ci, \dist}+(1-y)d_{\cj,\dist}+(2y-1)r_y$ that is meant to represent~$X_n$ in Theorem~\ref{thm: azuma}, where each increment is induced by a single model. We call $\deltaw$ as $X_{1+|\ci|+|\cj|}$.

To prove compatibility of our setting with the Theorem~\ref{thm: azuma}, we present the following remark.
\begin{remark}[Measurability of $X_{1+|\ci|+|\cj|}$ and the bounded differences]
Let $y=1$. We can write our random variable $X_{1+|\ci|+|\cj|}=\deltaw$ as
\[
\big(\deltaw|y=1\big)=w_*(2s_*-1)+b_*+\sum_{i\in\ci}(w_is_i+b_i)-\sum_{j\in\cj}(w_j(1-s_j)+b_j).
\]
That is, we represent $\big(\deltaw|y=1\big)$ as a random process with a total of $1+|\ci|+|\cj|$ increments. Let $X_0=0$, we treat the main sensor as the first increment such that
\[
X_1=w_*(2s_*-1)+b_*.
\]
For $t=1, ..., |\ci|$ we let
\[
X_{t+1}-X_t= w_is_i+b_i \hspace{.5em}\text{s.t. \ } i=t+1.
\]
Finally, for $t=|\ci|+1, ..., |\ci|+|\cj|$ we let
\[
X_{t+1}-X_t= -(w_j(1-s_j)+b_j) \hspace{.5em}\text{s.t. \ } j=t+1.
\]
and the similar analysis can be performed for $y=0$.

Above decomposition shows that $X_{1+|\ci|+|\cj|}$ is $\mathcal{F}_n$ measurable. Specifically, $X_{t+1}-X_t$ is $\mathcal{F}_t$ measurable for all $t=1, ..., 1+|\ci|+|\cj|$. Moreover, $X_{t+1}-X_t$ and $X_{t'+1}-X_{t'}$ are independent for $t\neq t'$. 

Using the increments introduced above, one can further show that the maximum increments $c_t$ for $t=1, ..., 1+|\ci|+|\cj|$ are given by

\[
|w_*+b_* - (-w_*+b_*)|=2w_*\leq 2{}_{\scriptscriptstyle \vee}w_*:=c_{1}.
\]
For $t=1, ..., |\ci|$ we let
\[
|X_{t+1}-X_t|= |(w_i+b_i) - b_i|\leq {}_{\scriptscriptstyle\vee}w_i:=c_{t+1} \hspace{.5em}\text{s.t. \ } i=t+1.
\]
Finally, for $t=|\ci|+1, ..., |\ci|+|\cj|$ we let
\[
|X_{t+1}-X_t|= |-(w_j+b_j) - (-b_j)|\leq {}_{\scriptscriptstyle\vee}w_j:=c_{t+1} \hspace{.5em}\text{s.t. \ } i=t+1.
\]

Recalling the bounds in~(\ref{eqn: w_* bound}, \ref{eqn: w_i+b+i}, \ref{eqn: b+i}, \ref{eqn: w_j+b+j}, \ref{eqn: b_j}), we have
\begin{equation}\label{eqn: bounded differences}
    \begin{split}
        c_1 = 2\log\frac{{}_{\scriptscriptstyle \vee}\alpha_*}{1-{}_{\scriptscriptstyle \wedge}\alpha_*}\hspace{1em}\text{for t=1 and \ }
        c_t = \log\frac{{}_{\scriptscriptstyle \vee}\alpha_{t}(1-{}_{\scriptscriptstyle \wedge}\epsilon_{t})}{{}_{\scriptscriptstyle \wedge}\epsilon_{t}(1-{}_{\scriptscriptstyle \vee}\alpha_{t})} \hspace{1em}\text{for\ } t\in\ci\cup\cj.
    \end{split}
\end{equation}
\end{remark}

Next, for any $y\in \mathcal{Y}$, we derive the following
\begin{equation*}
    \begin{split}
        &\prob_{\dist}[\deltaw<0|y]\\
        &=\prob_{\dist}\big[\deltaw-\E_{\sensors}[\deltaw]<-\E_{\sensors}[\deltaw]\big|y\big]\\
        &\stackrel{(*)}{\leq} \prob_{\dist}\big[\deltaw-\E_{\sensors}[\deltaw|y]<-\mu_{y, \dist}\big|y\big]
    \end{split}
\end{equation*}
where (*) stems from that $\mu_{y, \dist}$ is a lower bound to $\E_{\sensors}[\deltaw|y]$ as shown in Lemma~\ref{lemma: expected tail}.

Let $\epsilon=\mu_{y, \dist}$. If $\mu_{y, \dist}>0$, using Theorem~\ref{thm: azuma} for $\Psi_2=\frac{\sum_{t\in\{1\}\cup\ci\cup\cj}c^2_t}{\mu^2_{y, \dist}}$ where $c_t$ is as defined in~(\ref{eqn: bounded differences}) results in
\[
\prob_{\dist}[\deltaw<0|y]\leq \prob_{\dist}\big[\deltaw-\E_{\sensors}[\deltaw|y]<-\mu_{y, \dist}\big|y\big]\leq \exp(-2/\Psi_2).
\]
By further taking the expectation of $\prob_{\dist}[\deltaw<0|y]$ over $\distset$ and $y\in\mathcal{Y}$ such that
\begin{align*}
&\pipelineaccuracy=\mathbb{E}_{\dist\sim \{\dist_a,\dist_b\}}\mathbb{E}_{y\sim\mathcal{Y}}\big[\prob_{\dist}[o=y|y]\big]=\mathbb{E}_{\dist\sim \{\dist_a,\dist_b\}}\mathbb{E}_{y\sim\mathcal{Y}}\big[\prob_{\dist}[\deltaw>0|y]\big]\\
&=1-\mathbb{E}_{\dist\sim \{\dist_a,\dist_b\}}\mathbb{E}_{y\sim\mathcal{Y}}\big[\prob_{\dist}[\deltaw<0|y]\big]\geq 1-\mathbb{E}_{\mu_{y,\dist}}\big[\exp(-2\mu^2_{y,\dist}/v^2)\big]
\end{align*}
concludes the proof.
\end{proof}

\subsection{Proof of Theorem~\ref{proposition}}\label{subsec: proof of proposition}
We begin with recalling Theorem~\ref{proposition}. 
\begin{theorem*}[\recallthm]
Let the number of permissive and preventative models be the same and denoted by $n$ such that $n:=|\ci|=|\cj|$. Note that the weighted accuracy of the main model in terms of its truth rate is simply $\alpha_*:=\sum_{\distset}\pi_{\dist}\alpha_{*, \dist}$. Moreover, let $\ck, \ck'\in\{\ci, \cj\}$ with $\ck\neq\ck'$ and for any $\distset$, let 
\begin{equation*}
    \begin{split}
       \gamma_{\dist} :=\frac{1}{n+1} \min_{\ck}\Big\{ \alpha_{*, \dist}-1/2+\sum_{k\in\ck}\alpha_{k, \dist}-\sum_{k'\in\mathcal{K}'}\epsilon_{k', \dist}\Big\}. 
    \end{split}
\end{equation*}
If $\gamma_{\dist}>\sqrt{\frac{4}{n+1}\log\frac{1}{1-\alpha_{*}}}$ for all $\distset$, then $\pipelineaccuracy>\mainaccuracy$.
\end{theorem*}

\begin{proof}[Proof of Theorem~\ref{proposition}] 
We start by recalling the widely known Chernoff bound for the sum of independent and non-identical random variables.
\begin{lemma}[Chernoff Bound for Poisson Binomial Distributions]\label{lemma}
Let $X$ be a random variable with Poisson Binomial distribution. For $\delta \in [0, 1]$,
\begin{align*}
    \prob[X<(1-\delta)\mu_{X}]\leq \exp(-\delta^2\mu_{X}/2).
\end{align*}
\end{lemma}

Recall that \sys\ predicts $y$ to be $\hat{o}$ where
$$\hat{o}=\argmax_{\tilde y\in\mathcal{Y}} \mathbb{P}[o =\tilde y| \tilde  s_*, \tilde s_{\mathcal{I}},\tilde  s_{\mathcal{J}}, {\bf w}]=\argmax_{\tilde y\in\mathcal{Y}}\sigma(\Delta_{\bf w} (\tilde{y},s_*,s_{\mathcal{I}}, s_{\mathcal{J}}))$$
where
\begin{align*}
    \Delta_{\bf w} (\tilde{y},s_*,s_{\mathcal{I}}, s_{\mathcal{J}})=(2\tilde{y}-1)\Big(b+w_*(2s_*-1)+\sum_{i\in \ci}w_is_i-\sum_{j\in \cj} w_j(1-s_j)\Big).
\end{align*}

We showed earlier that there exist a set of parameters ${\bf w}$, and call it optimal parameters ${\bf w}^*$, where
$$\mathbb{P}[o =\tilde y| \tilde  s_*, \tilde s_{\mathcal{I}},\tilde  s_{\mathcal{J}}, {\bf w}^*]= \mathbb{P}[y =\tilde y| \tilde  s_*, \tilde s_{\mathcal{I}},\tilde  s_{\mathcal{J}}]$$

for all $\tilde y \in \mathcal{Y}$.

Note that, due to above equation, $\mathbb{P}[o =\tilde y| \tilde  s_*, \tilde s_{\mathcal{I}},\tilde  s_{\mathcal{J}}, {\bf w}^*]$ is Bayes classifier where the error of classifier is minimized over ${\bf w}$. Hence,

\begin{align*}
    \prob[\hat{o}\neq y|{\bf w}^*]\leq \prob[\hat{o}\neq y|{\bf w}]
\end{align*}
and 
\begin{align*}
    \prob[\hat{o}= y|{\bf w}^*]\geq \prob[\hat{o}= y|{\bf w}]
\end{align*}
for any ${\bf w}\in \mathbb{R}^{|\ci|+|\cj|+2}$.

Leveraging above fact, we will bound $\prob[\hat{o}= y|{\bf w}]$ from below where we will use some parameters ${\bf w}$ that are not optimal. That is, from now on, we will focus on $\prob[\hat{o}= y|{\bf w}]$ where ${\bf w}$ is not optimal but leads to a close resemblance of $\prob[\hat{o}= y|{\bf w}^*]$. In other words, we will perform a worst-case analysis where $\hat o$ will be a result of unweighted majority voting. Hence, we let ${\bf w}$ be given by
$\mathbf{w}=[0;1/2; (1)_{i\in\ci}; (1)_{j\in\cj}]$. For this case, $\Delta_{\bf w} (\tilde{y},s_*,s_{\mathcal{I}}, s_{\mathcal{J}})$ becomes a random variable with Poisson Binomial distribution and with some bias. That is, 
\begin{align*}
    &\Delta_{\bf w} (\tilde{y},s_*,s_{\mathcal{I}}, s_{\mathcal{J}})=(2\tilde{y}-1)\Big((s_*-1/2)+\sum_{i\in \ci}s_i-\sum_{j\in \cj} (1-s_j)\Big)
\end{align*}
where $s_*$, $s_{i\in\ci}$ and $s_{j\in\cj}$ are random variables in $\mathcal{Y}$.

Using the weight introduced above, we can now re-write the \acc\ of \sys\ as
\begin{equation}\label{eqn: pipeline accuracy proposition}
    \begin{split}
    &\pipelineaccuracy =\prob[\hat{o}=y|{\bf w}^*]\geq \prob[\hat{o}=y|{\bf w}] = \pi_{\adv}\prob_{\adv}[\hat{o}=y|{\bf w}]+\pi_{\benign}\prob_{\benign}[\hat{o}=y|{\bf w}]\\
    &=\pi_{\adv}\big(\prob_{\adv}[\hat{o}=y|{\bf w}, y=1]\prob_{\adv}[y=1]+\prob_{\adv}[\hat{o}=y|{\bf w}, y=0]\prob_{\adv}[y=0] \big)\\
    &+\pi_{\benign}\big(\prob_{\benign}[\hat{o}=y|{\bf w}, y=1]\prob_{\benign}[y=1]+\prob_{\benign}[\hat{o}=y|{\bf w}, y=0]\prob_{\benign}[y=0] \big).
    \end{split}
\end{equation}

Next, we will derive a lower bound for $\prob_{\dist}[\hat{o}=y|y=\tilde{y}, \mathbf{w}]$ for $\distset$ and for all $\tilde{y}\in\{0, 1\}$.

\paragraph{For $y=1$:}
We have
\begin{align*}
    &\prob_{\dist}[\hat{o}=y|{\bf w}, y=1] =  \prob_{\dist}[s_*+\sum_{i\in \ci}s_i+\sum_{j\in \cj} s_j - (|\cj|+1/2) \geq 0|y=1]\\
    & 1-\prob_{\dist}[s_*+\sum_{i\in \ci}s_i+\sum_{j\in \cj} s_j - (|\cj|+1/2) < 0|y=1]=1-\prob_{\dist}[s_*+\sum_{i\in \ci}s_i+\sum_{j\in \cj} s_j  < |\cj|+1/2|y=1]
\end{align*}
where $\prob_{\dist}[s_*=1|y=1]=\alpha_{*, \dist}$  (resp. $\prob_{\dist}[s_i=1|y=1]=\alpha_{i, \dist}$ and $\prob_{\dist}[s_j=1|y=1]=1-\epsilon_{j, \dist}$).

We let $$\Psi_{\dist, y=1}:=s_*+\sum_{i\in \ci}s_i+\sum_{j\in \cj} s_j - (|\cj|+1/2)$$ and $$\hat{\Psi}_{\dist, y=1}:=s_*+\sum_{i\in \ci}s_i+\sum_{j\in \cj} s_j=\Psi_{\dist, y=1} + |\cj|+1/2.$$

Similarly, the expected values of $\Psi_{\dist, y=1}$ and $\hat{\Psi}_{\dist, y=1}$ over $s_*, s_i$ and $s_j$ are given by $\mu_{\Psi_{\dist, y=1}}$ and $\mu_{\hat{\Psi}_{\dist, y=1}}$, respectively. Precisely, 

$$\mu_{\Psi_{\dist, y=1}}=\alpha_{*, \dist}-1/2+\sum_{i\in \ci}\alpha_{i, \dist}-\sum_{j\in \cj}\epsilon_{j, \dist}$$
and
$$\mu_{\hat{\Psi}_{\dist, y=1}}=\alpha_{*, \dist}+\sum_{i\in \ci}\alpha_{i, \dist}+\sum_{j\in \cj}(1-\epsilon_{j, \dist})=\mu_{\Psi_{\dist, y=1}}+|\cj|+1/2$$

We then write $\prob_{\dist}[\hat{o}\neq y|{\bf w}, y=1]$ as
\begin{align*}
   \prob_{\dist}[\hat{o}\neq y|{\bf w}, y=1] =  \prob[\Psi_{\dist, y=1} < 0] \leq \exp(-\delta^2_{\dist, y=1}\mu_{\hat{\Psi}_{\dist, y=1}}/2)
\end{align*}
where 
$$\delta_{\dist, y=1}=1-\frac{|\cj|+1/2}{\mu_{\hat{\Psi}_{\dist, y=1}}}=\frac{\mu_{{\Psi}_{\dist, y=1}}}{\mu_{\hat{\Psi}_{\dist, y=1}}}.$$

Let now $\gamma_{\dist, y=1}$ be the difference between true and false rates of sensors normalized over preventative models when $y=1$ such that

$$\gamma_{\dist, y=1}:= \frac{1}{|\cj|+1}(\alpha_{*, \dist}-1/2+\sum_{i\in\ci}\alpha_{i, \dist}-\sum_{j\in\cj}\epsilon_{j, \dist}).$$

Noting that $\mu_{\Psi_{\dist, y=1}}=(|\cj|+1)\gamma_{\dist, y=1}$, we have $\delta_{\dist, y=1}=\frac{(|\cj|+1)\gamma_{\dist,y=1}}{(|\cj|+1)\gamma_{\dist}+|\cj|+1/2}$ and $\mu_{\hat{\Psi}_{y=1}}=(|\cj|+1)\gamma_{\dist, y=1}+|\cj|+1/2$. Using Lemma~\ref{lemma} for a Poisson random variable $\hat{\Psi}_{y=1}$, we bound $\prob_{\dist}[\hat{o}\neq y|{\bf w}, y=1]$ as
\begin{equation}\label{eqn: y=1 final}
    \begin{split}
   &\prob_{\dist}[\hat{o}\neq y|{\bf w}, Y=1] =  \prob[\Psi_{\dist, y=1} < 0] = \prob[\hat{\Psi}_{\dist, y=1} < |\cj|+1/2] \leq \exp(-\delta^2_{\dist, y=1}\mu_{\hat{\Psi}_{\dist, y=1}}/2)\\
   &= \exp\bigg(-\frac{ (|\cj|+1)^2\gamma_{\dist, y=1}^2}{2\Big( (|\cj|+1)\gamma_{\dist, y=1}+|\cj|+1/2\Big)}\bigg){\leq}\exp\bigg(-\frac{ (|\cj|+1)^2\gamma_{\dist, y=1}^2}{2\Big( (|\cj|+1)\gamma_{\dist, y=1}+|\cj|+1\Big)}\bigg)\\
   &\hspace{19.5em}= \exp\bigg(-(|\cj|+1)\frac{\gamma^2_{\dist, y=1}}{2(\gamma_{\dist, y=1}+1)}\bigg)
    \end{split}
\end{equation}

\paragraph{For $y=0$:} 
We have
\begin{align*}
    &\prob_{\dist}[\hat{o}=y|{\bf w}, y=0] =  \prob_{\dist}[s_*-1/2+\sum_{i\in \ci}s_i-\sum_{j\in \cj} 1-s_j \leq 0|y=0]\\
    & 1-\prob_{\dist}[s_*-1/2+\sum_{i\in \ci}s_i-\sum_{j\in \cj} 1-s_j > 0|y=0]= 1-\prob_{\dist}[-s_*+1/2-\sum_{i\in \ci}s_i+\sum_{j\in \cj} 1-s_j < 0|y=0]\\
    &= 1-\prob_{\dist}[-s_*+1-1/2+\sum_{i\in \ci}1-s_i-|\ci|+\sum_{j\in \cj} 1-s_j < 0|y=0]\\
    &=1-\prob_{\dist}[-s_*+1+\sum_{i\in \ci}1-s_i+\sum_{j\in \cj} 1-s_j < |\ci|+1/2|y=0]
\end{align*}
where $\prob_{\dist}[s_*=1|y=0]=1-\alpha_{*, \dist}$  (resp. $\prob_{\dist}[s_i=1|y=0]=\epsilon_{i, \dist}$ and $\prob_{\dist}[s_j=1|y=0]=1-\alpha_{j, \dist}$).

We let $$\Psi_{\dist, y=0}:=1-s_*+\sum_{i\in \ci}1-s_i+\sum_{j\in \cj}1- s_j - (|\ci|+1/2)$$ and $$\hat{\Psi}_{\dist, y=0}:=1-s_*+\sum_{i\in \ci}1-s_i+\sum_{j\in \cj}1- s_j=\Psi_{\dist, y=0} + |\ci|+1/2.$$

Similarly, the expected values of $\Psi_{\dist, y=0}$ and $\hat{\Psi}_{\dist, y=0}$ over $s_*, s_i$ and $s_j$ are given by $\mu_{\Psi_{\dist, y=0}}$ and $\mu_{\hat{\Psi}_{\dist, y=0}}$, respectively. Precisely, 

$$\mu_{\Psi_{\dist, y=0}}=\alpha_{*, \dist}-1/2-\sum_{i\in \ci}\epsilon_{i, \dist}+\sum_{j\in \cj}\alpha_{j, \dist}$$
and
$$\mu_{\hat{\Psi}_{\dist, y=0}}=\alpha_{*, \dist}+\sum_{i\in \ci}1-\epsilon_{i, \dist}+\sum_{j\in \cj}\alpha_{j, \dist}=\mu_{\Psi_{\dist, y=0}}+|\ci|+1/2$$

We then write $\prob_{\dist}[\hat{o}\neq y|{\bf w}, y=0]$ as
\begin{align*}
   \prob_{\dist}[\hat{o}\neq y|{\bf w}, y=0] =  \prob[\Psi_{\dist, y=0} < 0] \leq \exp(-\delta^2_{\dist, y=0}\mu_{\hat{\Psi}_{\dist, y=0}}/2)
\end{align*}
where 
$$\delta_{\dist, y=0}=1-\frac{|\ci|+1/2}{\mu_{\hat{\Psi}_{\dist, y=0}}}=\frac{\mu_{{\Psi}_{\dist, y=0}}}{\mu_{\hat{\Psi}_{\dist, y=0}}}.$$

Let now $\gamma_{\dist, y=0}$ be the difference between true and false rates of sensors normalized over permissive models when $y=0$ such that

$$\gamma_{\dist, y=0}:= \frac{1}{|\ci|+1}(\alpha_{*, \dist}-1/2+\sum_{j\in\cj}\alpha_{j, \dist}-\sum_{i\in\ci}\epsilon_{i, \dist}).$$

Noting that $\mu_{\Psi_{\dist, y=0}}=(|\ci|+1)\gamma_{\dist, y=0}$, we have $\delta_{\dist, y=0}=\frac{(|\ci|+1)\gamma_{\dist}}{(|\ci|+1)\gamma_{\dist}+|\ci|+1/2}$ and $\mu_{\hat{\Psi}_{y=0}}=(|\ci|+1)\gamma_{\dist, y=0}+|\ci|+1/2$. Using Lemma~\ref{lemma} for a Poisson random variable $\hat{\Psi}_{y=0}$, we bound $\prob_{\dist}[\hat{o}\neq y|{\bf w}, y=0]$ as
\begin{equation}\label{eqn: y=0 final}
    \begin{split}
   &\prob_{\dist}[\hat{o}\neq y|{\bf w}, y=0] =  \prob[\Psi_{\dist, y=0} < 0] = \prob[\hat{\Psi}_{\dist, y=0} < |\ci|+1/2] \leq \exp(-\delta^2_{\dist, y=0}\mu_{\hat{\Psi}_{\dist, y=0}}/2)\\
   &= \exp\bigg(-\frac{ (|\ci|+1)^2\gamma_{\dist, y=0}^2}{2\Big( (|\ci|+1)\gamma_{\dist, y=0}+|\ci|+1/2\Big)}\bigg){\leq}\exp\bigg(-\frac{ (|\ci|+1)^2\gamma_{\dist, y=0}^2}{2\Big( (|\ci|+1)\gamma_{\dist, y=0}+|\ci|+1\Big)}\bigg)\\
   &\hspace{19em}=\exp\bigg(-(|\ci|+1)\frac{\gamma^2_{\dist, y=0}}{2(\gamma_{\dist, y=0}+1)}\bigg)
    \end{split}
\end{equation}


\paragraph{Last step:}
For convenience, let $n:=|\ci|=|\cj|$ and 
$$\gamma_{\dist}:=\min(\gamma_{\dist, y=1}, \gamma_{\dist, y=0}).$$ 
Using (\ref{eqn: y=1 final}) and (\ref{eqn: y=0 final}), we bound the pipeline accuracy in (\ref{eqn: pipeline accuracy proposition}) such that

\begin{equation}\label{eqn: pipeline accuracy proposition final}
    \begin{split}
    &\pipelineaccuracy \geq 1-\sum_{\distset}\pi_{\dist}\exp\bigg(-(n+1)\frac{\gamma^2_{\dist}}{2(\gamma_{\dist}+1)}\bigg)\\
    &\geq 1-\sum_{\distset}\pi_{\dist}\exp\bigg(-(n+1)\frac{\gamma^2_{\dist}}{4}\bigg).
    \end{split}
\end{equation}

Hence, if 
\begin{equation}\label{eqn: end of proposition}
    \begin{split}
    1-\exp\bigg(-(n+1)\frac{\gamma^2_{\dist}}{4}\bigg)>\mainaccuracy
    \end{split}
\end{equation}
for all $\distset$, then we have $\pipelineaccuracy>\mainaccuracy$. Manipulating (\ref{eqn: end of proposition}) for all $\distset$ concludes the proof.

\end{proof}

\subsection{Proof of Corollary~\ref{cor: homogenous models}}\label{subsec: proof of corollary}
We recall the respective setting as follows. We assume that the auxiliary models are \emph{homogeneous} for each type: permissive or preventative. For example, $\alpha_k$ is fixed with respect to $k\in\mathcal{I}\cup\mathcal{J}$, hence we drop the subscripts, i.e., $\alpha_{k, \dist}=\alpha$ and $\epsilon_{k, \dist}=\epsilon$. 
We assume that the same number of auxiliary models are used, namely $|\mathcal{I}|=|\mathcal{J}|=n$, and that the classes are balanced with $\prob_{\dist}(y=1)=\prob_{\dist}(y=0)$, for all $\distset$. Finally, we let $\alpha_{*, \benign}=1$ and $\alpha_{*, \adv}=0$, and $\alpha-\epsilon>0$. Then, the following holds.
\begin{corollary*}[\recallthm]
The weighted robust accuracy of \sys in the homogeneous setting satisfies
\begin{equation*}
    \begin{split}
      \mathcal{A}^{\text{\sys}} \geq1-\exp\big(-2n{(\alpha-\epsilon)^2}\big).
    \end{split}
\end{equation*}
In particular, one has $\lim_{n\rightarrow \infty}\mathcal{A}^{\text{\sys}}=1.$
\end{corollary*}

\begin{proof}[Proof of Corollary~\ref{cor: homogenous models}]
First, for $\alpha_{*,\benign}=1$ and $\alpha_{*, \adv}=0$, using (\ref{eqn: w_*+b_*}) and (\ref{eqn: -w_*+b_*}), we note that 
\begin{align*}
    \mainsensorweight=b_*= 0.
\end{align*}

Secondly, in the homogeneous case, the conditional independence reflects to the mixture model and models become conditionally independent in the mixture model as well. That is, the condition on the other models in (\ref{eqn: w_i+b+i}, \ref{eqn: b+i}, \ref{eqn: w_j+b+j}, \ref{eqn: b_j}) drops and we have closed form expression for all optimal parameters. Namely, for $\alpha_{i, \dist}=\alpha_{j, \dist}=\alpha$ and $\epsilon_{i, \dist}=\epsilon_{j, \dist}=\epsilon$ with $\alpha>\epsilon$, once can deduce from (\ref{eqn: w_i+b+i}, \ref{eqn: b+i}, \ref{eqn: w_j+b+j}, \ref{eqn: b_j}) that the optimal weight of auxiliary sensors are given by $w_i=w_j=\log\frac{\alpha}{\epsilon}$ and $b=\sum_{i\in\mathcal{I}}b_i-\sum_{j\in\mathcal{J}}b_j=\sum_{i\in\mathcal{I}}\log\frac{1-\alpha}{1-\epsilon}-\sum_{j\in\mathcal{J}}\log\frac{1-\alpha}{1-\epsilon}=0$. Also, $w_i=w_j>0$ for $\alpha>\epsilon$. For this setting, we can write out $\pipelineaccuracy$ as follows.

\begin{align*}
    &\pipelineaccuracy=\E_{\distset}\E_{y\sim\mathcal{Y}}\big[\prob[d_{*, \dist}+yd_{\mathcal{I},\dist}+(1-y)d_{\mathcal{I},\dist}>0|y]\big]\stackrel{(*)}{=}\E_{y\sim\mathcal{Y}}\big[\prob[yd_{\mathcal{I},\dist}+(1-y)d_{\mathcal{I},\dist}>0|y]\big]\\
    &\stackrel{(**)}{=}\frac{1}{2}\big(\prob[d_{\mathcal{I},\dist}>0|y=1]+\prob[d_{\mathcal{J},\dist}>0|y=0]\big)\stackrel{(***)}{=}\prob[d_{\mathcal{I},\dist}>0|y=1]
\end{align*}
where (*) follows from the homogeneity of models over both benign and adversarial distributions as well as that $d_{*, \dist}=w_*(2s_*-1)=0$, (**) follows from the class balance, and finally (***) stems from the symmetry.

Let $B(n, p)$ denote the Binomial distribution with count parameter $n$ and success probability $p$. Let also that $d_{\alpha}$ and $d_{\epsilon}$ be random variables with Binomial distributions such that $d_{\alpha}\sim B(n, \alpha)$ and $d_{\epsilon}\sim B(n, \epsilon)$. We then rewrite the \Acc of \sys as follows.
\[
\pipelineaccuracy=\prob[d_{\mathcal{I},\dist}>0|y=1]=1-\prob[d_{\mathcal{I},\dist}<0|y=1]=1-\prob[w(d_{\alpha}-d_{\epsilon})<0|y=1]=1-\prob[d_{\alpha}-d_{\epsilon}<0]
\]
where the last equality follows from that $w=\log\frac{\alpha}{\epsilon}>0$.

We then review the Bounded Differences Inequality which will enable us to bound the tail probability $\prob[d_{\alpha}-d_{\epsilon}<0|y=1]$.

\begin{theorem}[Bounded Differences Inequality~\cite{boucheron2013concentration}]\label{thm: bounded differences}
Assume that a function $\phi: \mathcal{X}^n\rightarrow \mathbb{R}$ of independent random variables $X_1, ..., X_n \in \mathcal{X}$ satisfies the bounded differences property with constants $c_1, ..., c_n$. Denote $v^2=\sum_{i=[n]}c_i^2$ and $Z=\phi(X_1, ..., X_n)$. $Z$ satisfies:
\[
\prob(Z-\mathbb{E}(Z)>t)\leq \exp\big({-\frac{2t^2}{v^2}}\big) \hspace{1em} \text{and} \hspace{1em} 
\prob(Z-\mathbb{E}(Z)<-t)\leq \exp\big({-\frac{2t^2}{v^2}}\big) .
\]
\end{theorem}
We refer to, for example,~\cite{boucheron2013concentration} for a proof of Theorem~\ref{thm: bounded differences}.

Using Theorem~\ref{thm: bounded differences} for $Z=d_{\alpha}-d_{\epsilon}$, $\pipelineaccuracy$ can be bounded as:
\begin{align*}
    \pipelineaccuracy=1-\prob[d_{\alpha}-d_{\epsilon}<0]=1-\prob[d_{\alpha}-d_{\epsilon}-\E[d_{\alpha}-d_{\epsilon}]<-\E[d_{\alpha}-d_{\epsilon}]]=1-\prob[d_{\alpha}-d_{\epsilon}-n(\alpha-\epsilon)<-n(\alpha-\epsilon)].
\end{align*}
Moreover, for $t=n(\alpha-\epsilon)$ and $v^2 = n$ we finally have
\[
\pipelineaccuracy=1-\prob[d_{\alpha}-d_{\epsilon}-n(\alpha-\epsilon)<-n(\alpha-\epsilon)]\geq 1-\exp\big( -2(n^2(\alpha-\epsilon)^2)/n\big)=\geq 1-\exp\big( -2n(\alpha-\epsilon)^2\big)
\]
concludes the proof for the lower bound.

As the final step, we will prove that $\pipelineaccuracy>\mainaccuracy$. Note that $\mainaccuracy=\E_{\distset}\E_{y\sim\mathcal{Y}}\big[\prob[d_{*, \dist}>0|y]\big]=\pi_{\benign}\alpha_{*, \benign}+\pi_{\adv}\alpha_{*, \adv}=1/2\cdot 1+1/2\cdot 0 =1/2.$ Therefore, it only remains to analyze whether $\pipelineaccuracy>1/2$ or not. Towards that, we state the following result.

\begin{lemma}[On the comparison of two binomial random variables]\label{lemma: two binomial tail}
Let $p, q \in [0, 1]$ denote the success probabilities for two Binomial random variables. If $p>q$, then $\prob[X>Y]>\frac{1}{2}$.
\end{lemma}
\begin{proof}
Let $X$ and $Y$ be random variables such that $X\sim B(n, p)$ and $Y\sim B(n, q)$. $Z:=X-Y$ can be shown to have the following probability mass function
\begin{equation*}
\prob(Z=z) = \begin{cases}
\sum_{k\in\{0\}\cup[n]}f(k+z,n,p)f(k,n,q) &\text{if \hspace{.2em} $x\geq 0$}\\
\sum_{k\in\{0\}\cup[n]}f(k,n,p)f(k+z,n,q)&\text{elsewhere}\\
\end{cases}
\end{equation*}
where $f(k, n, p)={n\choose k} p^k(1-p)^{n-k}$ for $k\leq n$. Moreover, we have
$$
\prob(Z>0)=\prob(X-Y>0)=\sum_{\substack{z\in[n]\\ k\in\{0\}\cup[n]}}f(k+z,n,p)f(k,n,q), \hspace{1.5em} \prob(Z\leq0)=\sum_{\substack{z\in[n]\\ k\in\{0\}\cup[n]}}f(k,n,p)f(k+z,n,q).
$$
Note that if $p>q$, then $f(k+z,n,p)f(k,n,q)>f(k,n,p)f(k+z,n,q)$ for fixed $n, k \geq 0$. Hence, the summation over $z\in[n], k\in\{0\}\cup[n]$ leads to $\prob(Z>0)>\prob(Z\leq0)$. It is further implied by $\prob(Z>0)+\prob(Z\leq0)=1$ that $\prob(Z>0)>\frac{1}{2}$. 
\end{proof}
Using Lemma \ref{lemma: two binomial tail} for $X=d_{\alpha}$ and $Y=d_{\epsilon}$ as well as that $\alpha>\epsilon$, we have
\[
\pipelineaccuracy=\prob[d_{\alpha}-d_{\epsilon}>0]=\prob[d_{\alpha}>d_{\epsilon}]>1/2=\mainaccuracy.
\]
Hence the proof results.
\end{proof}


\section{Experimental Details}
\subsection{Detailed Setup of Baselines}
\label{appendix:detailed-setup-of-baselines}
To demonstrate the superior \sys, we compare it with two state-of-the-art baselines: \textbf{adversarial training}~\citep{madry2017towards} and \textbf{DOA}~\citep{wu2019defending}, which are strong defenses against $\mathcal{L}_p$ bounded attacks and physically realizable attacks respectively.

For adversarial training, we adopt $\mathcal{L}_\infty$ bound $\epsilon\in\{4,8,16,32\}$ during training phase. Since adversarial training failed to make progress for $\epsilon\in\{16,32\}$, we use the curriculum training version~\citep{cai2018curriculum}, where the model is firstly trained on smaller $\epsilon$ with $\epsilon$ gradually increasing to the largest bound. For all versions of adversarial training in our implementation, we adopt 40 iterations of PGD attack with a step size of 1/255. In all cases, pixels are in $0 \sim 255$ range and the retraining takes 3000 training iterations with a batch size of 200 for each random iteration.

For DOA, we consider adversarial patches with the size of $5\times5$ and $7\times7$ respectively for rectangle occlusion during retraining. For both cases, we use an exhaustive search to pick the attack location and perform 30 iterations PGD inside the adversarial patch to generate noise. The retraining takes 5000 training iterations and the batch size is 200.

Thus, in total, we have 7 baseline CNN models~(1 standard CNN model, 4 adversarially trained CNN models, 2 DOA trained CNN models), and we use id numbers $1\sim7$ to denote ``GTSRB-CNN", ``AdvTrain~($\epsilon=4$)", ``AdvTrain~($\epsilon=8$)", ``AdvTrain~($\epsilon=16$)", ``AdvTrain~($\epsilon=32$)", ``DOA~(5x5)", ``DOA~(7x7)", respectively in Figure~\ref{fig:general_accuracy}(a).

\subsection{Details of Attacks and Corruptions}
\label{appendix:details-of-attacks-and-corruptions}

Since our constructed \sys pipeline is a compound model consisting of multiple sub-models, some of which are not differentiable, we can not directly generate adversarial examples via the standard end-to-end white-box attack. Alternatively, we further propose three different attack settings to evaluate the robustness of our \sys pipeline: \textbf{1)White-box sensor attack,} where adversarial examples are generated by directly applying gradient methods to the main task model of the \sys pipeline in a white-box fashion; \textbf{2)Black-box sensor attack.} In this setting, we train substitute model of the main task model using the same model architecture and the same standard training data, and generate adversarial examples with this substitute model; \textbf{3)Black-box pipeline attack,} in which we generate adversarial examples with a substitute model, which is obtained via distilling the whole \sys pipeline. For this setting, a substitute model with the same GTSRB-CNN architecture is trained on a synthetic training set, where all the images are from the original training set, while the labels are generated by the pipeline model. Then all the models are evaluated on the same set of adversarial test samples crafted on the trained substitute.

Specifically, \textbf{1) For $\mathcal{L}_{\infty}$ attack,} we consider the strength of $\epsilon \in \{4, 8, 16, 32\}$ in our evaluation. 1000 iterations of standard PGD~\citep{madry2017towards} with a step size of 1/255 is used
to craft the adversarial examples, and all the three attack settings introduced above are respectively applied; 
\textbf{2) For unforeseen attacks}, we consider the Fog,
Snow, JPEG, Gabor and Elastic attacks suggested in \citet{kang2019testing}, which are all gradient-based worst-case adversarial attacks, generating diverse test distributions distinct from the common $\mathcal{L}_p$ bounded attacks. For Fog attack, we consider $\epsilon \in \{256,512\}$. For Snow attack, we evaluate for $\epsilon \in \{0.25,0.75\}$ respectively. For JPEG attack, we
adopt the parameters $\epsilon \in \{0.125, 0.25\}$. For Gabor attack, $\epsilon \in \{20,40\}$ are tested. Finally, $\epsilon \in \{1.5,2.0\}$ are considered for Elastic attack. Since all of these attacks are gradient based, we also apply the three different settings above to generate adversarial examples respectively; 
\textbf{3) For physical attacks on stop signs,} we directly use the same stickers (i.e., the same color and mask) generated in \citet{eykholt2018robust} to attack the same 40 stop sign samples, and we also adopt the same end-to-end classification model used in \citet{eykholt2018robust} to construct \sys model. Since our ultimate goal is defense, we follow the same practice in \citet{wu2019defending}, where we only consider the digital representation of the attack instead of the real physical implementation, ignoring issues like the attack’s robustness to different viewpoints and environments. Thus, we implement the physical stop sign attack by directly placing the stickers on the stop sign samples in digital space; 
\textbf{4) For common corruptions,} we evaluate our models with the
15 categories of corruptions suggested in \citet{hendrycks2019benchmarking}. Empirically, in our traffic sign identification task, only 3 types of corruptions out of the 15 categories effectively reduce the accuracy (with a margin over $10\%$) of our standard GTSRB-CNN model. Thus, we only present the evaluation results of our models against the three most successful corruption — Fog, Contrast, Brightness. (Note that, here we use Fog corruption which is similar to the Fog attack in unforeseen attacks. However, they are different in that the Fog corruption here is not adversarially generated like that in Fog attack.)

Thus, based on different attack/corruption methods and attack settings, in total, we have 46 different attacks/corruptions. In Figure~\ref{fig:general_accuracy}(b)(c), we use id numbers $1\sim46$ to denote all the attacks we evaluate on, and we present the correspondence between id numbers and attacks in Table~\ref{table:correspondece_between_id_and_attack}. Moreover, besides the two representative baselines presented in the main body, we present the complete robustness improvement results under the 46 types of attacks/corruptions for all baselines in Figure~\ref{fig:complete_robustness_improvement}.

\begin{figure}
    \centering
    
    \subfigure[Baseline: KEMLP over GTSRB-CNN]{\includegraphics[width=0.48\textwidth]{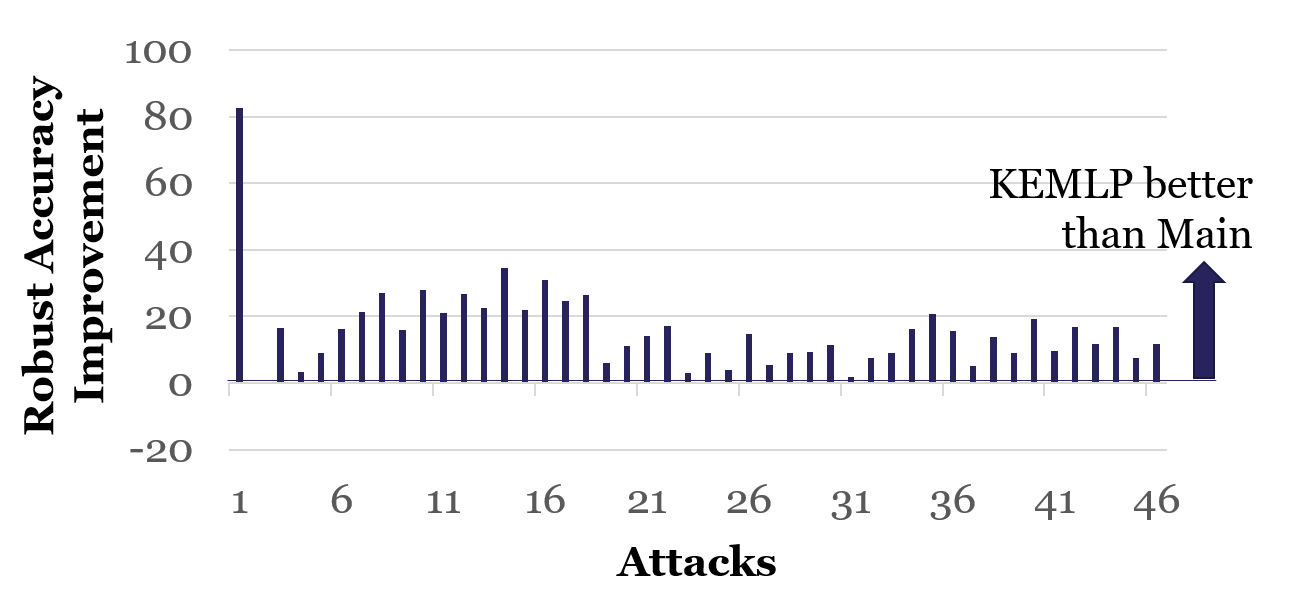}} 
    \subfigure[Baseline: KEMLP over AdvTrain~($\epsilon = 4 $)]{\includegraphics[width=0.48\textwidth]{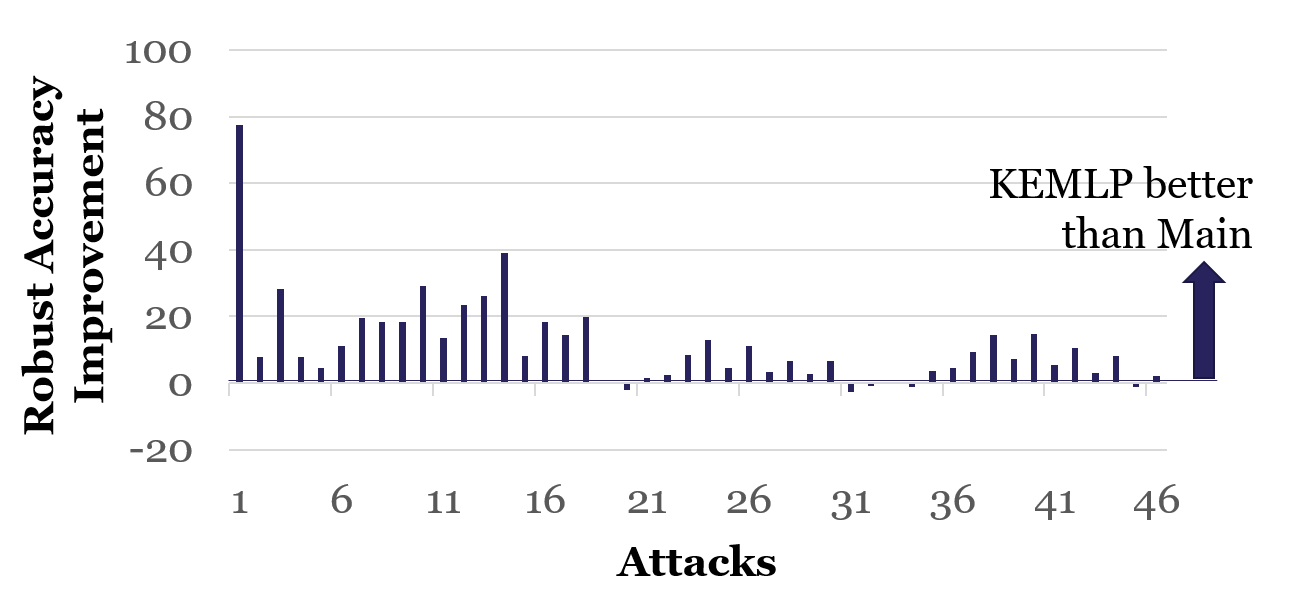}} 
    \subfigure[Baseline: KEMLP over AdvTrain~($\epsilon = 8 $)]{\includegraphics[width=0.48\textwidth]{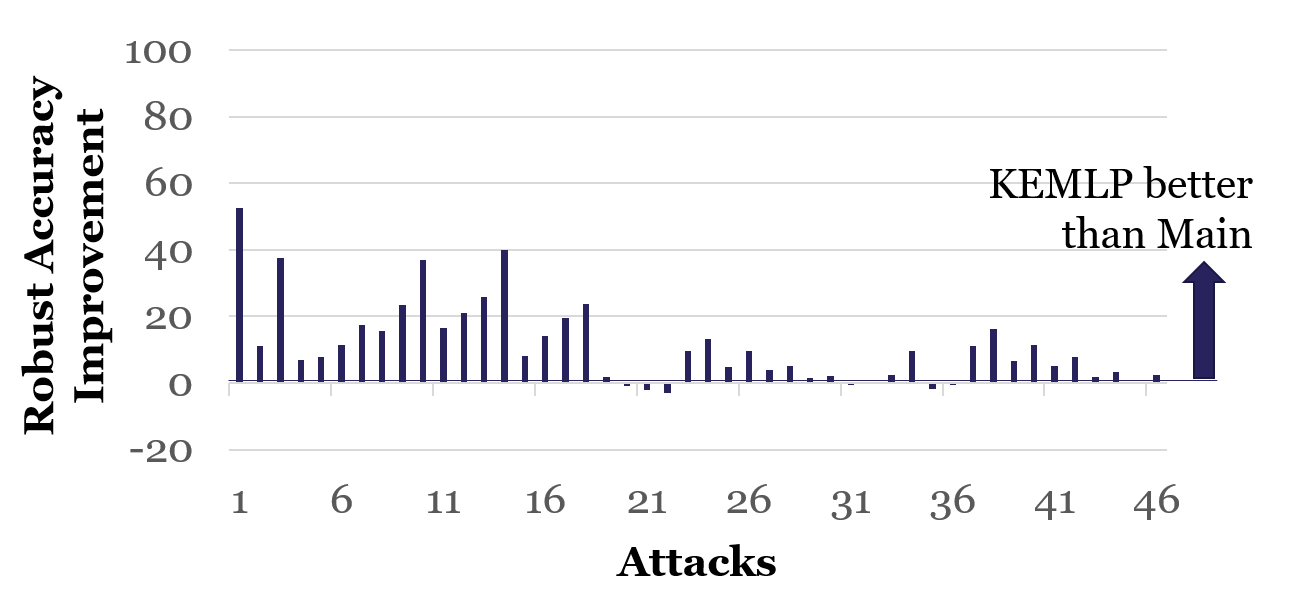}} 
    \subfigure[Baseline: KEMLP over AdvTrain~($\epsilon = 16 $)]{\includegraphics[width=0.48\textwidth]{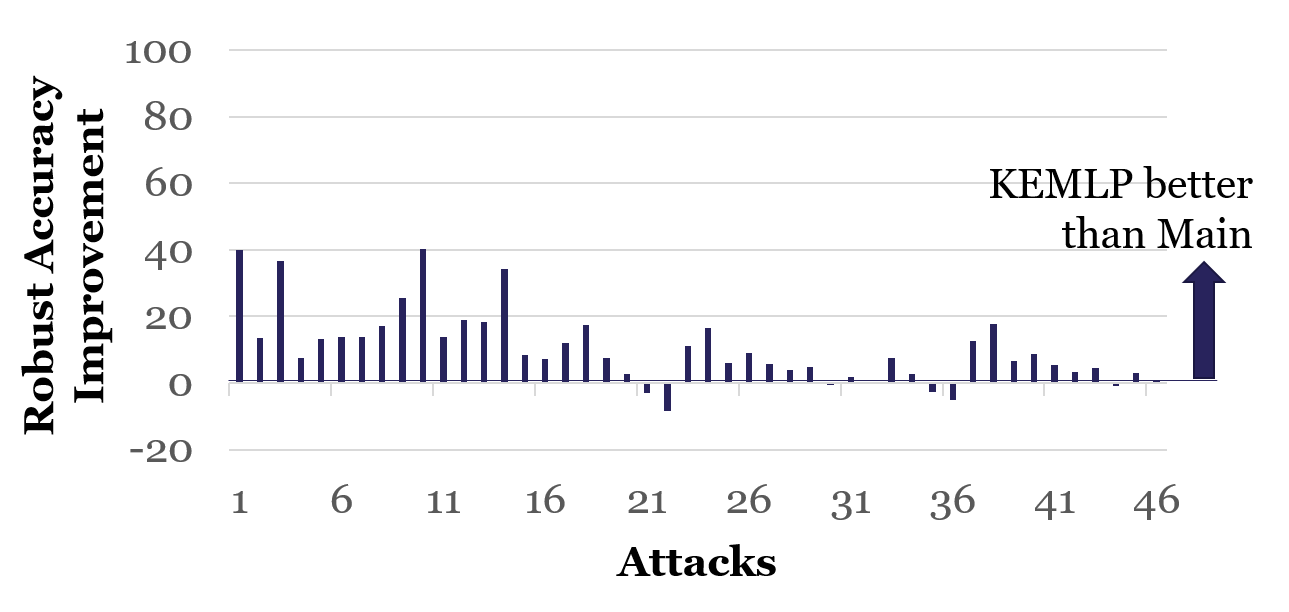}}
    \subfigure[Baseline: KEMLP over AdvTrain~($\epsilon = 32 $)]{\includegraphics[width=0.48\textwidth]{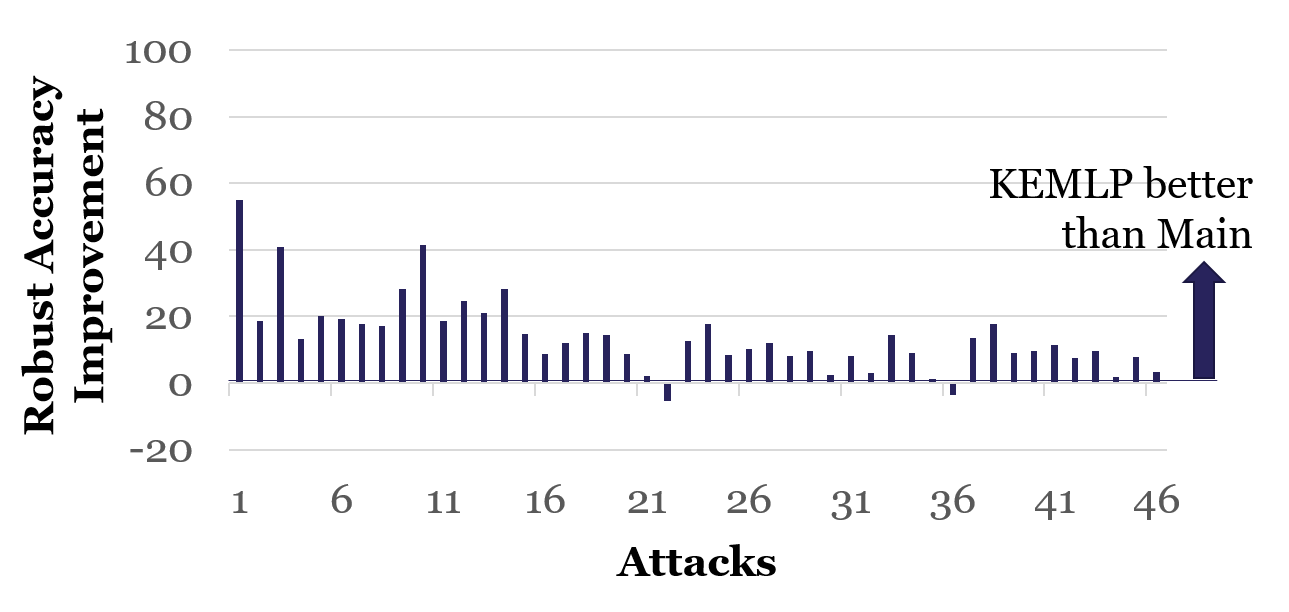}}
    \subfigure[Baseline: KEMLP over DOA~(5x5)]{\includegraphics[width=0.48\textwidth]{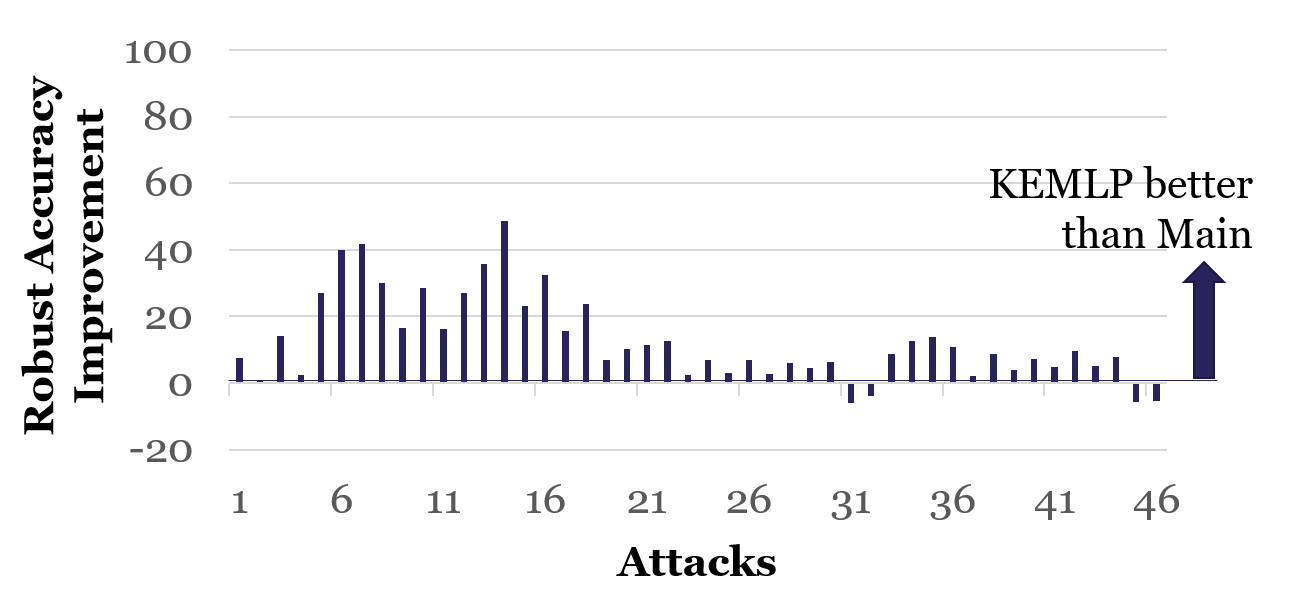}}
    \subfigure[Baseline: KEMLP over DOA~(7x7)]{\includegraphics[width=0.48\textwidth]{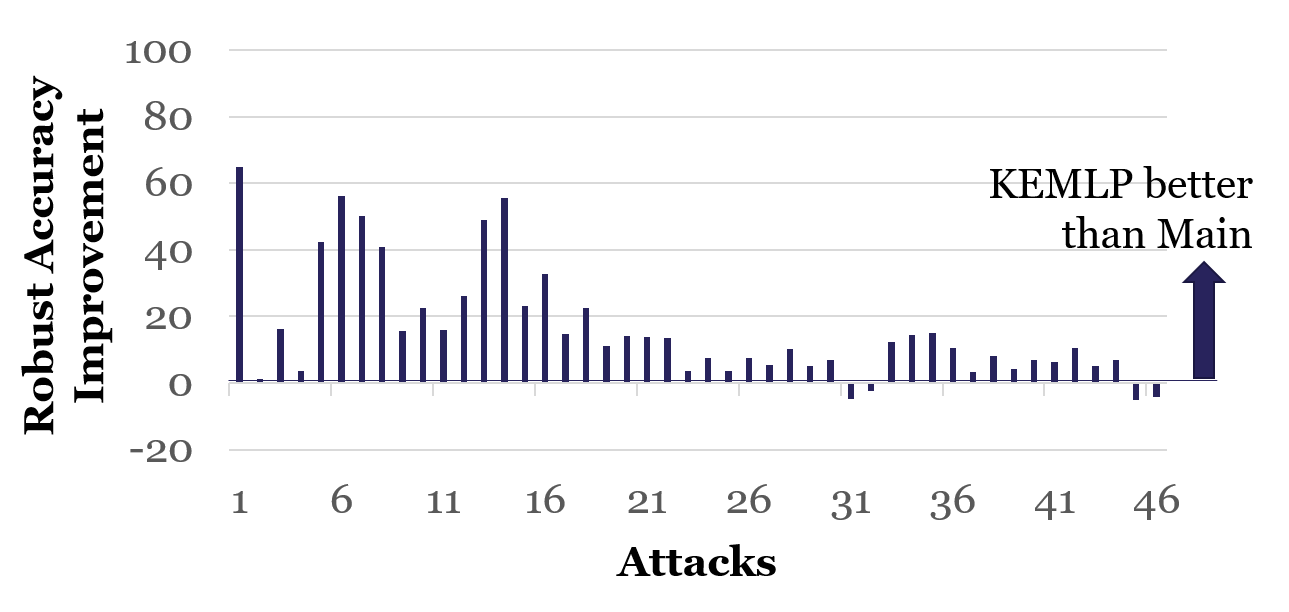}}
    
    \caption{Improvement of robustness accuracy after being enhanced by \sys.~($\alpha = 0.5$)}
    \label{fig:complete_robustness_improvement}
\end{figure}

\begin{table*}[tp]
    \caption{Correspondence between id numbers and attacks/corruptions}
    \label{table:correspondece_between_id_and_attack}
    \centering
    \resizebox{1\textwidth}{!}{
    \begin{tabular}{c|c|c|c|c}
        \toprule
            1 & 2 & 3 & 4 & 5  \\
            \cline{1-5} Physical Attack & Fog Corruption & Contrast Corruption & Brightness Corruption & $\mathcal{L}_\infty$ Attack~($\epsilon=4$, whitebox sensor)\\
            \hline
        \hline
            6 & 7 & 8 & 9 & 10 \\
            \cline{1-5} $\mathcal{L}_\infty$ Attack~($\epsilon=8$, whitebox sensor) & $\mathcal{L}_\infty$ Attack~($\epsilon=16$, whitebox sensor) & $\mathcal{L}_\infty$ Attack~($\epsilon=32$, whitebox sensor) & Fog Attack~($\epsilon=256$, whitebox sensor)  & Fog Attack~($\epsilon=512$, whitebox sensor) \\
            \hline
        \hline
            11 & 12 & 13 & 14 & 15  \\
            \cline{1-5} Snow Attack~($\epsilon=0.25$, whitebox sensor) & Snow Attack~($\epsilon=0.75$, whitebox sensor) & Jpeg Attack~($\epsilon=0.125$, whitebox sensor) & Jpeg Attack~($\epsilon=0.25$, whitebox sensor) & Gabor Attack~($\epsilon=20$, whitebox sensor)\\
            \hline
        \hline
            16 & 17 & 18 & 19 & 20  \\
            \cline{1-5} Gabor Attack~($\epsilon=40$, whitebox sensor) & Elastic Attack~($\epsilon=1.5$, whitebox sensor) & Elastic Attack~($\epsilon=2.0$, whitebox sensor) & $\mathcal{L}_\infty$ Attack~($\epsilon=4$, blackbox sensor) & $\mathcal{L}_\infty$ Attack~($\epsilon=8$, blackbox sensor)\\
            \hline
        \hline
            21 & 22 & 23 & 24 & 25  \\
            \cline{1-5} $\mathcal{L}_\infty$ Attack~($\epsilon=16$, blackbox sensor) & $\mathcal{L}_\infty$ Attack~($\epsilon=32$, blackbox sensor) & Fog Attack~($\epsilon=256$, blackbox sensor)  & Fog Attack~($\epsilon=512$, blackbox sensor) & Snow Attack~($\epsilon=0.25$, blackbox sensor)\\
            \hline
        \hline
            26 & 27 & 28 & 29 & 30\\
            \cline{1-5} Snow Attack~($\epsilon=0.75$, blackbox sensor) & Jpeg Attack~($\epsilon=0.125$, blackbox sensor) & Jpeg Attack~($\epsilon=0.25$, blackbox sensor) & Gabor Attack~($\epsilon=20$, blackbox sensor) & Gabor Attack~($\epsilon=40$, blackbox sensor)\\
            \hline
        \hline
            31 & 32 & 33 & 34 & 35\\
            \cline{1-5} Elastic Attack~($\epsilon=1.5$, blackbox sensor) & Elastic Attack~($\epsilon=2.0$, blackbox sensor) & $\mathcal{L}_\infty$ Attack~($\epsilon=4$, blackbox pipeline) & $\mathcal{L}_\infty$ Attack~($\epsilon=8$, blackbox pipeline) & $\mathcal{L}_\infty$ Attack~($\epsilon=16$, blackbox pipeline)\\
            \hline
        \hline
            36 & 37 & 38 & 39 & 40\\
            \cline{1-5} $\mathcal{L}_\infty$ Attack~($\epsilon=32$, blackbox pipeline) & Fog Attack~($\epsilon=256$, blackbox pipeline)  & Fog Attack~($\epsilon=512$, blackbox pipeline) & Snow Attack~($\epsilon=0.25$, blackbox pipeline) & Snow Attack~($\epsilon=0.75$, blackbox pipeline)\\
            \hline
        \hline
            41 & 42 & 43 & 44 & 45\\
            \cline{1-5} Jpeg Attack~($\epsilon=0.125$, blackbox pipeline) & Jpeg Attack~($\epsilon=0.25$, blackbox pipeline) & Gabor Attack~($\epsilon=20$, blackbox pipeline) & Gabor Attack~($\epsilon=40$, blackbox pipeline) & Elastic Attack~($\epsilon=1.5$, blackbox pipeline)\\
            \hline
        \hline
            46 &  &  &  & \\
            \cline{1-5} Elastic Attack~($\epsilon=2.0$, blackbox pipeline) &  &  & \\
            \hline
        \bottomrule
    
    \end{tabular}
    }
\end{table*}

\subsection{Implementation Details of KEMLP Pipeline for Traffic Sign Identification}
\label{appendix:implementation-details}


\begin{figure*}[t!]
\centering
\includegraphics[width=0.9\textwidth]{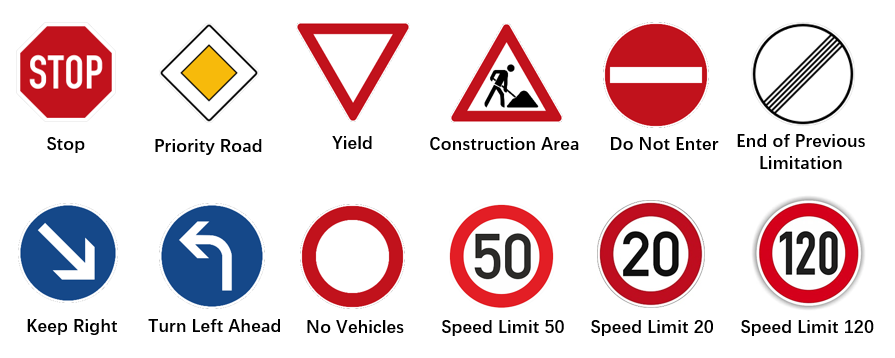}
\caption{The selected 12 types of signs from the full GTSRB.}
\label{fig:selected_signs}
\end{figure*}

To implement a nontrivial \sys pipeline for traffic sign identification, we need to design informative knowledge rules, connecting useful sensory information to each type of traffic sign. The full GTSRB dataset contains 43 types of signs, thus it requires a large amount of fine-grained sensory information and corresponding knowledge rules to distinguish between different signs, which requires a heavy engineering workload. Since the main purpose of this work is to illustrate the knowledge enhancement methodology rather than engineering practice, alternatively, we only consider a 12-class subset~(as shown in Figure~\ref{fig:selected_signs}) in our experiment, where the selected signs have diverse appearance and high frequencies.

For detailed \sys pipeline implementation, we consider two orthogonal domains --- logic domain and sensing domain, respectively. 

\begin{figure*}[t!]
\centering
\includegraphics[width=0.9\textwidth]{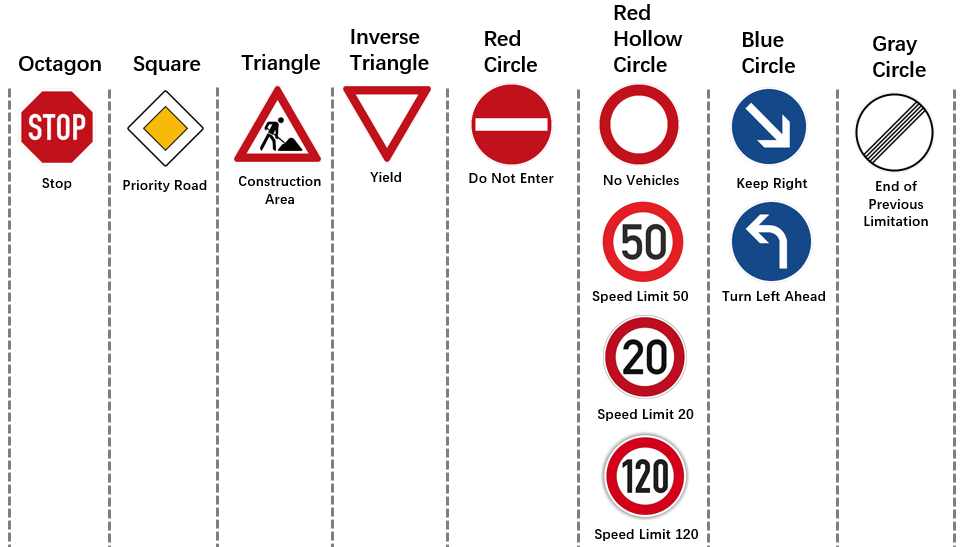}
\caption{Border patterns of the selected signs.}
\label{fig:border_type}
\end{figure*}

\underline{\textbf{In the logic domain}}, based on the specific tasks we need to deal with, we design a set of knowledge rules, which determine the basic logical structure of the predefined reasoning model. Specifically, for our task of traffic sign identification on the 12-class dataset, in total, we have designed 12 pieces of permissive knowledge rules and 12 pieces of preventative knowledge rules for the selected 12 types of signs. Each type of sign shares exactly one permissive knowledge rule and one preventative knowledge rule, respectively. 

In our design, we take \textit{border patterns} and \textit{sign contents} of the traffic signs as the sensory information to construct knowledge rules. As shown in Figure~\ref{fig:border_type}, based on the border pattern, we can always construct a preventative knowledge rule for each sign based on its border in the form as \textit{if it is a stop sign, it should be of the shape of octagon}. In our 12-class set, since there are six types of signs~(``Stop", ``Priority Road", ``Construction Area", ``Yield", ``Do Not Enter", ``End of Previous Limitation") sharing the unique border pattern, we also design an permissive rule for each of the six classes based on their borders, e.g. \textit{if the sign is of the shape of octagon, it must be a stop sign}. Then, for the rest of the six types~(``No Vehicles", ``Speed Limit 50", ``Speed Limit 20", ``Speed Limit 120", ``Keep Right", ``Turn Left Ahead"), whose borders can not uniquely determine their identity, we use their unique sign content to design permissive rules for them. Specifically, we define the content pattern \textit{Blank Circle}, \textit{Digits-20}, \textit{Digits-50}, \textit{Digits-120}, \textit{Arrow-Right-Down}, \textit{Arrow-Left-Ahead} to distinguish between these signs. We present the permissive relations in Figure~\ref{fig:permissive_relations}.

\begin{figure*}[t!]
\centering
\includegraphics[width=0.9\textwidth]{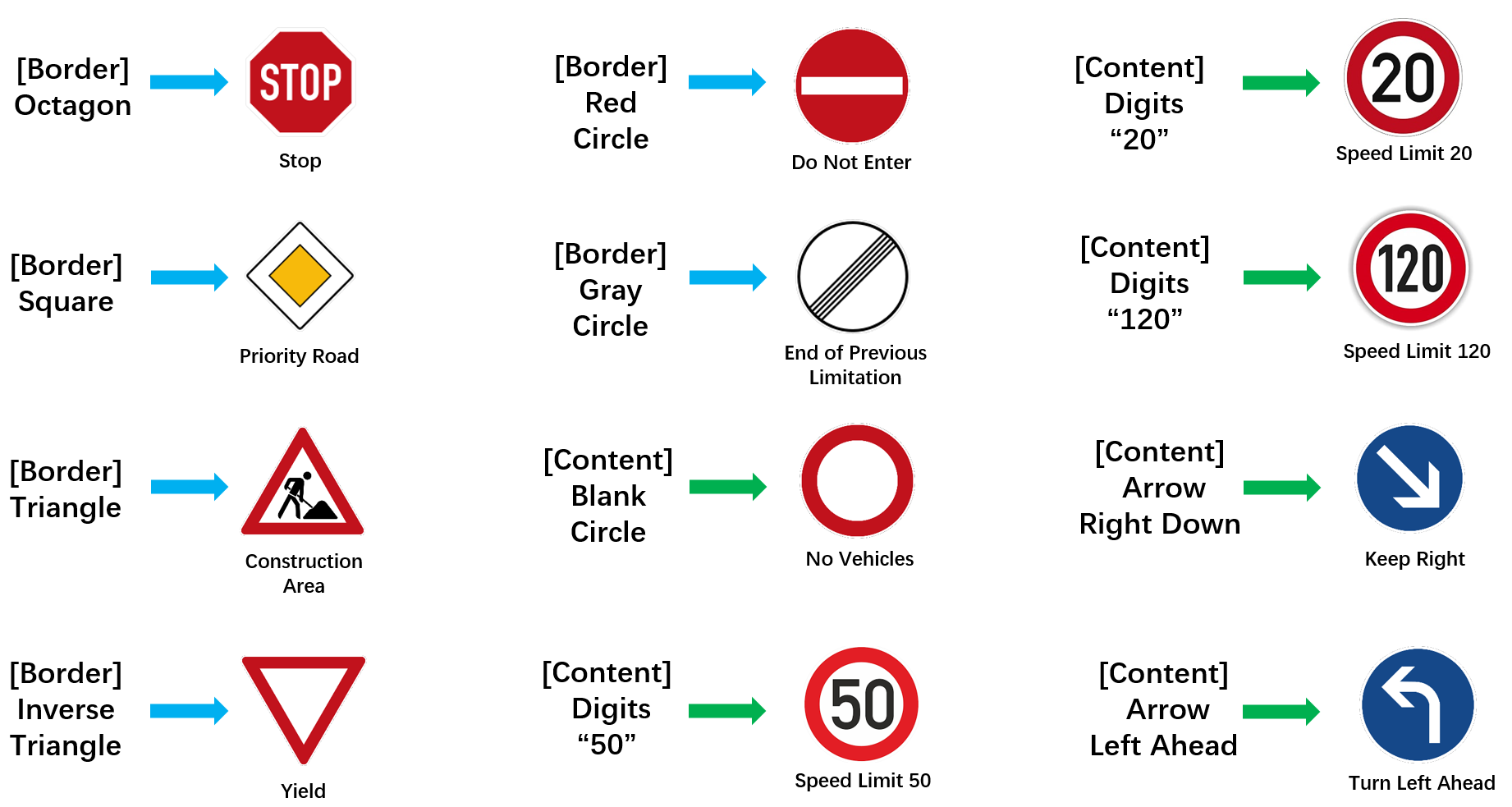}
\caption{permissive relations for each sign.}
\label{fig:permissive_relations}
\end{figure*}


\underline{\textbf{In the sensing domain}}, the principal task is to design a set of reliable auxiliary models to identify those sensory information required by the knowledge rules defined in the logic domain. For traffic sign identification, we adopt a non-neural pre-processing plus neural identification workflow to identify the border and content of each type. Specifically, to identify the border type~(e.g. shape and color), we first use GrabCut~\citep{rother2004grabcut} to get the mask of the sign and then discard all pixels of sign content and background, only retaining the border pixels, and finally a binary CNN classifier is used to make the statistical prediction~(e.g. predict whether the shape is octagon only based on the border pixels). For sign content, similarly, we first use GrabCut to filter out all irrelevant pixels except for the sign content, and then the edge operator will extract the contour of the content, finally CNN models are applied to recognize specific features like digits, arrows and characters. In Figure~\ref{fig:sensor_implementation}, we provide an overview of the workflow of our implemented auxiliary models.

\begin{figure*}[t!]
\centering
\includegraphics[width=0.5\textwidth]{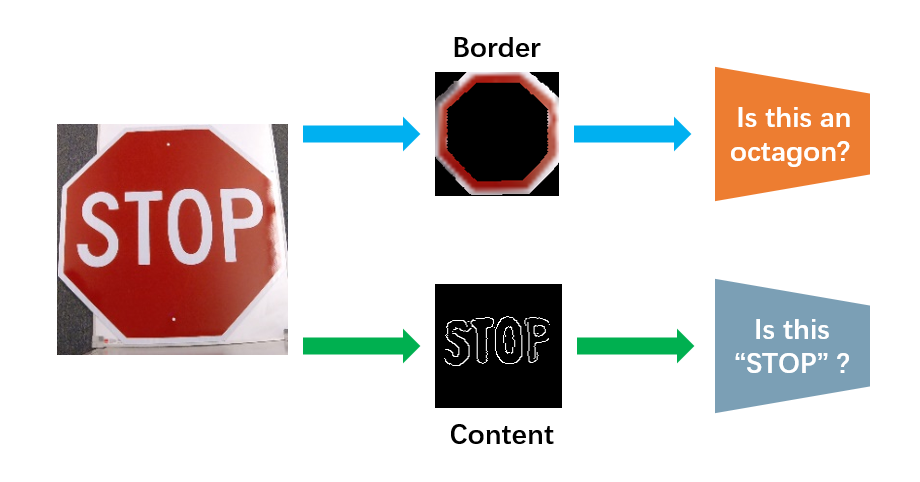}
\caption{Overview: workflow of the auxiliary models.}
\label{fig:sensor_implementation}
\end{figure*}

In total, in our \sys pipeline, we implement 19 submodels --- 1) One end-to-end GTSRB-CNN classifier~\citep{eykholt2018robust} as the main task model; 2) 8 binary preventative models for all 8 types of borders; 3) 6 binary permissive models for the 6 border types, each of which is shared only by a unique class of sign; 4) 3 binary permissive models based on edge map of sign content~(\textit{Blank Circle}, \textit{Arrow-Right-Down}, \textit{Arrow-Left-Ahead}); 5) A single permissive model for digit recognition, which is used to identify \textit{Digits-20}, \textit{Digits-50}, \textit{Digits-120}. All of the 17 binary classification neural models adopt the same backbone architecture in GTSRB-CNN and the rest digit recognition model adopts the architecture proposed in \citet{goodfellow2013multi}.

\textbf{Training Details.} 
To make our \sys pipeline function normally as the way we expect, next, we consider the training issues of the overall model.

Given the definition of permissive and preventative models, ideally, the permissive models should have low false rate and nontrivial truth rate, while the preventative models should have high truth rate and nontrivial false rate. These conditions are very critical for auxiliary models to bring accuracy improvement into the \sys pipeline. We guarantee the conditions to hold by assigning biased weights to classification loss on positive samples and negative samples during the training stage. Specifically, we train all of our binary auxiliary models with the following loss function:
$$
L(\mathcal{D},f) = a \mathbb{E}_{x \sim \mathcal{D^+}}[CE(f(x),1)] + b \mathbb{E}_{x \sim \mathcal{D^-}}[CE(f(x),0)],
$$
where $\mathcal{D} = \{D^+, D^-\}$ is the dataset, $D^+$ is the subset containing positive samples, $D^-$ is the subset containing negative samples, $f$ is the classifier and $CE$ is the crossentroy loss. For permissive model, we set $a << b$, so that low false rate will be encouraged at the cost of truth rate; while for preventative sensors, we set $a >> b$, then we can expect a high truth rate at the cost of some false rate. 

Besides the performance of each individual model, we also need to get proper weights for the reasoning graphical model in the \sys pipeline. Empirically, in our traffic sign identification task, since the end-to-end main task model has almost perfect accuracy on clean data, directly training on clean data will always give the main task model a dominant weight, leading to a trivial pipeline model. Thus, during training, we augment the training set with artificial adversarial samples, where the sensing signal from the main task model is randomly flipped. As a result, during training, to make correct predictions on these artificial adversarial samples, the optimizer must also assign nontrivial weights to other auxiliary models. We call the ratio of such artificial adversarial samples in the training set the ``adversarial ratio" in our context, indicating prior belief on the balance between benign and adversarial distributions, and use $\alpha$ to denote it. In our evaluation, we test different settings of $\alpha\in\{0,0.1,0.2,0.3,0.4,0.5,0.6,0.7,0.8,0.9,1.0\}$ and report the best results in Table~\ref{table:white_box_sensor_linf},\ref{table:black_box_sensor_linf},\ref{table:black_box_pipeline_linf},\ref{table:stop_sign_attack},\ref{table:white_box_sensor_unforeseen},\ref{table:black_box_sensor_unforeseen},\ref{table:black_box_pipeline_unforeseen},\ref{table:common_corruption}. In particular, we use $\alpha=0.8$ in Table~\ref{table:white_box_sensor_linf},\ref{table:white_box_sensor_unforeseen}, $\alpha=0.2$ in Table~\ref{table:black_box_sensor_linf},\ref{table:black_box_pipeline_linf},\ref{table:black_box_sensor_unforeseen},\ref{table:black_box_pipeline_unforeseen},\ref{table:common_corruption} and $\alpha=0.4$ in Table~\ref{table:stop_sign_attack}. Moreover, we also present the performance of \sys against different attacks with a fixed $\alpha = 0.5$ in Figure~\ref{fig:complete_robustness_improvement}.

For all the neural models, we use the standard Stochastic Gradient Descent Optimizer for training. The optimizer adopts a learning rate of $10^{-2}$, momentum of 0.9 and weight decay of $10^{-4}$. In all the training cases, we use 50000 training iterations with a batch size of 200 for each random training iteration. To train the weights of the graphical model in the pipeline, we perform Maximum Likelihood Estimate~(MLE) with the standard gradient descent algorithm, and we use a learning rate of $10^{-1}$ and run 4000 training iterations with a batch size of 50 for each random iterations.


\subsection{Visualization of Adversarial Examples and Corrupted Samples}
\label{appendix:examples-of-generated-adversarial-test-samples}

\begin{figure*}[t!]
\centering
\includegraphics[width=1.0\textwidth]{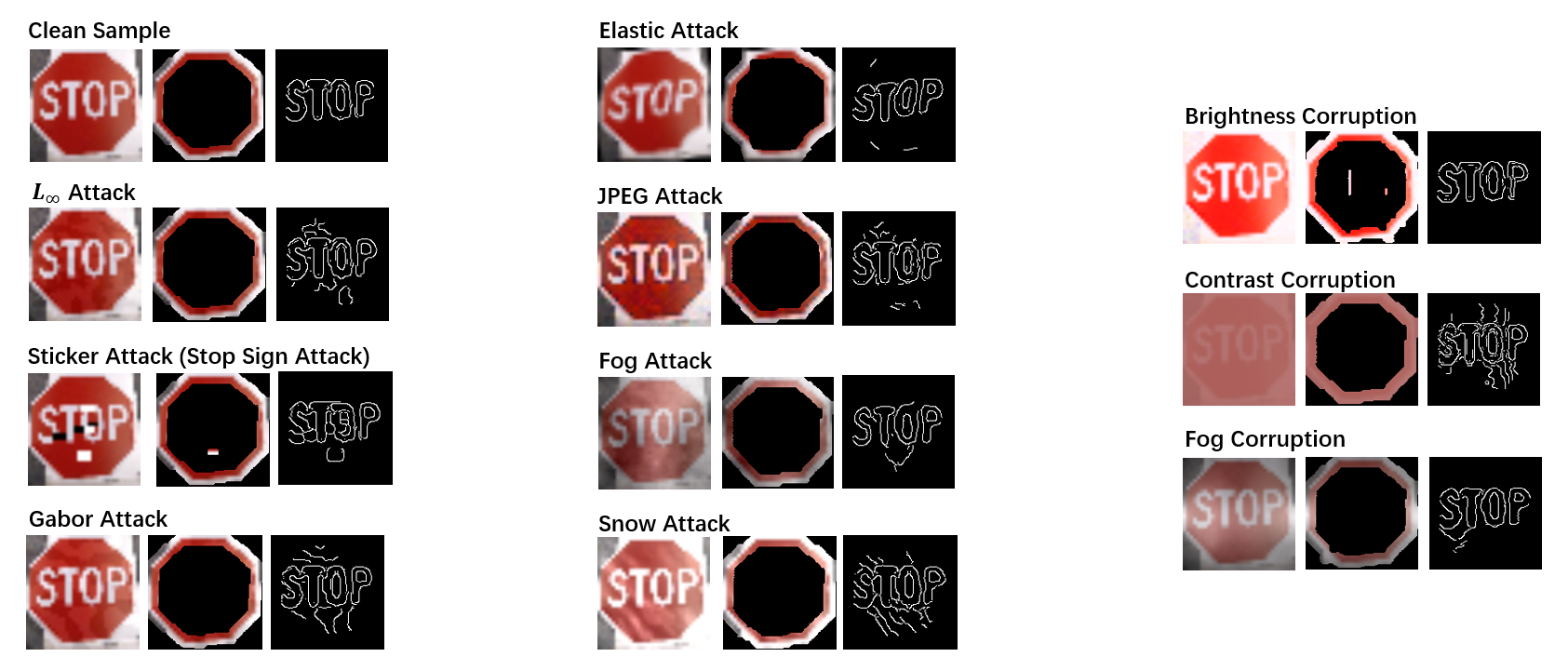}
\caption{Visualization of adversarial examples and corrupted samples.}
\label{fig:generated_adversarial_examples}
\end{figure*}

In Figure~\ref{fig:generated_adversarial_examples}, we provide a visualization of the generated adversarial examples~(corrupted samples) that are used for robustness evaluation in our work. For each type of attack~(corruption), we present the generated example~(the first image in each block), the extracted border~(the second image in each block), and the sign content~(the third image in each block) from the sample. 

As we can see, although the adversarial examples can easily fool an end-to-end neural network based main task model, the non-neural GrabCut algorithm and edge operator can still correctly extract the border and sign content from them. This allows other auxiliary models help to rectify the mistakes made by the main task model.

\subsection{Additional Experiment Results}
\label{appendix:additional_experiment_results}

In the main text, we have presented our evaluation results under the setting of whitebox sensor attack. In this subsection, we present the evaluation results of $\mathcal{L}_\infty$ attack and unforeseen attacks under blackbox sensor and blackbox pipeline attack settings. Specifically, we present the two blackbox results for $\mathcal{L}_\infty$ attack in table~\ref{table:black_box_sensor_linf} and table~\ref{table:black_box_pipeline_linf}, and accordingly the two blackbox results for unforeseen attacks in table~\ref{table:black_box_sensor_unforeseen} and table~\ref{table:black_box_pipeline_unforeseen}.

As shown, similar trends in whitebox sensor attack setting can also be observed in these two blackbox attack settings, which indicates that the robustness is not just coming from gradient masking~\citep{carlini2017towards,athalye2018obfuscated}.

\begin{table*}[tp]
    \caption{Adversarial accuracy under black-box sensor $\mathcal{L}_\infty$ attack, $\alpha=0.2$ (Accuracy $\%$)}
    \label{table:black_box_sensor_linf}
    \centering
    \resizebox{0.7\textwidth}{!}{
    \begin{tabular}{c|c|c|c|c|c|c}
        \toprule
             & & $\epsilon=0$ & $\epsilon=4$ & $\epsilon=8$ & $\epsilon=16$ & $\epsilon=32$ \\
             \hline
        \hline
            \multirow{2}{*}{GTSRB-CNN} & Main & $\bm{99.38}$ & $85.16$ & $67.98$ & $47.56$ & $25.69$ \\
            \cline{2-7} & KEMLP & $98.28(\tr{-1.10})$ & $91.36(\tg{+6.20})$ & $79.53(\tg{+11.55})$ & $61.21(\tg{+13.65})$ & $41.85(\tg{+16.16})$\\
            \hline
        \hline
            \multirow{2}{*}{AdvTrain ($\epsilon=4$)} & Main & $97.94$ & $94.88$ & $90.23$ & $72.99$ & $50.75$ \\
            \cline{2-7} & KEMLP & $97.89(\tr{-0.05})$ & $\bm{95.88}(\tg{+1.00})$ & $\bm{90.66}(\tg{+0.43})$ & $77.01(\tg{+4.02})$ & $55.56(\tg{+4.81})$\\
            \hline
        \hline
            \multirow{2}{*}{AdvTrain ($\epsilon=8$)} & Main & $93.72$ & $91.49$ & $89.02$ & $80.56$ & $64.76$ \\
            \cline{2-7} & KEMLP & $96.79(\tg{+3.07})$ & $94.29(\tg{+2.80})$ & $90.23(\tg{+1.21})$ & $\bm{81.40}(\tg{+0.84})$ & $65.92(\tg{+1.16})$\\
            \hline
        \hline
            \multirow{2}{*}{AdvTrain ($\epsilon=16$)} & Main & $84.54$ & $83.05$ & $82.00$ & $79.76$ & $\bm{73.20}$ \\
            \cline{2-7} & KEMLP & $94.68(\tg{+10.14})$ & $90.72(\tg{+7.67})$ & $86.52(\tg{+4.52})$ & $80.02(\tg{+0.26})$ & $70.47(\tr{-2.73})$\\
            \hline
        \hline
            \multirow{2}{*}{AdvTrain ($\epsilon=32$)} & Main & $74.74$ & $73.64$ & $72.79$ & $71.91$ & $67.77$ \\
            \cline{2-7} & KEMLP & $91.46(\tg{+16.72})$ & $86.60(\tg{+12.96})$ & $81.66(\tg{+8.87})$ & $75.69(\tg{+3.78})$ & ${66.77}(\tr{-1.00})$\\
            \hline
        \hline
            \multirow{2}{*}{DOA (5x5)} & Main & $97.43$ & $84.93$ & $70.70$ & $52.44$ & $33.15$ \\
            \cline{2-7} & KEMLP & $97.45(\tg{+0.02})$ & $92.21(\tg{+7.28})$ & $81.56(\tg{+10.86})$ & $64.07(\tg{+11.63})$ & $45.70(\tg{+12.55})$\\
            \hline
        \hline
            \multirow{2}{*}{DOA (7x7)} & Main & $97.27$ & $79.48$ & $65.77$ & $48.71$ & $30.99$ \\
            \cline{2-7} & KEMLP & $97.22(\tr{-0.05})$ & $90.56(\tg{+11.08})$ & $80.20(\tg{+14.43})$ & $62.55(\tg{+13.84})$ & $44.24(\tg{+13.25})$\\
            \hline
        \bottomrule
    
    \end{tabular}
    }
\end{table*}

\begin{table*}[tp]
    \caption{Adversarial accuracy under black-box pipeline $\mathcal{L}_\infty$ attack, $\alpha=0.2$ (Accuracy $\%$)}
    \label{table:black_box_pipeline_linf}
    \centering
    \resizebox{0.7\textwidth}{!}{
    \begin{tabular}{c|c|c|c|c|c|c}
        \hline
             & & $\epsilon=0$ & $\epsilon=4$ & $\epsilon=8$ & $\epsilon=16$ & $\epsilon=32$ \\
             \hline
        \hline
            \multirow{2}{*}{GTSRB-CNN} & Main & $\bm{99.38}$ & $81.17$ & $60.52$ & $37.60$ & $24.28$ \\
            \cline{2-7} & KEMLP & $98.28(\tr{-1.10})$ & $89.76(\tg{+8.59})$ & $76.18(\tg{+15.66})$ & $56.07(\tg{+18.47})$ & $37.50(\tg{+13.22})$\\
            \hline
        \hline
            \multirow{2}{*}{AdvTrain ($\epsilon=4$)} & Main & $97.94$ & $94.42$ & $88.32$ & $66.08$ & $46.60$ \\
            \cline{2-7} & KEMLP & $97.89(\tr{-0.05})$ & $\bm{95.88}(\tg{+1.46})$ & $\bm{89.61}(\tg{+1.29})$ & $71.91(\tg{+5.83})$ & $51.57(\tg{+4.97})$\\
            \hline
        \hline
            \multirow{2}{*}{AdvTrain ($\epsilon=8$)} & Main & $93.72$ & $90.72$ & $87.11$ & $75.49$ & $58.64$ \\
            \cline{2-7} & KEMLP & $96.79(\tg{+3.07})$ & $94.16(\tg{+3.44})$ & $89.40(\tg{+2.29})$ & ${77.31}(\tg{+1.82})$ & $60.26(\tg{+1.62})$\\
            \hline
        \hline
            \multirow{2}{*}{AdvTrain ($\epsilon=16$)} & Main & $84.54$ & $82.87$ & $81.46$ & $77.13$ & $\bm{70.09}$ \\
            \cline{2-7} & KEMLP & $94.68(\tg{+10.14})$ & $90.87(\tg{+8.00})$ & $86.37(\tg{+4.91})$ & $\bm{78.06}(\tg{+0.93})$ & ${68.44}(\tr{-1.65})$\\
            \hline
        \hline
            \multirow{2}{*}{AdvTrain ($\epsilon=32$)} & Main & $74.74$ & $73.66$ & $72.35$ & $70.16$ & $66.08$ \\
            \cline{2-7} & KEMLP & $91.46(\tg{+16.72})$ & $86.70(\tg{+13.04})$ & $81.74(\tg{+9.39})$ & $73.46(\tg{+3.30})$ & $65.23(\tr{-0.85})$\\
            \hline
        \hline
            \multirow{2}{*}{DOA (5x5)} & Main & $97.43$ & $81.94$ & $66.13$ & $48.28$ & $33.26$ \\
            \cline{2-7} & KEMLP & $97.45(\tg{+0.02})$ & $91.13(\tg{+9.19})$ & ${78.88}(\tg{+12.75})$ & $61.42(\tg{+13.14})$ & $42.36(\tg{+9.10})$\\
            \hline
        \hline
            \multirow{2}{*}{DOA (7x7)} & Main & $97.27$ & $77.85$ & $63.68$ & $46.55$ & $31.79$ \\
            \cline{2-7} & KEMLP & $97.22(\tr{-0.05})$ & $89.84(\tg{+11.99})$ & $77.78(\tg{+14.10})$ & $60.39(\tg{+13.84})$ & $40.90(\tg{+9.11})$\\
            \hline
        \hline
    
    \end{tabular}
    }
\end{table*}

\begin{table*}[tp]
    \caption{Adversarial accuracy under black-box sensor unforeseen attack, $\alpha=0.2$ (Accuracy $\%$)}
    \label{table:black_box_sensor_unforeseen}
    \centering
    \resizebox{1.0\textwidth}{!}{
    \begin{tabular}{c|c|c|c|c|c|c|c|c|c|c|c|c}
        \hline
             & & clean & fog-256 & fog-512 & snow-0.25 & snow-0.75 & jpeg-0.125 &jpeg-0.25 & gabor-20 & gabor-40 & elastic-1.5 & elastic-2.0\\
            \hline
        \hline
            \multirow{2}{*}{GTSRB-CNN} & Main & $\bm{99.38}$ & $77.55$  & $59.93$  & $78.50$ & $45.34$ & $83.10$ & $65.90$ & $75.36$ & $59.26$ & $77.16$ & $57.64$\\
            \cline{2-13} & KEMLP & ${98.28}(\tr{-1.10})$ & $84.03(\tg{+6.48})$  & $68.54(\tg{+8.61})$  & $83.08(\tg{+4.58})$ & $57.77(\tg{+12.43})$ & $88.97(\tg{+5.87})$ & $74.90(\tg{+9.00})$ & $84.88(\tg{+9.52})$ & $70.04(\tg{+10.78})$ & $82.10(\tg{+4.94})$ & $66.69(\tg{+9.05})$\\
            \hline
        \hline
            \multirow{2}{*}{AdvTrain ($\epsilon=4$)} & Main & $97.94$ & $70.68$ & $54.06$ & $77.70$ & $49.67$ & $87.45$ &$72.84$ & $88.14$ & $68.21$ & $83.38$ & $70.09$\\
            \cline{2-13} & KEMLP & ${97.89}(\tr{-0.05})$ & $79.37(\tg{+8.69})$  & $64.38(\tg{+10.32})$  & $82.38(\tg{+4.68})$ & $59.21(\tg{+9.54})$ & $\bm{91.80}(\tg{+4.35})$ & $80.09(\tg{+7.25})$ & $91.51(\tg{+3.37})$ & $75.05(\tg{+6.84})$ & $84.80(\tg{+1.42})$ & $73.12(\tg{+3.03})$\\
            \hline
        \hline
            \multirow{2}{*}{AdvTrain ($\epsilon=8$)} & Main & $93.72$ & $67.70$ & $53.73$ & $76.13$ & $51.75$ & $86.27$ &$76.75$ & $89.25$ & $76.47$ & $80.71$ & $67.85$\\
            \cline{2-13} & KEMLP & ${96.79}(\tg{+3.07})$ & $76.70(\tg{+9.00})$  & $64.97(\tg{+11.24})$  & $80.99(\tg{+4.86})$ & $60.39(\tg{+8.64})$ & $91.02(\tg{+4.75})$ & $82.54(\tg{+5.79})$ & $\bm{91.56}(\tg{+2.31})$ & $79.45(\tg{+2.98})$ & $83.26(\tg{+2.55})$ & $71.37(\tg{+3.52})$\\
            \hline
        \hline
            \multirow{2}{*}{AdvTrain ($\epsilon=16$)} & Main & $84.54$ & $66.44$ & $49.64$ & $75.15$ & $52.73$ & $81.58$ &$77.78$ & $83.90$ & $82.48$ & $76.23$ & $68.26$\\
            \cline{2-13} & KEMLP & ${94.68}(\tg{+10.14})$ & $77.11(\tg{+10.67})$  & $63.84(\tg{+14.20})$  & $81.58(\tg{+6.43})$ & $60.73(\tg{+8.00})$ & $87.68(\tg{+6.10})$ & $\bm{82.77}(\tg{+4.99})$ & $89.27(\tg{+5.37})$ & $\bm{83.44}(\tg{+0.96})$ & $81.07(\tg{+4.84})$ & $71.55(\tg{+3.29})$\\
            \hline
        \hline
            \multirow{2}{*}{AdvTrain ($\epsilon=32$)} & Main & $74.74$ & $65.82$ & $50.18$ & $71.97$ & $52.37$ & $72.61$ &$71.09$ & $76.26$ & $77.16$ & $68.03$ & $64.38$\\
            \cline{2-13} & KEMLP & ${91.46}(\tg{+16.72})$ & $77.62(\tg{+11.80})$  & $64.56(\tg{+14.38})$  & $79.60(\tg{+7.63})$ & $61.09(\tg{+8.72})$ & $83.85(\tg{+11.24})$ & $79.30(\tg{+8.21})$ & $85.60(\tg{+9.34})$ & $80.09(\tg{+2.93})$ & $77.67(\tg{+9.64})$ & $70.81(\tg{+6.43})$\\
            \hline
        \hline
            \multirow{2}{*}{DOA (5x5)} & Main & $97.43$ & $78.24$ & $62.32$ & $79.55$ & $56.69$ & $86.55$ & $71.32$ & $82.23$ & $67.28$ & $87.96$ & $75.75$\\
            \cline{2-13} & KEMLP & ${97.41}(\tr{-0.02})$ & $\bm{84.26}(\tg{+6.02})$  & $\bm{69.08}(\tg{+6.76})$  & $83.36(\tg{+3.81})$ & $\bm{62.58}(\tg{+5.89})$ & ${90.41}(\tg{+3.86})$ & $77.98(\tg{+6.66})$ & $87.06(\tg{+4.83})$ & $73.69(\tg{+6.41})$ & $\bm{86.09}(\tr{-1.87})$ & $\bm{75.90(\tg{+0.15})}$\\
            \hline
        \hline
            \multirow{2}{*}{DOA (7x7)} & Main & $97.27$ & $76.34$ & $61.32$ & $79.30$ & $55.94$ & $83.20$ & $66.10$ & $82.25$ & $67.54$ & $86.73$ & $73.77$\\
            \cline{2-13} & KEMLP & ${97.22}(\tr{-0.05})$ & $82.74(\tg{+6.40})$  & ${68.52}(\tg{+7.20})$  & $\bm{83.74}(\tg{+4.44})$ & ${62.47}(\tg{+6.53})$ & $89.04(\tg{+5.84})$ & ${76.44}(\tg{+10.34})$ & $87.60(\tg{+5.35})$ & $74.51(\tg{+6.97})$ & $85.91(\tr{-0.82})$ & $75.49(\tg{+1.72})$\\
            \hline
        \hline
    \end{tabular}
    }
\end{table*}

\begin{table*}[tp]
    \caption{Adversarial accuracy under black-box pipeline unforeseen attack, $\alpha=0.2$ (Accuracy $\%$)}
    \label{table:black_box_pipeline_unforeseen}
    \centering
    \resizebox{1.0\textwidth}{!}{
    \begin{tabular}{c|c|c|c|c|c|c|c|c|c|c|c|c}
        \hline
             & & clean & fog-256 & fog-512 & snow-0.25 & snow-0.75 & jpeg-0.125 &jpeg-0.25 & gabor-20 & gabor-40 & elastic-1.5 & elastic-2.0\\
            \hline
        \hline
            \multirow{2}{*}{GTSRB-CNN} & Main & $\bm{99.38}$ & $71.17$  & $49.13$  & $70.73$ & $36.45$ & $75.44$ & $51.98$ & $72.61$ & $53.47$ & $70.88$ & $54.53$\\
            \cline{2-13} & KEMLP & ${98.28}(\tr{-1.10})$ & $78.96(\tg{+7.79})$  & $60.65(\tg{+11.52})$  & $80.02(\tg{+9.29})$ & $52.16(\tg{+15.71})$ & $85.31(\tg{+9.87})$ & $67.64(\tg{+15.66})$ & $84.13(\tg{+11.52})$ & $69.24(\tg{+15.77})$ & $80.66(\tg{+9.78})$ & $67.80(\tg{+13.27})$\\
            \hline
        \hline
            \multirow{2}{*}{AdvTrain ($\epsilon=4$)} & Main & $97.94$ & $66.23$ & $47.33$ & $73.46$ & $42.10$ & $84.23$ &$65.07$ & $87.29$ & $66.95$ & $82.10$ & $68.80$\\
            \cline{2-13} & KEMLP & ${97.89}(\tr{-0.05})$ & $74.97(\tg{+8.74})$  & $58.62(\tg{+11.29})$  & ${80.63}(\tg{+7.17})$ & $54.09(\tg{+11.99})$ & $\bm{90.84}(\tg{+6.61})$ & $76.00(\tg{+10.93})$ & $90.61(\tg{+3.32})$ & $74.77(\tg{+7.82})$ & $84.85(\tg{+2.75})$ & $74.95(\tg{+6.15})$\\
            \hline
        \hline
            \multirow{2}{*}{AdvTrain ($\epsilon=8$)} & Main & $93.72$ & $63.14$ & $45.14$ & $72.87$ & $46.66$ & $84.59$ & $71.35$ & $88.86$ & $73.74$ & $80.30$ & $67.88$\\
            \cline{2-13} & KEMLP & ${96.79}(\tg{+3.07})$ & $72.89(\tg{+9.75})$  & $58.02(\tg{+12.88})$  & $79.73(\tg{+6.86})$ & $55.86(\tg{+9.20})$ & $90.59(\tg{+6.00})$ & $80.02(\tg{+8.67})$ & $\bm{90.92}(\tg{+2.06})$ & $77.93(\tg{+4.19})$ & $83.80(\tg{+3.50})$ & $73.77(\tg{+5.89})$\\
            \hline
        \hline
            \multirow{2}{*}{AdvTrain ($\epsilon=16$)} & Main & $84.54$ & $62.32$ & $42.98$ & $73.23$ & $50.08$ & $80.97$ &$76.26$ & $83.51$ & $81.22$ & $75.80$ & $68.75$\\
            \cline{2-13} & KEMLP & ${94.68}(\tg{+10.14})$ & $73.48(\tg{+11.16})$  & $58.18(\tg{+15.20})$  & $80.45(\tg{+7.22})$ & $57.54(\tg{+7.46})$ & $86.99(\tg{+6.02})$ & $\bm{80.92}(\tg{+4.66})$ & $88.30(\tg{+4.79})$ & $\bm{82.23}(\tg{+1.01})$ & $81.71(\tg{+5.91})$ & $72.69(\tg{+3.94})$\\
            \hline
        \hline
            \multirow{2}{*}{AdvTrain ($\epsilon=32$)} & Main & $74.74$ & $61.86$ & $45.01$ & $70.47$ & $50.57$ & $72.38$ & $69.70$ & $76.16$ & $76.39$ & $68.65$ & $64.99$\\
            \cline{2-13} & KEMLP & ${91.46}(\tg{+16.72})$ & $73.33(\tg{+11.47})$  & $58.49(\tg{+13.48})$  & $78.94(\tg{+8.47})$ & $58.67(\tg{+8.10})$ & $83.33(\tg{+10.95})$ & $77.42(\tg{+7.72})$ & $84.95(\tg{+8.79})$ & $79.09(\tg{+2.70})$ & $78.37(\tg{+9.72})$ & $71.45(\tg{+6.46})$\\
            \hline
        \hline
            \multirow{2}{*}{DOA (5x5)} & Main & $97.43$ & $75.01$ & $56.97$ & $77.67$ & $53.14$ & $83.15$ & $63.79$ & $82.07$ & $65.77$ & $\bm{88.17}$ & $\bm{78.88}$\\
            \cline{2-13} & KEMLP & ${97.41}(\tr{-0.02})$ & $\bm{80.40}(\tg{+5.39})$  & $\bm{64.40}(\tg{+7.43})$  & $82.28(\tg{+4.61})$ & $59.52(\tg{+6.38})$ & ${88.89}(\tg{+5.74})$ & $73.69(\tg{+9.90})$ & $87.04(\tg{+4.97})$ & $73.43(\tg{+7.66})$ & ${86.99}(\tr{-1.18})$ & ${77.88}(\tr{-1.00})$\\
            \hline
        \hline
            \multirow{2}{*}{DOA (7x7)} & Main & $97.27$ & $73.97$ & $57.05$ & $77.21$ & $53.55$ & $81.40$ & $62.68$ & $82.15$ & $67.28$ & $87.42$ & $78.27$\\
            \cline{2-13} & KEMLP & ${97.22}(\tr{-0.05})$ & $80.04(\tg{+6.07})$  & ${64.17}(\tg{+7.12})$  & $\bm{82.46}(\tg{+5.25})$ & $\bm{59.75}(\tg{+6.20})$ & $88.30(\tg{+6.90})$ & $73.48(\tg{+10.80})$ & $87.09(\tg{+4.94})$ & $73.95(\tg{+6.67})$ & $86.42(\tr{-1.00})$ & $78.58(\tg{+0.31})$ \\
            \hline
        \hline
    \end{tabular}
    }
\end{table*}

\end{document}